\documentclass{article}

\usepackage{fullpage}
\usepackage{amsmath}
\usepackage{hyperref}
\usepackage[numbers,compress]{natbib}
\usepackage[noend]{algpseudocode}
\newcommand{\T}{^\top}
\newcommand{\x}{w}
\newcommand{\R}{R}
\usepackage{listings}
\usepackage[table]{xcolor}

\usepackage{microtype}
\usepackage{graphicx}
\usepackage{enumitem,caption,subcaption}
\usepackage{booktabs} %
\usepackage{multirow}
\usepackage{adjustbox}

\usepackage{stmaryrd}
\usepackage{graphicx}
\usepackage{caption}
\usepackage{subcaption}

\usepackage{textcomp} %
\usepackage{quoting}
\usepackage{titletoc} %
\usepackage{fontawesome} %
\usepackage{wrapfig}
\usepackage{xspace}
\usepackage{ifthen} %
\usepackage{cleveref}

\usepackage{pifont}%

\PassOptionsToPackage{colorlinks = true,
            linkcolor = blue,
            urlcolor  = blue,
            citecolor = blue,
            anchorcolor = blue}{hyperref}

\usepackage{algorithm}
\PassOptionsToPackage{noend}{algpseudocode}
\algrenewcommand\algorithmicrequire{\textbf{Input:}}
\algrenewcommand\algorithmicensure{\textbf{Output:}}

\newlength{\continueindent}
\usepackage{etoolbox}
\makeatletter
\newcommand*{\ALG@customparshape}{\parshape 2 \leftmargin \linewidth \dimexpr\ALG@tlm+\continueindent\relax \dimexpr\linewidth+\leftmargin-\ALG@tlm-\continueindent\relax}
\apptocmd{\ALG@beginblock}{\ALG@customparshape}{}{\errmessage{failed to patch}}
\makeatother
\algblock{ParFor}{EndParFor}
\algnewcommand\algorithmicpardo{\textbf{in parallel do}}
\algrenewtext{ParFor}[1]{\algorithmicfor\ #1\ \algorithmicpardo}
\algrenewtext{EndParFor}{\textbf{end for}}
\algtext*{EndParFor}
\algdef{SE}[SUBALG]{Indent}{EndIndent}{}{\algorithmicend\ }%
\algtext*{Indent}
\algtext*{EndIndent} %
\usepackage{bm,mathtools,amsthm,amssymb}
\usepackage{bbm}
\usepackage{lscape}

\newcommand{\update}[1]{#1}

\newcommand{\myparagraph}[1]{\vspace*{0.5em}\par\noindent\textbf{{#1}.}} %

\newcommand\D{\textnormal{d}}

\newcommand{\fedavg}{{FedAvg}\xspace}
\newcommand{\fedavgsub}{{FedAvg-Sub}\xspace}
\newcommand{\fedprox}{{FedProx}\xspace}
\newcommand{\qffl}{$q${-FFL}\xspace}
\newcommand{\afl}{{AFL}\xspace}
\newcommand{\term}{{Tilted-ERM}\xspace}
\newcommand{\newfl}{$\Delta$-{FL}\xspace}

\renewcommand{\epsilon}{\varepsilon}

\newcommand{\tabemph}[1]{}

\usepackage{color}
\definecolor{puorange}{rgb}{0.80,0.20,0}
\definecolor{bluegray}{rgb}{0.04,0,0.7}
\definecolor{greengray}{rgb}{0.05,0.50,0.15}
\definecolor{darkbrown}{rgb}{0.40,0.2,0.05}
\definecolor{darkcyan}{rgb}{0,0.4,1}
\definecolor{black}{rgb}{0,0,0}
\definecolor{grey}{rgb}{0.93,0.93,0.93}

\newcommand \reals {\mathbb{R}}

\newcommand \inv {^{-1}} %
    \renewcommand \T {^{\top}}	%
    \renewcommand{\x}{ w}
    \renewcommand{\R}{\mathbb{R}}
\newcommand \prob {\mathbb{P}}

\newcommand \expect {\mathbb{E}}
\newcommand \supq {\mathbb{S}}
\newcommand{\riskM}{\mathbb{M}}

\newcommand \pow [1]{^{(#1)}}
\DeclarePairedDelimiterX{\inp}[2]{\langle}{\rangle}{#1, #2} %
\DeclarePairedDelimiterX{\normsq}[1]{\Vert}{\Vert^2}{#1} %

\newcommand{\eps}{\epsilon}
\newcommand \prox {\mathop{\mathrm{prox}}\nolimits}  %

\newcommand \grad {\nabla}

\theoremstyle{thmstyleone}%
\newtheorem{theorem}{Theorem}%
\newtheorem{proposition}[theorem]{Proposition}%
\newtheorem{lemma}[theorem]{Lemma}
\newtheorem{claim}[theorem]{Claim}

\newtheorem{property}[theorem]{Property}

\theoremstyle{thmstyletwo}%
\newtheorem{remark}{Remark}%

\theoremstyle{thmstylethree}%

 \usepackage{thmtools}
 \declaretheoremstyle[
notefont=\bfseries, notebraces={}{},
bodyfont=\normalfont\itshape,
headformat=\NAME \NOTE
]{nopar}

\usepackage[utf8]{inputenc} %
\usepackage[T1]{fontenc}    %
\usepackage{hyperref}       %
\usepackage{url}            %
\usepackage{booktabs}       %
\usepackage{amsfonts}       %
\usepackage{nicefrac}       %
\usepackage{microtype}      %
\usepackage{lipsum}
\usepackage{amsmath}
\usepackage{amsthm}
\usepackage{amssymb}
\usepackage{changepage}
\usepackage{enumitem}
\setlist{leftmargin=5.5mm}

\usepackage{algorithm}
\usepackage{algorithmicx}
\usepackage[noend]{algpseudocode}
\algnewcommand{\Initialize}[1]{%
  \State \textbf{Initialize:}
  \Statex \hspace*{\algorithmicindent}\parbox[t]{.8\linewidth}{\raggedright #1}
}

\DeclareMathOperator*{\argmin}{arg\,min}
\DeclareMathOperator*{\argmax}{arg\,max}

\DeclareMathOperator{\ZZ}{{\mathbb{Z}}}
\DeclareMathOperator{\Rd}{\mathbb{R}^d}

\newcommand{\range}[1]{\{1, \dots, #1\}}

\newcommand{\dotproduct}[2]{\left\langle #1, #2 \right\rangle}

\DeclareMathOperator{\rv}{X}

\newcommand{\Fcal}{\mathcal{F}}
\newcommand{\Pcal}{\mathcal{P}}
\newcommand{\Tcal}{\mathcal{T}}
\newcommand{\Ncal}{\mathcal{N}}
\newcommand{\Xcal}{\mathcal{X}}
\newcommand{\Acal}{\mathcal{A}}

\newcommand{\expectation}[2]{\mathbb{E}_{#2}\left[#1\right]}
\newcommand{\norm}[1]{\left\|#1 \right\|}
\newcommand{\squarednorm}[1]{\left\|#1 \right\|^2}

\newcommand{\globalmodel}[1]{w^{(#1)}}

\newcommand{\localmodel}[3]{w_{#2, #3}^{(#1)}} %

\newcommand{\meanoverselectedsets}[1]{#1 \mathbb{E}_{S \sim U_m}}

\DeclareMathOperator{\totalnoise}{\Tcal_3}

\renewcommand{\tabemph}[1]{\cellcolor{gray!10}{#1}}%
\renewcommand{\update}[1]{#1}

\title{Federated Learning with Superquantile Aggregation for Heterogeneous Data}

\author{Krishna Pillutla$^{* 1}$  $\qquad$
Yassine Laguel$^{*2}$ $\qquad$
Jérôme Malick$^3$ $\qquad$
Zaid Harchaoui$^4$ 
\vspace{0.5em} 
\\ 
{\small
$^1$Google Research $\qquad$
$^2$Rutgers University}  \\
{\small
$^3$CNRS $\qquad$
$^4$University of Washington 
}
}
\date{\vspace*{-3em}}

\begin{document}
\maketitle
%
%
%
\let\thefootnote\relax\footnotetext{$^*$These authors contributed equally to this work.}

\begin{abstract}
We present a federated learning framework that is designed to robustly deliver good predictive performance across individual clients with heterogeneous data. The proposed approach hinges upon a superquantile-based learning objective that captures the tail statistics of the error distribution over heterogeneous clients. We present a stochastic training algorithm that interleaves differentially private client filtering with federated averaging steps. We prove finite time convergence guarantees for the algorithm: $O(1/\sqrt{T})$ in the nonconvex case in $T$ communication rounds and $O(\exp(-T/\kappa^{3/2}) + \kappa/T)$ in the strongly convex case with local condition number $\kappa$. Experimental results on benchmark datasets for federated learning demonstrate that our approach is competitive with classical ones in terms of average error and outperforms them in terms of tail statistics of the error. 
 \end{abstract}

\section{Introduction} \label{sec:sfl:intro}
 Federated learning is a distributed machine learning framework where many clients (e.g. mobile devices) collaboratively train a model under the orchestration of a central server (e.g. service provider) while keeping the training data private and local to the client throughout the training process~\cite{mcmahan2017communication,kairouz2019flsurvey}.
It has found widespread adoption 
across industry~\cite{bonawitz2019towards,paulik2021federated}
for applications ranging from 
smart device apps~\citep{yang2018applied,hard2018federated} %
to healthcare~\citep{brisimi2018federated,huang2019patient}.

A key feature of federated learning is the 
statistical heterogeneity, i.e., 
client data distributions
are {\em not} identically distributed~\cite{kairouz2019flsurvey,li2020federated}.
In typical cross-device federated learning scenarios, 
each client corresponds to a user. The diversity in the data 
they generate reflects the diversity in their unique personal, cultural, regional, and geographical characteristics. 

This data heterogeneity in federated learning manifests itself as a train-test distributional shift. 
Indeed, the usual approach minimizes the prediction error of the model on average over the population of clients available for training~\cite{mcmahan2017communication} while at test time, the same model is deployed on individual clients. 
This approach can fail on clients whose data distribution is far from most of the population or who may have less data than most of the population.
It is highly desirable, therefore, to have a federated learning method that can robustly deliver good predictive performance across
a wide variety of natural distribution shifts posed by individual clients. 

We present in this paper a robust approach to federated learning that guarantees a minimum level of predictive performance to all clients, even in situations where the population is heterogeneous.  The method we develop addresses these issues by minimizing a learning objective based on the notion of a superquantile~\cite{rockafellar2002conditional,rockafellar2008risk}, a risk measure that captures the tail behavior of a random variable. 

Training models with a learning objective involving the superquantile raises challenges. 
The superquantile is a non-smooth functional with sophisticated properties. Furthermore, the superquantile function can be seen as a kind of nonlinear expectation that we would like to blend well with averaging mechanisms. 
We show how to address the former by leveraging the dual formulation and the latter by leveraging the tail-domain viewpoint.
As a result, we can obtain an algorithm that can be implemented in a similar way to \fedavg~\cite{mcmahan2017communication} yet offers important benefits to heterogeneous populations.

The approach we propose, \newfl, allows one to control higher percentiles of the distribution of errors over the heterogeneous population of clients. We show in the experiments that our approach is more efficient than a direct approach, simply seeking to minimize the worst error over the population of clients. Compared to~\fedavg,~\newfl delivers improved prediction to tail clients or data-poor clients. 
Our algorithm relies on differentially private quantile computation to filter out clients on which to run federated averaging steps.
We present finite-time theoretical convergence guarantees for our algorithm when used to train additive models or deep networks and prove bounds on the privacy and utility of the algorithm. 

\begin{figure*}[t]
\begin{center}
\includegraphics[width=0.99\linewidth]{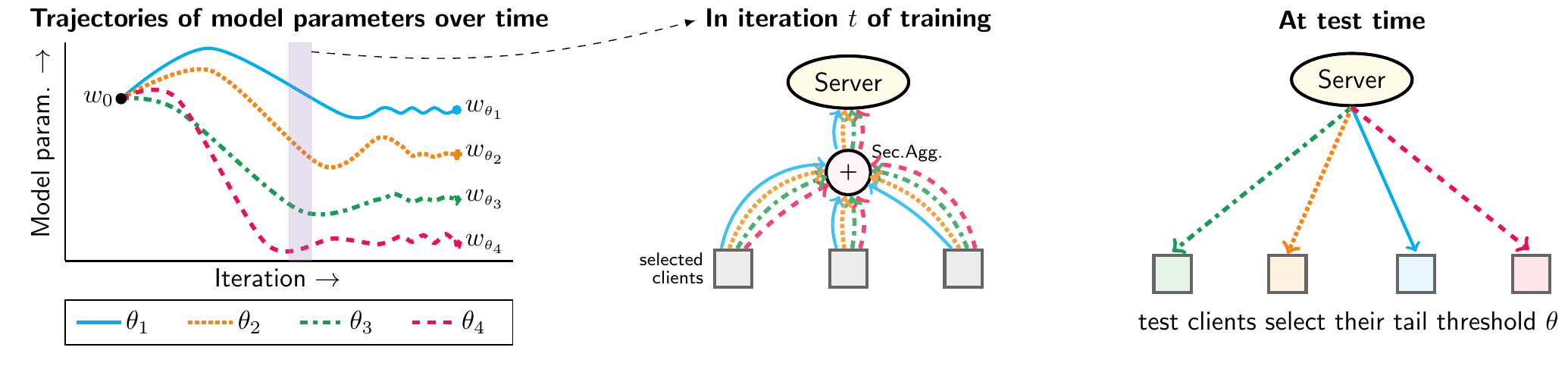}
\end{center}
    \caption{\small{
    Schematic summary of the \newfl framework.~~\textbf{Left}: The server maintains multiple models $w_{\theta_j}$, one for each tail threshold\;$\theta_j$. 
    \textbf{Middle}:
    During training, 
    selected clients participate in training 
    {\em each} model $w_{\theta_j}$.
    Individual updates are 
    securely aggregated to
    update the server model. 
    \textbf{Right}:
    Each test user is allowed to select their tail threshold
    $\theta$, and are served the corresponding
    model $w_\theta$.
    }}
    \label{fig:sfl:schematic_of_framework}
\end{figure*}

\subsection{Contributions}
We make the following concrete contributions in this work.

\myparagraph{The \newfl Framework}
The usual objective of federated learning, which we call the \emph{vanilla FL} objective is 
\begin{align} \label{eq:fl:avg}
    \min_{w \in \reals^d} 
    \frac{1}{n}\sum_{i=1}^n F_i(w) + \frac{\lambda}{2}\normsq{w} \,,
\end{align}
where $F_i(w) = \expect_{\xi \sim q_i}[f(w; \xi)]$ is the expected loss on client $i$ under its data distribution $q_i$ for $i=1, \ldots, n$, and $\lambda$ is a regularization parameter~\cite{mcmahan2017communication}.
Minimizing the average loss can lead to poor performance on clients whose distribution $p$ is far from the population training distribution $p_{\mathrm{train}} = (1/n) \sum_{i=1}^n q_i$. 
\update{Our goal is to improve the performance on such \emph{tail clients}.}

\update{
To this end, we directly minimize the average loss across tail clients whose loss is above a certain tail threshold. We formalize this through the notion of a risk measure known as the \textbf{superquantile}, a tail summary statistic of random variables~\cite{rockafellar2002conditional}. 
The $(1-\theta)$-superquantile is defined for a  continuous random variable $Z$ and $\theta \in (0, 1)$ as $\supq_\theta(Z) = \expect[Z \,\vert\, Z > Q_\theta(Z)]$, where $Q_\theta(Z)$ is the $(1-\theta)$-quantile of $Z$.
A similar interpretation holds for discrete distributions; the formal definition of the superquantile for this case is given in \Cref{sec:sfl:setup:heterogeneity}. 
}

\update{
Instead of minimizing the average loss as in \eqref{eq:fl:avg}, 
the proposed framework \newfl minimizes the tail loss across clients, as measured by the superquantile. 
Concretely, at a tail threshold $\theta \in (0, 1)$, we minimize
\begin{align} \label{eq:sfl}
    F_\theta(w) := \supq_\theta\big(F_1(w), \ldots, F_n(w)\big) + \frac{\lambda}{2} \norm{w}^2  \,,
\end{align}
where $\supq_\theta(a_1, \ldots, a_n)$ is the $(1-\theta)$-superquantile of the empirical distribution $(1/n)\sum_{i=1}^n \delta_{a_i}$.
Thus, the objective \eqref{eq:sfl} measures the tail statistics of the per-client loss distribution. 
}

\update{
By a duality argument, we show that the superquantile objective \eqref{eq:sfl} promotes distributional robustness. 
If we have a test client who is unseen during training and whose distribution $p_\pi = \sum_{i=1}^n \pi_i q_i$ can be written as a mixture of the training distributions $q_1, \ldots, q_n$, then the \newfl objective can be written as 
\[
    F_\theta(w) = \max_{\pi_i \le 1/(\theta n)} \expect_{\xi \sim p_\pi} \left[f(w; \xi) \right] + \frac{\lambda}{2}\norm{w}^2\,. 
\]
In other words, we minimize the \emph{worst-case loss over all mixture distributions with a constraint} $\pi_i \le 1/(\theta n)$ on the mixture weights; see \Cref{sec:sfl:framework} for details. 
}

\myparagraph{Optimization Algorithms}
\update{
To design a federated optimization algorithm to optimize the \newfl objective, the nonsmoothness of the 
superquantile $a \mapsto \supq_\theta(a_1, \ldots, a_n)$ might lead to potential difficulties in optimization. 
Fortunately, we can derive an expression for the subgradient of the \newfl objective \eqref{eq:sfl}: when $\theta n$ is an integer, we have
}
\[
  \sum_{i=1}^n \pi_i^\star F_i(w) + \lambda w \in \partial F_\theta(w)\,,
  \quad \text{where} \quad
  \pi_i^\star = \frac{\mathbb{I}(F_i(w) \ge Q_\theta)}{\sum_{j = 1}^n \mathbb{I}(F_j(w) \ge Q_\theta)} \,,
\]
and $Q_\theta = Q_\theta(F_1(w), \ldots, F_n(w))$ is the $(1-\theta)$-quantile of the losses evaluated at $w$.
\update{In other words, averaging the gradients of the losses that are larger than the quantile $Q_\theta$ gives a valid subgradient of the objective \eqref{eq:sfl}.} 

Using this expression, we design a federated optimization algorithm that interleaves federated averaging with differentially private quantile estimation. Specifically, the local updates $w_i^+$ from the subsample of $m$ selected clients $i \in S$ are aggregated to update the global model with the following two steps: 
\begin{itemize}[topsep=0.0ex,partopsep=0.0ex,itemsep=0.0ex,leftmargin=1.5em]
    \item estimate $\hat Q_\theta \approx Q_\theta(F_i(w) \,:\, i\in S)$ using the distributed discrete Gaussian mechanism~\cite{kairouz2021distributed} and hierarchical histograms~\cite{cormode2019answering}, and
    \item aggregate the updates from the tail clients  where $F_i(w) \ge \hat Q_\theta$ to find the new global model $w^+$ as 
    \[
        w^+ = \frac{1}{\vert S_\theta \vert}\sum_{i \in S_\theta} w_i^+
        \,,
      \quad \text{where} \quad
      S_\theta = \{i \,:\, F_i(w) \ge \hat Q_\theta\} \,.
    \]
\end{itemize}
Similar to FedAvg, this aggregation rule enjoys a simplification in the case of a single local update per-client with a learning rate $\gamma$. Specifically, under the assumption of full client participation (i.e., $m=n$), if the local update $w - w_i^+ = \gamma \grad\left(F_i(w) + (\lambda/2) \normsq{w}\right)$ is a single gradient step and $\hat Q_\theta = Q_\theta(F_1(w), \ldots, F_n(w))$ is the exact quantile of the per-client losses, the aggregated update is simply a subgradient step $w-w^+ = \gamma \grad F_\theta(w)$ where we denote the subgradient as $\grad F_\theta(w) \in \partial F_\theta(w)$. 

\myparagraph{Convergence Analysis}
Apart from the nonsmoothness of the superquantile, the convergence analysis also has to overcome the difficulty that we cannot obtain unbiased minibatch subgradient estimators for the superquantile objective.
\update{
Given $m$ i.i.d. copies $Z_1, \ldots, Z_m$ of a random variable $Z$, the empirical mean $\bar Z_m = (1/m) \sum_{i=1}^m Z_i$ is an unbiased estimate of the population mean, i.e., $\expect[\bar Z_m] = \expect[Z]$. This is no longer true for the superquantile, i.e., $\expect[\supq_\theta(Z_1, \ldots, Z_m)] \neq \supq_\theta(Z)$. As a result, we cannot access unbiased stochastic gradients in the learning setting,  where $m$ is the minibatch size. 
Moreover, it is not reasonable to assume in federated learning that we have access to all the clients due to a diurnal availability pattern of clients~\cite{kairouz2019flsurvey}. We overcome this issue by actually minimizing the \emph{expected minibatch superquantile} instead. It is defined as 
\[
    \overline F_{\theta}(w) := \expect_{(i_1, \ldots, i_m) \sim U_m} \left[ \supq_\theta\big(F_{i_1}(w), \ldots, F_{i_m}(w)\big) \right] \,,
\] 
where $U_m$ is the uniform distribution over all subsets of $\{1, \ldots, n\}$ of batch size $m$.
We can build an unbiased subgradient estimator for this objective by sampling a minibatch $(i_1, \ldots, i_m) \sim U_m$.
This is a uniform close surrogate of the original objective~\cite[Prop. 1]{levy2020largescale}
}
\[
    \vert F_\theta(w) - \overline F_{\theta}(w) \vert \le O \left( \frac{\max_{i=1, \ldots, n}  \vert F_i(w)\vert }{\sqrt{\theta m}}  \right) \,.
\]

Assuming that each $F_i$ is $G$-Lipschitz and $L$-smooth,
we establish a rate of $\sqrt{LG^2/T}$ 
in the nonconvex (and nonsmooth) case where $\lambda=0$. 
If, additionally, each $F_i$ is convex and $\lambda > 0$,
the problem is strongly convex and
we establish a rate of 
$\exp(-T / \kappa^{3/2}) + G^2/(\lambda T)$ in this case where
$\kappa = 1 + L /\lambda$ is the per-client condition number.

\myparagraph{Privacy and Utility Analysis}
The standard algorithms to compute quantiles with differential privacy are based on the exponential mechanism and require a trusted central aggregator~\cite{smith2011privacy}. Since this is not usually the case in federated learning, we 
estimate the cumulative distribution using 
the hierarchical histogram method and combine it with the distributed discrete Gaussian mechanism~\cite{kairouz2021distributed} in order to simulate a central aggregation using a cryptographic primitive known as secure aggregation~\cite{bonawitz2017practical}.  
\update{The hierarchical histogram method, also known as tree aggregation, is a classical approach to answer range queries under differential privacy~\cite{hay2010boosting,dwork2010continual,chan2011private,smith2017interaction,cormode2019answering}.
}

Privacy guarantees are obtained by adding noise to the per-client computations, resulting in a degradation of utility (i.e., the performance relative to the non-private case). This leads to a tradeoff between privacy and utility. 
For a hierarchical histogram of $b$ bins, we prove a $(1/2)\eps^2$-concentrated differential privacy~\cite{bun2016concentrated} guarantee given a per-client noise of scale $\log b / (\eps \sqrt{n})$ and a quantile error of $\log^2 b / (\eps n)$ up to constants and log factors. 

\myparagraph{Experiments}
We perform numerical experiments using 
neural networks and linear models on tasks including image classification and sentiment analysis based on public datasets. 
The experiments demonstrate the superior performance of \newfl over 
state-of-the-art baselines on the upper quantiles of the error on test clients, with particular improvements on data-poor clients, while being competitive on the mean error.
A deeper analysis reveals that \newfl helps improve performance on data-poor clients. 

We numerically study the privacy-utility tradeoff of the differentially private quantile estimation algorithm described above and the \newfl algorithm with end-to-end differential privacy guarantees. We find that \newfl outperforms FedAvg on the tail error across a wide range of privacy budgets while exhibiting a comparable privacy-utility tradeoff to FedAvg on the mean error.

\subsection{Outline}
We start with \Cref{sec:sfl:related} to describe the related work.
\Cref{sec:sfl:setup} describes the general setup, recalls the FedAvg algorithm, and formally defines the superquantile as a tail summary of a random variable. 
\Cref{sec:sfl:sfl} presents a federated optimization algorithm for \newfl. We analyze its convergence in the convex and non-convex cases, as well as its differential privacy properties in \Cref{sec:sfl:theory}. 
We discuss an extension to other risk measures and relations to fair allocation in \Cref{sec:sfl:discussion}.  
\Cref{sec:sfl:expt} presents experimental results, comparing the proposed approach to existing ones.
Detailed proofs and additional details can be found in the supplement, while the code and the scripts to reproduce the experiments can be found at~\url{https://github.com/krishnap25/simplicial-fl}.

An early version of this work was presented at IEEE CISS~\cite{laguel2021superquantile}. This paper extends and improves upon it in several respects. First, we give an improved and tighter convergence analysis in both the convex and general nonconvex cases. Second, we augment our algorithm with differential privacy and analyze its privacy and utility. Finally, we conduct an expanded numerical study, including 
(a) comparing with baselines such as \term~\cite{li2020tilted} that were published after our paper~\cite{laguel2021superquantile}, 
(b) an empirical comparison to model personalization,
and,
(c) a study of the privacy-utility tradeoff of \newfl under differential privacy.  
 
\myparagraph{Notation}
The norm $\norm{\cdot}$ denote the Euclidean norm $\norm{\cdot}_2$ in $\reals^d$. We use $\Delta^{n-1} = \left\{ \pi \in \reals^n_+ \, :\, \sum_{i=1}^n \pi_i = 1\right\}$ to denote the probability simplex in $\reals^n$.

\section{Related Work} \label{sec:sfl:related}

Federated learning was introduced by~\citep{mcmahan2017communication}
to handle distributed on-client learning~\citep{kairouz2019flsurvey,li2020federated,gafni2021federated}.
A plethora of recent extensions have also been proposed~\cite{yurochin2019bayesian,sattler2020clustered,mills2020communication,wei2020federated,amiri2020machine,shlezinger2021uveqfed,jhunjhunwala2021adaptive,sery2021over,collins2021expoiting}.
Our approach to addressing the statistical heterogeneity by proposing a new objective is broadly applicable in these settings.

Distributionally robust optimization~\citep{bental2013robust},
which aims to train models that perform
uniformly well across all subgroups instead of just on average,
has witnessed a flurry of recent research~\citep{lee2018minimax,duchi2019variance,kuhn2019wasserstein}. This approach is closely related to the risk measures studied in economics and finance~\citep{artzner1999coherent,rockafellar2000optimization,ben2007old,zbMATH06621946}.
The recent works \cite{laguel2020first,levy2020large,curi2020adaptive} study optimization algorithms for risk measures. 
More broadly, risk measures have been successfully utilized in problems ranging from bandits~\cite{sani2012risk,cassel2018general}, 
reinforcement learning~\cite{chow2015risk,tamar2015policy,chow2017risk},
and fairness in machine learning~\cite{willaimson2019fairness,rezaei2021robust}. 
The federated learning method here is based on the superquantile~\cite{rockafellar2002conditional}, a popular risk measure.
We propose a stochastic optimization algorithm adapted to the federated setting and prove its convergence.

Addressing statistical heterogeneity in federated learning has led to two lines of work. 
The first includes algorithmic advances to alleviate the effect of heterogeneity on convergence rates while still minimizing the classical expectation-based objective function of empirical risk minimization. These techniques include the use of proximal terms~\cite{li2020fedprox}, control variates~\cite{karimireddy2020scaffold} or augmenting the server updates~\cite{wang2020tackling,reddi2021adaptive}; we refer to the recent survey~\cite{wang2021field} for details.
More generally, the framework of local SGD has been used to study federated optimization algorithms~\citep{stich2018local,zhou2018convergence,haddadpour2019local,dieuleveut2019communication,li2020convergence_fedavg,khaled2020tighter,koloskova2020unified}.
Compared to these works, which study federated optimization algorithms in the smooth case, we tackle in our analysis the added challenge of nonsmoothness of the superquantile-based objective in both the general nonconvex and strongly convex cases.

The second line of work addressing heterogeneity involves designing new objective functions
by modeling statistical heterogeneity and designing optimization algorithms. The \afl framework to minimize the worst-case error across all training clients and associated generalization bounds were given in~\cite{mohri2019agnostic}. The concurrent work of \cite{li2020fair} proposes the \qffl framework whose objective is inspired by fair resource allocation to minimize the $L^p$ norm of the per-client losses. Several related works were also published following the initial presentation of this work~\cite{laguel2020device}.
A federated optimization algorithm for \afl was proposed and its convergence was analyzed in \cite{deng2020distributionally}. 
Distributional robustness to affine shifts in the data was considered in \cite{reisizadeh2020robust} along with convergence guarantees. Finally, a classical risk measure, namely the entropic risk measure, was considered in \cite{li2020tilted}. We note that no convergence guarantees are currently known for the stochastic optimization algorithms of \cite{li2020fair}. Furthermore, it is unclear if any of these algorithms can be implemented with differential privacy. 

Differential privacy was introduced in \cite{dwork2006calibrating,dwork2006data} to 
formalize the loss of privacy of an individual user in releasing population-level aggregates. 
DP-FedAvg~\cite{mcmahan2018learning}, a differentially private variant of FedAvg, is also implemented in industrial systems~\cite{ramaswamy2020training}. Recent contributions in this direction include differential privacy mechanism compatible with secure aggregation~\cite{kairouz2021distributed,agarwal2021skellam} and improving privacy-utility tradeoffs of federated learning with personalization~\cite{jain2021differentially,bietti2022personalization}.
 
\section{Problem Setup} \label{sec:sfl:setup}

We begin this section by recalling the standard setup of federated learning in \Cref{sec:sfl:setup:fl}. We then describe the standard approach to federated learning and its associated optimization, \fedavg~\cite{mcmahan2017communication} in \Cref{sec:sfl:setup:fedavg}. We then define the superquantile in \Cref{sec:sfl:setup:heterogeneity}.

\subsection{Federated Learning Setup} \label{sec:sfl:setup:fl}

Federated learning consists of heterogeneous 
 clients who collaboratively 
train a machine learning model under the orchestration of a central server. %
The model is then deployed to all clients, including 
those not seen during training. 

Let the vector $w \in \reals^d$ denote the $d$ model parameters.
We assume that each client has a distribution $q$
over some data space such that the data on the client is sampled 
i.i.d. from $q$. 
The loss incurred by the model $w\in\reals^d$ on this client
is $F(w; q) := \expect_{\xi \sim q} [f(w; \xi)]$,
where $f(w;\xi)$ is the chosen loss function, such as the logistic loss, on 
input-output pair $\xi$ under the model $w$. 
The expectation above is assumed to be well-defined and finite.
For a given distribution $q$, smaller values of $F(\cdot; q)$ denote a better fit of the model to the data.  

There are $n$ clients available for training. We number these clients as $1, \ldots, n$ and denote the distribution on training client $i$ by $q_i$. We denote the loss on client $i$ by 
$F_i(w) := F(w ; q_i)$.

The goal of federated learning is to train a model $w$ so that 
it achieves good performance when deployed on {\em each} test client, including those unseen during training. Owing to the statistical heterogeneity of federated learning, the distribution $p$ of a specific test client could be different from the average distribution $(1/n)\sum_{i=1}^n q_i$ that the model is trained on. 

Each federated learning method is characterized by an objective function and the federated optimization algorithm used to minimize it. 
It is not possible to achieve good performance on 
each client {\em simultaneously} with a single model $w$, as it would be a difficult multiobjective optimization problem.
The usual approach is to combine the per-client losses into a scalar and minimize this objective. 
The choice of the objective function and optimization algorithm is primarily determined by the three key aspects of federated learning~\cite{kairouz2019flsurvey,li2020federated}:
\begin{enumerate}[label=(\arabic*),nolistsep,leftmargin=\widthof{ (3) }]
    \item \textit{Communication Bottleneck}: The repeated exchange of massive models between the server and clients over resource-limited wireless networks makes communication a critical bottleneck. Therefore, training algorithms should be able to trade off more local computation for a lower communication cost. 
    \item \textit{Statistical Heterogeneity}: The training distribution $q_i$ and a specific test distribution $p$ are likely to be different from each other. Therefore, a model which works well {\em on average} over all test clients might not work well on {\em each individual} test client.
    \item \textit{Privacy}: The data on each client is highly privacy-sensitive. Federated learning is designed to protect data privacy since no user data is transferred to a data center. This privacy is enhanced by \textit{secure aggregation} of model parameters, which refers to aggregating client updates such that no client update is directly revealed to any other client or the server. This is achieved by cryptographic protocols based on secure multiparty communication~\cite{bonawitz2017practical}.
\end{enumerate}

\subsection{Federated Learning and the~\fedavg algorithm} \label{sec:sfl:setup:fedavg}
Analogous to the classical expectation-based objective function in the empirical risk minimization approach, 
the standard objective in federated learning is to minimize the average loss on the training clients
\begin{align} \label{eq:fl:vanilla-fl-obj}
    \min_{w \in \reals^d} \frac{1}{n}\sum_{i=1}^n F_i(w) + \frac{\lambda}{2} \normsq{w} \,,
\end{align}
where  $\lambda \geq 0$ is a regularization parameter. 
We will call this objective the {\em vanilla FL} objective. 

The de facto standard training algorithm is  \fedavg~\cite{mcmahan2017communication}. Each round of the algorithm consists of the following steps:
\begin{enumerate}[label=(\alph*),nolistsep,leftmargin=\widthof{ (a) }]
    \item The server samples a set $S$ of $m$ clients from $[n]$ and broadcasts the current model $w\pow{t}$ to these clients. 
    \item Staring from $w_{i, 0}\pow{t} = w\pow{t}$, each client $i \in S$ makes $\tau$ local gradient descent steps
    with a learning rate $\gamma$:
    \[
        w_{i, k+1}\pow{t} = w_{i, k}\pow{t} - \gamma \grad F_i(w_{i, k}\pow{t}) \,.
    \]
    \update{In practice, one could also use local stochastic gradient steps, but we restrict ourselves to local full gradient steps for simplicity. }
    \item The models from the selected clients are sent to the server and aggregated to update the server model
    \[
        w\pow{t+1} = \frac{1}{m}\sum_{i \in S}  w_{k, \tau}\pow{t} \,.
    \]
\end{enumerate}

\fedavg addresses the communication bottleneck by using $\tau > 1$ local computation steps as opposed to $\tau = 1$ local steps in minibatch SGD. 
It also securely performs the averaging step (c) to enhance data privacy. However, the vanilla FL objective places a limit on how well statistical heterogeneity can be addressed. By minimizing the average training loss, the resulting model $w$ can sacrifice performance on ``difficult'' clients to perform well on average. 
In other words, it is not guaranteed to perform well on {\em individual} test clients, whose distribution $p$ might be quite different from the average training distribution $(1/n)\sum_{i=1}^n q_i$. 
Our goal in this work is to design an objective function, different from the vanilla FL objective \eqref{eq:fl:vanilla-fl-obj} to better handle statistical heterogeneity and the associated train-test mismatch. We also design a federated optimization algorithm similar to \fedavg to optimize it.

\update{
\subsection{Summarizing the Tail Behavior with the Superquantile}
\label{sec:sfl:setup:heterogeneity}

In this work, we consider clients with  heterogeneous local data distributions $q_1, \ldots, q_n$. 
This data heterogeneity manifests itself as a 
spread over the losses $F_1(w), \ldots, F_n(w)$ for any $w$. In particular, some clients might suffer large losses due to their distributions being far from the average population distribution. 
Our goal is to improve the loss (and hence, predictive performance) on such tail clients whose loss is worse than average. 
In other words, we are concerned with the 
right tail statistics of the empirical distribution over the losses $F_1(w), \ldots, F_n(w)$. 
}

\update{
A natural summary of the right tail of a random variable $Z$ is its high quantiles. 
Recall that the $(1-\theta)$-quantile $Q_\theta(Z)$ of a real-valued random variable $Z$ is defined as 
\[
    Q_\theta(Z) := \inf\left\{\eta \in \reals\,:\, \prob(Z > \eta) \le  \theta \right\}.
\]
Unfortunately, the quantile function of discrete random variables such as the empirical loss distribution is piecewise constant and is not amenable to gradient-based optimization. 
A better-behaved tail summary in this regard is the superquantile, also known as the conditional value at risk (CVaR)~\cite{rockafellar2000optimization,rockafellar2002conditional}.
}

\update{
The superquantile $\supq_\theta(Z)$ of a random variable $Z$ is defined as the average of all quantiles greater than the $(1-\theta)$-quantile:
\begin{equation}\label{eq:supq_formal_def}
    \supq_\theta(Z) = \frac{1}{\theta} \int_{0}^\theta Q_\alpha(Z) \, \D \alpha \,.
\end{equation}
For continuous random variables, we have the equivalence $\supq_\theta(Z) = \expect[Z \, \vert \, Z > Q_\theta(Z)]$ of the superquantile as the \emph{tail mean}, as illustrated in \Cref{fig:sfl:superquantile}. Owing to this interpretation, we refer to the parameter $\theta$ as the \emph{tail threshold}.
}

\update{
Central to our development is the dual expression of the superquantile~\cite{follmer2002convex}:
\begin{align} \label{eq:sq:dual}
\begin{aligned}
    \supq_\theta(a_1, \ldots, a_n)
    &= 
    \max_{\pi \in \Pcal_\theta} \, \pi\T a \,, \\
    \text{where} \quad
    \Pcal_\theta &= \{ \pi \in \Delta^{n-1} \,:\,
        \pi_i \le (\theta n)^{-1} 
        \, \text{ for all } i
        \} \,.
\end{aligned}
\end{align}
Here, 
$\supq_\theta(a_1, \ldots, a_n)$ 
denotes the $(1-\theta)$-superquantile of the 
empirical measure $(1/n)\sum_{i=1}^n \delta_{a_i}$
and $\Delta^{n-1}$ is the probability simplex in $\reals^n$.
The discrete superquantile is thus the support function of the polytope $\Pcal_\theta$, which is illustrated in \Cref{fig:sfl:superquantile}.
Not only is the discrete superquantile a continuous function of its inputs (unlike the quantile function), but it is also convex as  it is 
the maximum of a family of linear functions in 
the expression~\eqref{eq:sq:dual}. 
}

\section{Handling Heterogeneity with \newfl} \label{sec:sfl:sfl}
\begin{figure}[t]
  \begin{center}
    \includegraphics[width=0.5\linewidth]{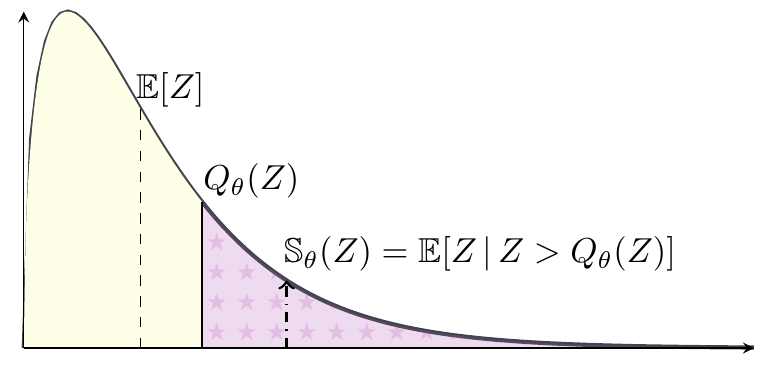}
    \includegraphics[width=0.45\linewidth]{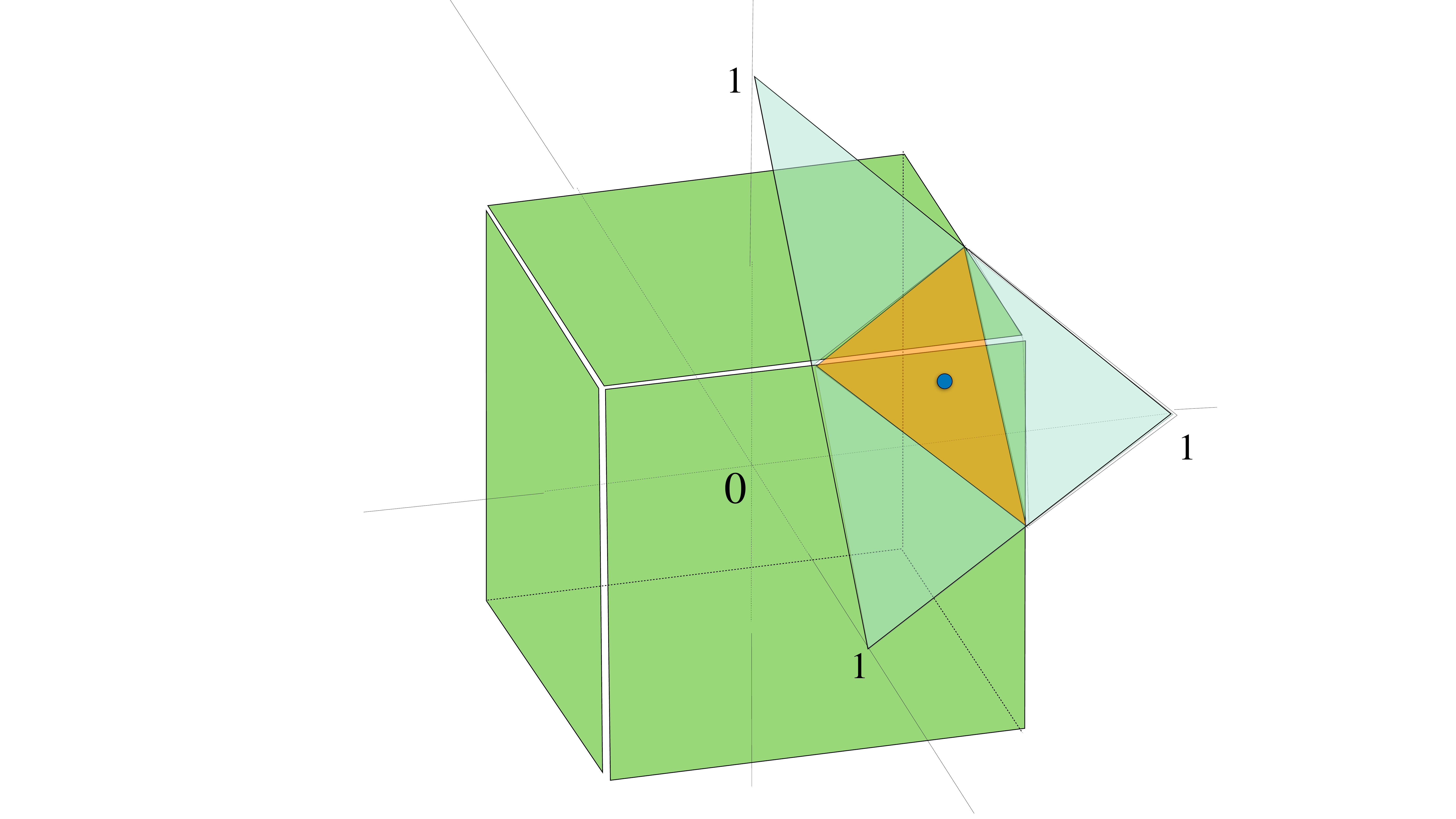}
  \end{center}
  \caption{\small{
  \textbf{Left}: $\!(1\!-\!\theta)$-quantile $Q_\theta(Z)$ and super\-quantile $\supq_\theta(Z)$ of a continuous r.v.\;$Z$.
    \update{\textbf{Right}: The set of feasible mixture weights $\pi=(\pi_1, \pi_2, \pi_3) \in \mathcal{P}_\theta$ in the dual formulation~\eqref{eq:sq:dual} is given by the intersection of the box constraints $0 \le \pi_i \le (3\theta)^{-1}$ for $i=1, 2, 3$, with the simplex constraint $\pi_1 + \pi_2 + \pi_3 = 1$.}
}} \label{fig:sfl:superquantile}
\end{figure}

In this section, we introduce the \newfl framework in \Cref{sec:sfl:framework} and propose an algorithm to optimize in the federated setting in \Cref{sec:sfl:algo}. 

\subsection{The \newfl Framework} \label{sec:sfl:framework}
\update{
The \newfl framework aims to improve the performance of the tail clients by minimizing 
the superquantile of the loss distribution.
Given a discretization
$\{\theta_1, \ldots, \theta_r\}$ of $(0, 1]$, 
\newfl maintains $r$ models $w_1, \ldots, w_r$, one for 
each \update{tail threshold} 
$\theta_j$. 
We allow each test client to select the best model $w \in \{w_1, \ldots, w_r\}$, according to its local data.
Recall the schematic in \Cref{fig:sfl:schematic_of_framework} for an illustration.

For a given tail threshold $\theta$, we propose to minimize the $(1-\theta$)-superquantile of the distributions of losses:
\begin{align} \label{eq:method:obj:supq}
\begin{aligned}
    \min_{w \in \reals^d} \left[
        F_\theta(w) :=  \supq_\theta \big(F_1(w), \ldots, F_n(w) \big) + \frac{\lambda}{2}\normsq{w}
    \right]\,.
\end{aligned}
\end{align}

The objective \eqref{eq:method:obj:supq} focuses on poor-performing clients --- specifically those with performance worse than the $(1-\theta)$-quantile of the distribution of losses $(F_1(w), \ldots, F_n(w))$.
In contrast, the vanilla FL objective optimizes $(1/n) \sum_{i=1}^n F_i(w) + \lambda/2 \normsq{w}$, which is $\lim_{\theta\to 1} F_\theta(w)$; this equally weights all clients involved in training.
At the other extreme $\theta \to 0$, we recover the worst-case loss over all clients. 
}

\update{
\myparagraph{Distributionally Robust Interpretation}
We have the following dual characterization of \newfl as a distributionally robust learning objective, as a consequence of the dual representation \eqref{eq:sq:dual} of the superquantile.
\begin{property}\label{property:main:duality}
The \newfl objective \eqref{eq:method:obj:supq} 
can also be written as 
\begin{align} \label{eq:method:obj:minmax}
\begin{aligned}
        F_\theta(w) &= \max_{\pi \in \mathcal{P}_\theta} \,\, \sum_{i=1}^n \pi_i F_i(w)\,, \\
    \text{where},\quad 
    \mathcal{P}_\theta &:= \left\{\pi \in\Delta^{n-1}\,: \, 
        \pi_i \le (n\theta)^{-1}\, \text{for all } i \in [n]\,\right\}.
\end{aligned}
\end{align}
\end{property}
This reformulation shows that \newfl can be interpreted as a distributionally robust variant of the vanilla FL objective: 
since $\sum_{i=1}^n \pi_i F_i(w) = F(w; p_\pi)$ is loss of $w$ on the mixture $p_\pi = \sum_{i=1}^n \pi_i q_i$ of the training distributions $q_1, \ldots, q_n$, we 
get that \newfl aims to minimize the 
worst-case loss over all mixtures $p_\pi $ subject to the constraint that $\pi_i \le (n\theta)^{-1}$.

This formulation also reveals two important properties of the \newfl objective. First, we note that the objective $F_\theta$, as a max function. is convex whenever the losses $F_i$ are convex. 
Second, it is a non-smooth function, with the non-smoothness stemming from the maximum over the polytope $\mathcal{P}_\theta$ (cf. \Cref{fig:sfl:superquantile}). These two properties will play important role in the convergence analysis of our federated algorithm in \Cref{sec:sfl:convergence}.
}

\begin{algorithm}[t]
	\caption{The \newfl Algorithm}
	\label{algo:sfl-private}
\begin{algorithmic}[1]
		\Require Initial iterate $w\pow{0}$,
		    number of communication rounds $T$, 
		    number of clients per round $m$, 
		    number of local updates $\tau$,
		    local step size $\gamma$
	    \For{$t=0, 1, \ldots, T-1$}
	        \State Sample $m$ clients from $[n]$ without replacement in $S$ \label{line:sfl-private:sample}
	        \State Estimate the $(1-\theta)$-quantile of 
	            $F_i(w\pow{t})$ for $i \in S$ with distributed differential privacy (\Cref{algo:sfl:quantile-dp:new}); call this $Q\pow{t}$ 
	        \label{line:sfl-private:quantile}
	        \For{each selected client $i \in S$ in parallel}
	            \State Set $\tilde \pi_i\pow{t} = \, \mathbb{I}\left(F_i(w\pow{t}) \ge Q\pow{t}\right)$
	            \label{line:sfl-private:reweigh}
	            \State Initialize $w_{k, 0}\pow{t} = w\pow{t}$
	    	    \For{$k=0, \ldots, \tau-1$} \label{line:sfl-private:local}
	    	        \State $w_{i, k+1}\pow{t} = (1-\gamma\lambda)w_{i, k}\pow{t} - \gamma \grad F_i(w_{i, k}\pow{t})$
	    	    \EndFor
	    	\EndFor
	    	\State $w\pow{t+1} = \sum_{i \in S} \tilde\pi\pow{t}_i w_{i, \tau}\pow{t} /  \sum_{i \in S} \tilde\pi_i\pow{t}$ \label{line:sfl-private:aggregation}
	    \EndFor
	    \State \Return $w_T$
\end{algorithmic}
\end{algorithm}

\subsection{Federated Optimization for \newfl } \label{sec:sfl:algo}
We now propose a federated optimization algorithm for the \newfl objective \eqref{eq:method:obj:supq}. While there could be many approaches to optimizing \eqref{eq:method:obj:supq}, 
we consider algorithms similar to \fedavg for their ability to avoid communication bottlenecks and preserve the privacy of user data. 
Owing to the tail mean interpretation of the superquantile (\Cref{fig:sfl:superquantile}), a natural algorithm to minimize it first evaluates the loss on all the clients and only performs gradient updates on those clients in the tail above the $(1-\theta)$-quantile.
However, since a practical algorithm cannot assume that all the clients are available at a given time, we perform the same operation on a subsample of clients. 

The optimization algorithm for the \newfl objective~\eqref{eq:method:obj:supq} is given in  \Cref{algo:sfl-private}. It has the following four steps:
\begin{enumerate}[label=(\alph*),nolistsep,leftmargin=\widthof{ (a) }]
    \item \textit{Model Broadcast} (\cref{line:sfl-private:sample}): The server samples a set $S$ of $m$ clients from $[n]$ and sends the current model $w\pow{t}$.
    
    \item \textit{Quantile Computation and Reweighting} (lines~\ref{line:sfl-private:quantile} and~\ref{line:sfl-private:reweigh}): 
    Selected clients $i \in S$ 
    and the server collaborate to estimate the $(1-\theta)$-quantile of the losses $F_i(w\pow{t})$ with differential privacy.
    The clients then update their weights to be zero if their loss is smaller than the estimated quantile and leave them unchanged otherwise. 
    This ensures that model updates are only aggregated from the tail clients; cf. \Cref{fig:sfl:superquantile}.
    
    \item \textit{Local Updates} (loop of line~\ref{line:sfl-private:local}): Staring from $w_{k, 0}\pow{t} = w\pow{t}$, each client $i \in S$ makes $\tau$ local gradient or stochastic gradient descent steps
    with a learning rate $\gamma$.
    
    \item \textit{Update Aggregation} (line~\ref{line:sfl-private:aggregation}): The models from the selected clients are sent to the server and aggregated to update the server model, with weights from  line~\ref{line:sfl-private:reweigh}).
\end{enumerate}
Compared to \fedavg, \newfl has the additional step of computing the quantile and new weights $\tilde \pi\pow{t}_i$ for each selected client $i \in S$ in lines~\ref{line:sfl-private:quantile} and~\ref{line:sfl-private:reweigh}. 
Let us consider \newfl in relation to the three key aspects of federated learning we introduced in \Cref{sec:sfl:setup:fl}. 
\begin{enumerate}[label=(\arabic*),nolistsep,leftmargin=\widthof{ (3) }]
    \item \textit{Communication Bottleneck}: Identical to \fedavg, \newfl algorithm performs multiple computation rounds per communication round.
    \item \textit{Statistical Heterogeneity}: 
    \update{The \newfl objective is designed to optimize the tail mean of the per-client loss distribution as formalized by the superquantile. The vanilla FL objective, in contrast, is oblivious to performance disparities across clients.}
    \item \textit{Privacy}: Identical to \fedavg, \newfl does not require any data transfer, and the aggregation of line~\ref{line:sfl-private:aggregation} can be securely performed using secure multiparty communication. 
    The extra step of quantile computation is also performed with distributed differential privacy, as we describe next. 
\end{enumerate}

\myparagraph{Quantile Estimation with Distributed Differential Privacy}
The na\"ive way to compute the quantile of the per-client losses in line~\ref{line:sfl-private:quantile} of \Cref{algo:sfl-private} is to have the clients send their losses to the server. To avoid the privacy risk of leakage of information about the clients to the server, we compute the quantile with distributed differential privacy~\cite{kairouz2021distributed} using the discrete Gaussian mechanism~\cite{canonne2020discrete}. 
The key idea behind differential privacy~\cite{dwork2006data,dwork2016calibrating} is 
to ensure that the addition or removal of the data from one client does not lead to a substantial change in the output of an algorithm. A significant difference in the output would give a privacy adversary enough signal to learn about the client who was added or removed.

\begin{algorithm}[t]
    \caption{Quantile Computation with Distributed Differential Privacy}
    \label{algo:sfl:quantile-dp:new}
\begin{algorithmic}[1]
        \Require Ring size $M$, set $S$ of $m = \vert S \vert$ clients where each client $i$ has a scalar $\ell_i \in [0, B]$, 
        target quantile $1-\theta \in (0, 1)$, 
        discretization $l_0,l_1, \ldots, l_b$ of $[0, B]$,
        variance proxy $\sigma^2$, scaling factor $c \in \ZZ_+$

        \State Each client $i$ computes a hierarchical histogram
         $x_i(r, j) = \mathbb{I}\left(  
            l_{2^r(j-1) + 1} \le F_i(w) <  l_{2^r j}
         \right)$ for $j=1, \ldots, b/2^r$ and $r = 0, \ldots, \log_2 b -1$
        
        \State Each client $i$ samples $\xi_i(r, j) \sim \Ncal_{\ZZ}(0, \sigma^2)$ i.i.d. and 
        sets $\tilde x_i(r, j) = \big(c x_i(r, j) + \xi_i(r, j) \big) \mod M$ for each $r, j$
        
        \State Compute $s = (\sum_{i \in S} \tilde x_i)\mod M$ securely
        
        \State Set hierarchical histogram $\hat h = s/c$ and define for $j \in [b]$ its cumulative sum $\hat H(j) = \sum_{(r, o) \in P_j} \hat h(r, o)$ using a maximal dyadic partition $P_j$ of $[1, j]$
        \State \Return Quantile estimate $l_{j_\theta^*(\hat h)}$ 
            corresponding to index $j_\theta^*(\hat h)$; cf. Eq.~\eqref{eq:sfl-privacy:quantile}
\end{algorithmic}
\end{algorithm}

Distributed differential privacy simulates a trusted central aggregator by using a secure summation oracle~\cite{bonawitz2017practical}, 
which enables the computation of summations $\sum_{i \in S} v_i$ where $v_i \in \reals^d$ is a privacy-sensitive vector residing with client $i$.
Practical implementations of such algorithms are based on cryptographic techniques such as secure multiparty computation~\cite{evans2018pragmatic}, which requires each component of the vectors $v_i$ to be discretized to the ring $\ZZ_M$ of integers modulo $M$.
We abstract out the details of the secure summation oracle and only require that it returns the sum $\left( \sum_{i \in S} x_v \right) \mod M$ without revealing any further information to a privacy adversary.

We assume that the losses are bounded as $F_i(w) \in [0, B]$ for each $i \in S$, and that we are given 
$b$ bin edges $0 \le l_0 < l_1 < \cdots < l_b = B$.
We aim to construct a hierarchical histogram $h$ 
that maintains the number of clients not only in every single bin but also in groups of bins organized as a binary tree.
Concretely, $h(r, j)$ maintains the number of clients whose losses lie between the bin edges $l_{2^r(j-1) + 1}$ and $l_{2^r j}$
for index $j=1, \ldots, b/2^r$ and level $r = 0, \ldots, \log_2 b -1$.%
\footnote{
We assume for simplicity that $b$ is a power of $2$ so that $\log_2 b$ is an integer.
}
The lower levels $r=0$ and $r=1$ correspond respectively to individual bins and pairs of bins, while the topmost level $r= \log_2 b - 1$ refers to two groups: the first $b/2$ bins and the last $b/2$ bins. 
We skip the topmost level in the tree because the count at this node is the publicly known number $m = \vert S \vert$ of clients. 
The hierarchical histogram method, also known as tree aggregation, is a classical technique to answer range queries and in cumulative distribution estimation~\cite{hay2010boosting,dwork2010continual,chan2011private,smith2017interaction}.

Our algorithm is given in \Cref{algo:sfl:quantile-dp:new}. 
Each client $i$ first computes its local hierarchical histogram $x_i$ as
\[
    x_i(r, j) = \mathbb{I}\left(  
        l_{2^r(j-1) + 1} \le F_i(w) <  l_{2^r j}
    \right) \,,
\]
such that the overall hierarchical histogram can be obtained as $h = \sum_{i \in S} x_i$. 
To enforce differential privacy, each client then adds random discrete Gaussian\footnote{
    See \Cref{sec:a:sfl:privacy} for a formal definition.
} noise $\xi_i \sim \Ncal_{\ZZ}(0, \sigma^2 I)$ with scale parameter $\sigma^2$ and of appropriate dimension.
These noisy $\tilde x_i$'s are summed up 
using a secure summation oracle so that
the server receives an approximate hierarchical histogram $\hat h$ which approximates the true histogram $h = \sum_{i \in S} x_i$. 
With slight abuse of notation, we still refer to $\hat h$ as a hierarchical histogram, although it could have negative entries and
could be inconsistent, i.e., 
the count $\hat h(r, j)$ at a node might not equal 
the sum $\hat h(r-1, 2j-1) + \hat h(r-1, 2j)$ of counts at its children nodes.

The final step is to define and return an appropriate notion of a $(1-\theta)$-quantile of the approximate histogram $\hat h$. 
A non-negative hierarchical histogram $h$ can be viewed as a random variable $Z$ with 
(scaled) cumulative distribution function $H(j) = m\, \prob\big( Z \le l_{j} \big) = h(0, 1) + \ldots + h(0, j)$, from which we can estimate the quantile. 
We can obtain a greater utility under differential privacy by expressing the cumulative distribution function
$H(j)$ of this random variable $Z$ by using nodes higher up in the tree. 
Concretely, using a maximal dyadic partition $P_j$ of the range $[1, j]$, we have $H(j) = \sum_{(r, o) \in P_j} h(r, o)$ from summing up $\vert P_j \vert \le \log_2 b$ terms. 
For instance, the dyadic partition for $j=15$ is $P_{15} = [1, 8] \cup [9, 12] \cup [13, 14] \cup [15]$, where the counts of each range on the right side can be obtained from an intermediate node in the hierarchical histogram $h$. 

With this definition of the cumulative mass $H(j)$, we define $(1-\theta)$-quantile of the hierarchical histogram $h$ as the quantile function of this induced random variable $Z$:
\[
    Q_\theta(H) := Q_\theta(Z) = 
    \min_{j \in [b]} \Big\{ l_{j} \,: \,
       H(j) > (1-\theta) m
    \Big\} \,.
\]

\noindent 
Similarly, for approximate hierarchical histograms $\hat h$ that are inconsistent and allow for negative values, 
we define the cumulative function $\hat H(j) = \sum_{(r, o) \in P_j} \hat h(r, o)$ from a maximal dyadic partition $P_j$ of $[1, j]$. 
As an estimate of the quantile, 
we return the bin edge $l_j$ such that the estimated cumulative mass $\hat H(j)$ is as close to $1-\theta$ as possible: 
\begin{align} \label{eq:sfl-privacy:quantile}
\begin{aligned}
    Q_\theta(\hat h) &:= l_{j^*_\theta(\hat h)}
    \quad \text{where} \quad
    j_\theta^*(\hat h) &= \argmin_{j \in [b]}
    \big\vert 
       \hat H(j) - (1-\theta) m 
    \big\vert \,.
\end{aligned}
\end{align}

\section{Theoretical Analysis} \label{sec:sfl:theory}
In this section, we analyze the convergence analysis of 
\newfl (\Cref{sec:sfl:convergence})
and study the differential privacy properties of the quantile computation (\Cref{sec:sfl:privacy}).

\subsection{Convergence Analysis} \label{sec:sfl:convergence}
We study the convergence of \Cref{algo:sfl-private} with respect to the objective~\eqref{eq:method:obj:supq}
in two cases: (i) the general non-convex case, and (ii) when each $F_i(w)$ is convex. 

\myparagraph{Assumptions}
We make some assumptions on the per-client losses $F_i$, which are assumed to hold throughout this section. For each client $i \in [n]$, the objective $F_i$ is 
\begin{enumerate}[label=(\alph*),nolistsep,leftmargin=\widthof{(a) }]
\item $B$-bounded, i.e., $0 \le F_i(w) \le B$ for all $w \in \reals^d$,
\item $G$-Lipschitz, i.e., $\vert F_i(w) - F_i(w')\vert  \le G\norm{w-w'}$ for all $w, w' \in \reals^d$, and, 
\item $L$-smooth, i.e., $F_i$ is continuously differentiable and its gradient $\grad F_i$ is $L$-Lipschitz.
\end{enumerate}

\begin{algorithm}[t]
	\caption{The \newfl Algorithm with Exact Reweighting}
	\label{algo:sfl-new}
\begin{algorithmic}[1]
		\Require Same as \Cref{algo:sfl-private}
	    \For{$t=0, 1, \ldots, T-1$}
	        \State Sample $m$ clients from $[n]$ without replacement in $S$\label{line:sfl-new:sample}
	        \State Compute $\pi\pow{t} = \argmax_{\pi \in \mathcal{P}_{\theta, S}} \sum_{i \in S} \pi_i F_i(w\pow{t})$
	        \label{line:sfl-new:reweigh}
	        \For{each selected client $i \in S$ in parallel}
	        \label{line:sfl-new:local}
	            \State Initialize $w_{i, 0}\pow{t} = w\pow{t}$
	    	    \For{$k=0, \ldots, \tau-1$}
	    	        \State $w_{i, k+1}\pow{t} = (1-\gamma\lambda)w_{i, k}\pow{t} - \gamma \grad F_i(w_{i, k}\pow{t})$
	    	    \EndFor
	    	\EndFor
	    	\State $w\pow{t+1} = \sum_{i \in S} \pi_i\pow{t} w_{i, \tau}\pow{t}$ \label{line:sfl-new:aggregation}
	    \EndFor
	    \State \Return $w_T$
\end{algorithmic}
\end{algorithm}

\myparagraph{Equivalent Algorithm}
\Cref{algo:sfl-private} is not amenable to theoretical analysis as it is stated because the quantile function of discrete random variables computed in line~\ref{line:sfl-private:quantile} is piecewise constant and discontinuous. To overcome this obstacle, we introduce a near-equivalent algorithm in \Cref{algo:sfl-new}, 
which replaces the reweighting step of \Cref{algo:sfl-private} (lines~\ref{line:sfl-private:quantile} and~\ref{line:sfl-private:reweigh}) with the ideal reweighting suggested by the dual representation of~\eqref{eq:method:obj:minmax}. 

Let us start with the case of $S = [n]$. 
Our first observation shows that 
the weights $\pi\pow{t}$ that attain the maximum over $\pi$ in the objective \eqref{eq:method:obj:minmax} can be used to construct a subgradient of $F_\theta$ in the general nonconvex case --- this will eventually allow us to derive convergence guarantees. 
\begin{property} \label{prop:sfl:subdiff-property}
    Fix a $w \in \reals^d$ and let $\pi^\star \in \argmax_{\pi \in \Pcal_\theta} \sum_{i=1}^n \pi_i F_i(w)$. Then, we have, 
    \[
        \sum_{i=1}^n \pi_i^\star F_i(w) + \lambda w \in \partial F_\theta(w) \,,
    \]
    where $\partial F_\theta(w)$ denotes the regular subdifferential of $F_\theta$. 
\end{property}
\begin{proof}
    Let $h_\theta(a) := \max_{\pi \in \Pcal_\theta} \pi\T a$ denote the support function of the polytope $\Pcal_\theta$, 
    and let $g_n(w) = (F_1(w), \ldots, F_n(w))$ denote the concatenation of the losses into a vector.
    Then, $F_\theta(w) = h_\theta \circ g_n (w) + (\lambda/2)\normsq{w}$. 
    Since $h_\theta$ is convex, we get that its (convex) subdifferential~\cite[e.g.,][Cor.\;4.4.4]{hiriart1996convex} is 
    \[
        \partial h_\theta(a) = \argmax_{\pi \in \Pcal_\theta} \pi\T a \,. 
    \]
    Since $g_n$ is smooth and $h_\theta$ is convex with full domain, we
    obtain the regular subdifferential of $h_\theta \circ g_n$ by
    the chain rule~\cite[][Thm. 10.6]{rockafellar2009variational} as
    \[
        \partial (h_\theta \circ g_n) = \grad g_n(w) \partial h_\theta\big( g_n(w) \big) \,,
    \]
    where $\grad g_n(w) \in \reals^{d \times n}$ is the transpose of the Jacobian matrix of $g_n$. 
    We can handle the regularization by absorbing it into the superquantile by defining $\tilde F_i(w) = F_i(w) + (\lambda/2)\normsq{w}$.
\end{proof}

\Cref{algo:sfl-new} extends this intuition to the setting where only a subsample $S \subset [n]$ of clients are available in each round. 
We define the counterpart of the constraint set $\mathcal{P}_\theta$ from \eqref{eq:method:obj:minmax} 
defined on a subset $S \subset [n]$ of $m$ clients as:
\begin{align} \label{eq:sfl:constraint-subsample}
    \mathcal{P}_{\theta, S} = \left\{\pi \in \Delta^{\vert S\vert -1} \, :\, 
    \pi_i \le \frac{1}{\theta m},\text{ for } i \in S
    \right\} \,,
\end{align}
where 
we denote $(\pi_i)_{i \in S} \in \reals^{\vert S\vert }$ by $\pi$ with slight abuse of notation.
With this notation, \Cref{algo:sfl-new} computes the new weights of the clients as 
\[
    \pi\pow{t} = \argmax_{\pi \in \mathcal{P}_{\theta, S}} \sum_{i \in S} \pi_i F_i(w\pow{t}) \,.
\]

We now analyze how close \Cref{algo:sfl-new} is to 
\Cref{algo:sfl-private}. 
Let $Z(w)$ be a discrete random variable which takes the value $F_i(w)$ with probability $1/n$ for $i = 1,\ldots, n$, 
and let $Q_\theta(Z(w))$ denote its $(1-\theta)$-quantile. 
The weights $\hat \pi \in \Delta^{n-1}$ considered in \Cref{algo:sfl-private} (assuming that $Q\pow{t}$ is the exact quantile of $\{F_i(w\pow{t}) \, :\, i \in S\}$) are given by a
hard-thresholding based on whether $F_i(w)$ is larger than its $(1-\theta)$-quantile:
\begin{align} \label{eq:sfl-reweight-via-quantile}
    \tilde \pi_i =
        \mathbb{I}\big(F_i(w) \ge Q_\theta(Z(w))\big)\,,
    \quad
    \text{and}, \quad
    \hat \pi_i = \frac{\tilde \pi_i}{\sum_{i'=1}^n \tilde \pi_{i'}}.
\end{align}
The objective defined by these weights is $\hat F_\theta(w) = \sum_{i=1}^n \hat \pi_i F_i(w) + (\lambda/2) \normsq{w}$.
The next proposition shows that $\hat F_\theta(w) = F_\theta(w)$ under certain conditions, or is a close approximation, in general.  
\begin{proposition} \label{prop:sfl-quantile}
    Assume  $F_1(w) < \cdots < F_n(w)$
    and let $i^\star = \lceil \theta n \rceil$. 
    Then, we have,
    \begin{enumerate}[label=(\alph*),nolistsep,leftmargin=\widthof{(a) }]
        \item $\pi^\star = \argmax_{\pi \in \Pcal_\theta} \sum_{i=1}^n \pi_i F_i(w)$ is unique,
        \item $Q_\theta(Z(w)) = F_{i^\star}(w)$,
        \item if $\theta n$ is an integer, then
            $\hat \pi = \pi^\star$ so that $\hat F_\theta(w) = F_\theta(w)$, and,
        \item if $\theta n$ is not an integer, then
            \[
                0 \le F_\theta(w) - \hat F_\theta(w) 
                \le \frac{B}{\theta n}  \,.
            \]
    \end{enumerate}
\end{proposition}
\begin{proof}
    We assume w.l.o.g. that $\lambda = 0$.
    We apply the property that the superquantile is a tail mean (cf. \Cref{fig:sfl:superquantile}) for discrete random variables~\citep[][Proposition 8]{rockafellar2002conditional} to get
    \[
        F_\theta(w) =  \frac{1}{\theta n} \sum_{i=i^\star+1}^n  F_i(w)  + \left( 1- \frac{\lfloor \theta n \rfloor }{\theta n} \right) F_{i^\star}(w) \,.
    \]
    Comparing with dual representation \eqref{eq:method:obj:minmax}, this gives a closed-form expression for $\pi^\star$, which is unique because 
    $F_{i^\star - 1}(w) < F_{i^\star}(w) < F_{i^\star + 1}(w)$.
    For (b), note that 
    $Q_\theta(Z(w)) = \inf\{\eta \in \reals \, : \, \prob(Z(w) > \eta) \le  \theta \}$ equals
    $F_{i^\star}(w)$ by definition of $i^\star$.
    Therefore, if $\theta n$ is an integer,
    $\pi^\star$ coincides exactly with $\hat \pi$.
    When $\theta n$ is not an integer, we have 
    \[
        \hat F_\theta(w) = \frac{1}{n - i^\star + 1} \sum_{i=i^\star}^n F_i(w)  \,.
    \]
    The bound on $\hat F_\theta(w) - F_\theta(w)$ follows from
    elementary manipulations together with 
    $0 \le F_i(w) \le B$. 
\end{proof}

In our context where we sample $m$ clients per round,
\Cref{prop:sfl-quantile} holds for each round.
In particular, part (c) of \Cref{prop:sfl-quantile} states that when $\theta m$ is an integer, the weights $\pi^\star$ computed as an exact argmax in \Cref{algo:sfl-new} are identical to the weights $\hat \pi$ in
\Cref{algo:sfl-private} where line~\ref{line:sfl-private:reweigh} exactly computes the quantile of the per-client losses. 
We record another consequence of \Cref{prop:sfl-quantile}, namely, that the reweighting $\pi\pow{t}$ is sparse. 
\begin{remark} \label{remark:sfl-client-filtering}
\Cref{prop:sfl-quantile} shows that \newfl's reweighting $\pi\pow{t}$ (line~\ref{line:sfl-new:reweigh} of \Cref{algo:sfl-new}) is sparse. That is, $\pi_i\pow{t}$ is non-zero only for exactly $\lceil \theta m\rceil$ clients with the largest losses. 
\end{remark}

\myparagraph{Bias due to Partial Participation}
Note that the dual representation \eqref{eq:method:obj:minmax} is the maximum over all distributions in $\Pcal_\theta$, but \Cref{algo:sfl-new} and \Cref{algo:sfl-private} only maximize the weights over a set $S$ of $m$ clients in each round (line~\ref{line:sfl-new:reweigh}). Therefore, the updates performed by \Cref{algo:sfl-new} are not unbiased. 
To formalize this, define the objective 
\begin{align*}
    \overline F_\theta(w) := \expect_{S \sim U_m}\left[ F_{\theta, S}(w) \right],
    \quad \text{where }
    F_{\theta, S}(w) = \max_{\pi \in \Pcal_{\theta, S}} \sum_{i \in S} \pi_i F_i(w) + \frac{\lambda}{2}\normsq{w}
\end{align*}
is the analogue of  \eqref{eq:method:obj:minmax} defined on a sample $S \subset [n]$ of clients, and $U_m$ is the uniform distribution over subsets of $[n]$ of size $m$. 
In each step, \Cref{algo:sfl-new} approximates the subgradients of $F_{\theta, S}$. Indeed, \Cref{prop:sfl:subdiff-property} gives 
\begin{align} \label{eq:sfl-ncvx:subdiff-partial}
    \sum_{i\in S} \pi_i\pow{t} F_i(w\pow{t}) + \lambda w\pow{t} \in \partial F_{\theta, S}(w\pow{t}) \,.
\end{align}
In expectation, \Cref{algo:sfl-new} therefore takes subgradient steps for $\overline F_\theta$ --- this introduces a bias when compared to the original $F_\theta$ that we would like to optimize.
Fortunately, this bias can be bounded as~\cite[Prop. 1]{levy2020largescale}
\begin{align} \label{eq:sfl-main:bias-bound:no-smooth}
    \sup_{w \in \reals^d} \left\vert\overline F_{\theta}  (w)
    - F_{\theta}  (w)\right\vert \le \frac{B}{\sqrt{\theta m}} \,.
\end{align}
Our analysis strategy will be to study the convergence (near-stationarity or near-optimality) in terms of the objective $\overline F_\theta$ which \Cref{algo:sfl-new} actually minimizes, and then translate that to a convergence result on the original objective $F_\theta$ using the bound \eqref{eq:sfl-main:bias-bound:no-smooth}. 

\myparagraph{Convergence: Nonconvex Case}
We start with the convergence analysis in the nonconvex case with no regularization (i.e., $\lambda = 0$). 
Since $\overline F_\theta$ is nonsmooth and nonconvex, we 
state the convergence guarantee in terms of the Moreau envelope of $\overline F_\theta$~\cite{hiriart1996convex} following the idea of~\cite{drusvyatskiy2019efficiency,davis2019stochastic}.
Given a parameter $\mu > 0$, we define the Moreau envelope of $\overline F_\theta$ as 
\begin{align} \label{eq:sfl:moreau-env}
    \overline \Phi_\theta^\mu(w) = \inf_{z \in \reals^d} \left\{
    \overline F_\theta(z) + \frac{\mu}{2} \normsq*{w-z} 
    \right\} \,.
\end{align}
The Moreau envelope satisfies several remarkable properties for $\mu > L$~\cite[][Lemma 4.3]{drusvyatskiy2019efficiency}.  
First, it is well-defined, and the infimum on the right-hand side admits a unique minimizer, called the proximal point of $w$, and denoted $\prox_{\overline F_\theta/\mu}(w)$.
Second, the Moreau envelope is continuously differentiable
with $\grad \overline \Phi_\theta^\mu(w) = \mu(w - \prox_{\overline F_\theta/\mu}(w))$.
Finally,
the stationary points of $\overline \Phi_\theta^\mu$ and $\overline F_\theta$ coincide. 
Interestingly, 
the bound $\norm{\grad \overline \Phi^\mu_\theta(w)} \le \epsilon$ directly implies a near-stationarity on $\overline F_\theta$, and hence the original $F_\theta$, in the following variational sense: the proximal point $z=  \prox_{\overline F_\theta/\mu}(w)$
satisfies~\cite[Sec. 4.1]{drusvyatskiy2019efficiency}:
\begin{enumerate}[label=(\alph*),nolistsep,leftmargin=\widthof{(a) }]
    \item $z$ is close to $w$; that is, $\norm{z - w} \le \eps / \mu$,
    \item $z$ is nearly stationary on $\overline F_\theta$; that is $\mathrm{dist}\left(0, \partial \overline F_\theta(z)\right) \le \epsilon$, where $\partial\overline F_\theta$ refers
    to the regular subdifferential, and, 
    \item $\overline F_\theta$ is uniformly close to $F_\theta$ as per \eqref{eq:sfl-main:bias-bound:no-smooth}.
\end{enumerate}
Thus, we state the convergence guarantee of our algorithm in the nonsmooth nonconvex case in terms of $\overline \Phi^\mu_\theta$ (although it never appears in the algorithm).

\begin{theorem} \label{thm:sfl:nonsmooth-nonconvex:main}
    Let the number of rounds $T$ be fixed and set $\mu = 2L$.
    Denote $\Delta F_0 = F_\theta(w\pow{0}) - \inf \overline F_\theta$.
    Let $\hat w$ denote a uniformly random sample from the sequence $\big(w\pow{0}, \ldots, w\pow{T-1}\big)$ produced by \Cref{algo:sfl-new}.
    Then, there exists a learning rate $\gamma$ depending on the number of rounds $T$ and problem parameters $\tau, L, G, \Delta F_0$ such that 
    \[
        \expect\normsq*{\grad \overline \Phi^\mu_\theta(\hat w)} 
        \le 
        \sqrt{\frac{\Delta_0 LG^2}{T}}
        + (1-\tau^{-1})^{1/3} \left(\frac{\Delta_0 L G }{T}\right)^{2/3}
        + \frac{\Delta_0 L}{T}\,.
    \]
\end{theorem}
\begin{proof}[Proof Sketch]
Let $z\pow{t} = \prox_{\overline F_\theta/\mu}(w\pow{t})$ be the proximal point of $w\pow{t}$.
We expand out the recursion $w\pow{t+1} = w\pow{t} - \gamma \sum_{i\in S} \pi_i\pow{t} \sum_{k=0}^{\tau-1} \grad F_i(w_{i, k}\pow{t})$ to get 
\begin{align*} \nonumber
    \overline \Phi_\theta^\mu&(w\pow{t+1})
    \le \overline F_\theta(z\pow{t}) + \frac{\mu}{2} \normsq*{z\pow{t} - w\pow{t+1}} \\ \nonumber
    &
    \begin{aligned}
    =  \, \overline F_\theta(z\pow{t}) + \frac{\mu}{2} \normsq*{z\pow{t} - w\pow{t}} 
        &+ \mu \gamma\inp*{z\pow{t} - w\pow{t}}{\sum_{i \in S}\pi_i\pow{t} \sum_{k=0}^{\tau-1} \grad F_i(w_{i, k}\pow{t})} 
        \\&+ \frac{\mu\gamma^2}{2} \normsq*{\sum_{i \in S} \pi_i\pow{t} \sum_{k=0}^{\tau-1} \grad F_i(w_{i, k}\pow{t})} 
    \end{aligned}
    \\
    &= \overline \Phi_\theta^\mu(w\pow{t}) + \Tcal_1 + \Tcal_2 \,.
\end{align*}
The term $\Tcal_1$ which carries a $O(\gamma)$-coefficient controls the convergence rate while $\Tcal_2$ carries a $O(\gamma^2)$-coefficient and is a noise term. The latter can be controlled 
by making the learning rate small.
We can handle the first term $\Tcal_1$ by leveraging a property of $F_{\theta, S}$ known as \emph{weak convexity}, meaning that adding a quadratic makes it convex. In particular, $F_{\theta, S} + (L/2) \normsq{\cdot}$ is convex, so that 
\begin{align*}
    \Tcal_1' &:= \mu\tau\gamma \inp*{z\pow{t} - w\pow{t}}{\sum_{i \in S}\pi_i\pow{t}  \grad F_i(w\pow{t})}
    \\ &\le
    \mu\tau\gamma\left(F_{\theta, S}(z\pow{t}) - F_{\theta, S}(w\pow{t}) + \frac{L}{2} \normsq*{z\pow{t}-w\pow{t}}\right) \,,
\end{align*}
where we used \eqref{eq:sfl-ncvx:subdiff-partial} to construct a subgradient of $F_{\theta, S}$. 
This term $\Tcal_1'$ is the result of a single step with learning rate $\tau\gamma$ rather than $\tau$ local steps with learning rate $\gamma$. The difference $\Tcal_1' - \Tcal_1$ is the effect of the drift induced by multiple local steps, which we will handle later. 
We take an expectation with respect to the sampling $S$ of clients (i.e., conditioned on $\Fcal\pow{t} = \sigma(w\pow{t})$, the $\sigma$-algebra generated by $w\pow{t}$). 
Since $z\pow{t}$ is independent of $S$ (i.e., $z\pow{t}$ is $\Fcal\pow{t}$-measurable), 
we get $\overline F_\theta$ on the right-hand side.
Next, we use that $z\pow{t}$ minimizes the strongly convex right hand side of \eqref{eq:sfl:moreau-env} to get
\begin{align*}
    \expect_t[\Tcal_1']
    &\le
    -\mu\tau\gamma(\mu-L) \normsq*{z\pow{t} - w\pow{t}}
    = -\frac{\tau\gamma(\mu-L)}{\mu}\normsq*{\grad \overline \Phi_\theta^\mu(w\pow{t})} \,.
\end{align*}
Next, we bound the effect of the drift using the Cauchy-Schwarz inequality and the smoothness of $F_i$'s as 
\begin{align*}
    \expect_t\left\vert 
    \Tcal_1 - \Tcal_1'
    \right\vert
    &=
    \mu\gamma \, \expect_t \left\vert 
    \inp*{z\pow{t}-w\pow{t}}{\sum_{i \in S} \pi_i\pow{t} \sum_{k=0}^{\tau-1} \left(  
        \grad F_i(w_{i, k}\pow{t}) - \grad F_i(w\pow{t})
    \right)}
    \right\vert \\
    &\le \frac{\mu\tau\gamma(\mu-L)}{2}\normsq{z\pow{t}-w\pow{t}}
        + \frac{\mu\gamma L^2}{2(\mu - L)} \,
        \expect_t \left[ 
        \sum_{i\in S}\pi_i\pow{t} \sum_{k=0}^{\tau-1} \normsq{w_{i, k}\pow{t} - w\pow{t}}  \right]
        \\
    &\le \frac{\tau \gamma(\mu - L)}{2\mu} \normsq*{\grad \overline \Phi_\theta^\mu(w\pow{t})} + O(\gamma^3) \,,
\end{align*}
where we bound the client drift $d\pow{t} =  \expect_t\left[ \sum_{i\in S}\pi_i\pow{t} \sum_{k=0}^{\tau-1} \normsq{w_{i, k}\pow{t} - w\pow{t}} \right] = O(\gamma^2)$ using standard techniques. 
We plug in $\mu = 2L$ to get a bound on $\Tcal_1'$ in terms of $\normsq*{\grad \overline \Phi_\theta^\mu(w\pow{t})}$. 
A standard argument to handle the noise term $\Tcal_2 \le O(\gamma^2)$ and telescoping the resulting inequality over 
$t=0, \ldots, T-1$ completes the proof. 
The full details are given in \Cref{sec:a:sfl:expt:ncvx}.
\end{proof}

\myparagraph{Convergence: Convex Case}
We consider the convergence of function values in the case where each $F_i$ is convex. Owing to the non-smoothness of $F_\theta$ and $\overline F_\theta$, we consider the following smoothed version of the objective in \eqref{eq:method:obj:minmax} and the corresponding modification to \Cref{algo:sfl-new}. 
First, define the Kullback-Leibler (KL) divergence between $\pi \in \Delta^{\vert S\vert -1}$ and the uniform distribution $(1/\vert S\vert , \ldots, 1/\vert S\vert )$ over $S \subset [n]$ as 
\[
    D_S(\pi) = \sum_{i\in S} \pi_i \log (\pi_i \, \vert S\vert ) \,.
\]
We simply write $D(\pi)$ when $S = [n]$.
Inspired by \cite{nesterov2005smooth,beck2012smoothing,devolder2014firstorder}, we define the smooth counterpart to \eqref{eq:method:obj:minmax} as
\begin{align} \label{eq:sfl-entropic-smoothing}
    F_\theta^\nu(w) = \max_{\pi \in \mathcal{P}_\theta} 
        \left\{ \sum_{i=1}^n \pi_i F_i(w) - \nu D(\pi) \right\} + \frac{\lambda}{2}\normsq{w} \,,
\end{align}
where $\nu > 0$ is a fixed smoothing parameter. 
We have that $\vert F_\theta^\nu(w) - F_\theta(w)\vert \le 2 \nu \log n$.
Finally, we modify line~\ref{line:sfl-new:reweigh} of \Cref{algo:sfl-new} to handle $F_\theta^\nu$ rather than $F_\theta$ as
\begin{align} \label{eq:sfl:smooth-argmax}
    \pi\pow{t} = \argmax_{\pi \in \mathcal{P}_{\theta, S}} \left\{ \sum_{i \in S} \pi_i F_i(w\pow{t}) - \nu D_S(\pi)  \right\} \,.
\end{align}

\begin{theorem} \label{thm:sfl-main:cvx}
    Suppose each function $F_i$ is convex
    and $0 < \lambda < L$. Define the condition number $\kappa = (L + \lambda)/\lambda$ and fix a time horizon $T \ge 16 \kappa^{3/2}$. Consider the sequence  $(w\pow{t})_{t=0}^T$ of iterates produced by the \Cref{algo:sfl-new} with line~\ref{line:sfl-new:reweigh} replaced by \eqref{eq:sfl:smooth-argmax}.
    Define the averaged iterate
    \[
        \overline w\pow{t} = \frac{\sum_{j=0}^t \beta_j w\pow{i}}{\sum_{i=0}^t \beta_j}, \quad \text{where} \quad
        \beta_j = \left(1 - \frac{\gamma\lambda \tau}{2}\right)^{-(1 + j)} \,,
    \]
    and $w^\star = \argmin_{w\in\reals^d} F_\theta(w)$. 
    Then, there exist learning rate $\gamma$ and smoothing parameter $\nu$ depending on the number of communication rounds $T$ as well as problem parameters $\tau, G, \lambda, L, \normsq{w\pow{0} - w^\star}, \theta, m$, 
    such that the iterate $\overline w\pow{T}$ satisfies the bound
    \begin{align*}
        \expect F_\theta(\overline w\pow{T})
        - F_\theta(w^\star) \le 
         \lambda \normsq{w\pow{0} - w^\star} \exp\left(- \frac{T} {16 \kappa^{3/2}}\right)
    + \frac{G^2}{\lambda T}
    + \frac{G^2 \kappa^2}{\lambda T^2} + \frac{B}{\sqrt{\theta m}}  \,,
    \end{align*}
    where we hide absolute constants and factors polylogarithmic in $T$ and problem parameters. 
\end{theorem}

\begin{remark}[About the Rate] \label{remark:sfl-cvx:rate}
As soon as 
$T \gtrsim \kappa^{3/2}$ (ignoring constants and polylog factors), 
we achieve the optimal rate of $1/(\lambda T)$ rate of strongly convex stochastic optimization
up to the bias $B/\sqrt{\theta m}$.  
Further,
the bias $B/\sqrt{\theta m}$ due to partial participation is larger at small $\theta$ and can be controlled by choosing the cohort size $m$ large enough. 
In the experiments of \Cref{sec:sfl:expt}, we obtain meaningful numerical results when $m$ is around $50$ or $100$ and $\theta$ around $1/2$, indicating that the worst-case bound \eqref{eq:sfl-main:bias-bound:no-smooth} can be pessimistic. 

\end{remark}

\begin{proof}[Proof Sketch of \Cref{thm:sfl-main:cvx}]
We start with some additional notation. 
We absorb the regularization into the client losses to define $\tilde F_i(w) = F_i(w) + (\lambda/2) \normsq{w}$.
Now, consider the smoothed counterpart of \eqref{eq:method:obj:minmax} on a subset $S \subset [n]$ with a smoothing parameter $\nu > 0$ as 
\begin{align*}
    F_{\theta, S}^\nu(w) = \max_{\pi \in \mathcal{P}_{\theta, S}} \left\{ \sum_{i \in S} \pi_i \tilde F_i(w) - \nu D_S(\pi)  \right\} \,.
\end{align*}
It follows from the properties of smoothing~\cite{nesterov2005smooth,beck2012smoothing}
and composition rules that $F_{\theta, S}^\nu$ is $L'$-Lipschitz, where $L' = L + \lambda + G^2/\nu$.
Finally, let $\Fcal_t$ denote the sigma-algebra generated by $w\pow{t}$ and let $\expect_t[\cdot] := \expect[ \cdot \vert \Fcal_t]$.

We start the proof with the decomposition
\begin{align*}
    \normsq{w\pow{t+1} - w}
    =& \, \normsq{w\pow{t}-w} 
    - 
    \underbrace{2 \gamma \sum_{i \in S} \pi_i\pow{t} \sum_{k=0}^{\tau - 1} \inp{\grad \tilde F_i(w_{i, k}\pow{t})}{w\pow{t} - w}}_{=: \Tcal_1} \\
    &+ \underbrace{\gamma^2 \normsq*{\sum_{i \in S} \pi_i\pow{t} \sum_{k=0}^{\tau-1}\grad \tilde F_i(w_{i, k}\pow{t})}}_{=: \Tcal_2} \,,
\end{align*}
where $w$ is arbitrary.
For the first order term $\Tcal_1$, we bound using $\lambda$-strong convexity and $L$-smoothness of $\tilde F_i$ as 
\begin{align*}
     \Tcal_1' &:= 2 \tau\gamma \, \sum_{i \in S} \pi_i\pow{t} \inp*{\grad \tilde F_i(w\pow{t}}{w\pow{t} - \overline w^\star}
    = 2 \tau\gamma\, \inp*{\grad F_{\theta, S}^\nu(w\pow{t})}{w\pow{t} - \overline w^\star} \\
    &\ge 2\tau\gamma\, \left(F_{\theta,S}^\nu(w\pow{t}) -  F_{\theta,S}^\nu (\overline w^\star) 
            + \frac{\lambda}{2} \normsq{w\pow{t} - \overline w^\star} \right) \,.
\end{align*}
where we used $\sum_{i \in S} \pi_i\pow{t} \grad F_i(w\pow{t}) = \grad F_{\theta, S}^\nu(w\pow{t})$ holds with smoothing, analogous to \eqref{eq:sfl-ncvx:subdiff-partial}, 
and strong convexity. 

The gap $\Tcal_1 - \Tcal_1'$ is due to the effect of the drift from multiple local steps. We bound this term similar to the non-convex case of \Cref{thm:sfl:nonsmooth-nonconvex:main}. 
For the second order term $\Tcal_2$, we rely
on the variance bound~\cite[Prop. 2]{levy2020largescale}
\[
    \expect_{S \sim U_m} \normsq*{\sum_{i\in S} \pi_i\pow{t} \grad \tilde F_i(w\pow{t})
    - \grad \overline F_\theta^\nu(w\pow{t})}
    \le \frac{8G^2}{\theta m} \,,
\]
where $U_m$ is the uniform distribution over subsets $S \subset [n]$ of size $m$, and 
$\overline F_\theta^\nu(w)  := \expect_{S \sim U_m} F_{\theta, S}^\nu(w)$ as the expectation of 
$F_{\theta, S}^\nu$ over random subsets $S \sim U_m$. 
Putting these together and taking $w = \overline w^\star := \argmin \overline F_\theta^\nu$ gives the inequality, 
\begin{align}
\label{eq:s-fl:new-main:proof-1}
    \overline F_\theta^\nu  (w\pow{t}) - \overline F_\theta^\nu(\overline w^\star) 
    \le\,\,  & \,
    \gamma A + \gamma^2 B +  \\
    & \frac{1}{\gamma \tau}\left(1 - \frac{\lambda \gamma \tau}{2}\right) \normsq{w\pow{t} - \overline w^\star} 
    - \frac{1}{\gamma \tau} \expect_t \normsq{w\pow{t+1} - \overline w^\star} 
     \,, \nonumber
\end{align}
where $A, B$ are problem-dependent constants. 
We sum this up with the averaging weights $\beta_t$ given in the statement of the theorem to get
\[
 \overline F_\theta^\nu  (\overline w\pow{T}) - \overline F_\theta^\nu(\overline w^\star) 
    \le 
    \frac{\lambda \normsq{w\pow{0} - w^\star}}{\exp(\lambda \tau \gamma T) - 1}
    + A \gamma + B \gamma^2 \,.
\]
The final missing piece is a bound 
which allows us to translate statements about the convergence of
$\overline F_{\theta}^\nu $ in terms of the convergence of 
$F_{\theta}$.
We achieve this using the bias bound of \eqref{eq:sfl-main:bias-bound:no-smooth} together with the approximation error of smoothing.
Finally, we optimize the choice of the learning rate and smoothing coefficient to give the final statement of the theorem. The details are provided in \Cref{sec:a:sfl:expt:cvx}.
\end{proof}

\subsection{Privacy and Utility Analysis} \label{sec:sfl:privacy}

We now analyze the privacy and utility of \Cref{algo:sfl:quantile-dp:new}.
In this section, we assume without loss of generality that $S = [n]$ so that $m = \vert S \vert = n$. 

First, we recall the definition of concentrated differential privacy~\cite{bun2016concentrated}.
A randomized algorithm $\Acal$ satisfies $(1/2)\epsilon^2$-concentrated differential privacy if
the R\'enyi $\alpha$-divergence
$D_\alpha(\Acal(X) \Vert \Acal(X')) \le  \alpha \epsilon^2 /2 $
for all $\alpha \in (0, \infty)$ and all sequences $X, X'$ of inputs that differ by the addition or removal of one client's data.
Intuitively, the addition or removal of the data contributed by one client should not change the output distribution of the randomized algorithm by much, as measured by the R\'enyi divergence. A smaller value of $\epsilon$ implies a stronger privacy guarantee. This notion of differential privacy can be translated back and forth with the usual one, cf.~\cite{canonne2020discrete}.

\myparagraph{Error Criterion}
We approximate the $(1-\theta)$-quantile of the $n$ per-client losses $\ell_i = F_i(w)$ for $i=1,\ldots,n$ by the quantile of a hierarchical histogram $h$ with entries
$h(r, j) = \sum_{i=1}^n \mathbb{I}\left(  l_{2^r(j-1) + 1} \le F_i(w) <  l_{2^r j} \right)$
where $0=l_0 < l_1 < \cdots < l_b=B$ are 
the bin edges. 
The edge $l_j$ corresponding to index $j \in [b]$ approximates the $(1-\theta)$-quantile well if the cumulative mass $H(j) \approx (1-\theta)n$. We measure this error of approximation by the difference between the two sides. 
Formally, we define the error $R_\theta(H, j)$ of approximating the $(1-\theta)$-quantile of the cumulative function $H$ of a 
hierarchical histogram with index $j \in [b]$ by
\begin{align}
    R_\theta(H, j) = \left\vert \frac{H(j)}{n} - (1-\theta) \right\vert \,.
\end{align}
We define the best achievable error $R^*_\theta(H)$ for estimating the $(1-\theta)$-quantile of the cumulative function $H$ and 
the best approximating index $j^*(H)$ as
\begin{align}
    R_\theta^*(H) = \min_{j \in [b]} R_\theta(H, j) \,,
    \quad \text{and} \quad
    j_\theta^*(H) = \argmin_{j \in [b]} R_\theta(H, j) \,,
\end{align}
where we assume ties are broken in an arbitrary but deterministic manner --- note that $j_\theta^*(H)$ is defined here identically to \eqref{eq:sfl-privacy:quantile}.
Lastly, we define the \emph{quantile error} 
$\Delta_\theta(H, \hat H)$ of estimating the quantile of the cumulative function $H$ from that of $\hat H$ as
\begin{align} \label{eq:sfl:quantile-error}
    \Delta_\theta(\hat H, H) = R_\theta\big(H, j^*_\theta(\hat H) \big) \,.  
\end{align}
Essentially, if the index $j^*_\theta(\hat H)$ computed from the estimate $\hat H$ corresponds to the $(1-\theta')$-quantile of $H$, the quantile error satisfies $\Delta_\theta(\hat H,  H) = \vert\theta-\theta'\vert$.

\myparagraph{Privacy and Utility Analysis}
We now analyze the differential privacy bound of \Cref{algo:sfl:quantile-dp:new} and the error in the quantile computation.

\begin{theorem} \label{thm:sfl:quantile:new}
Fix a $\delta > 0$. Suppose that $\sigma \ge 1/2$
and $c > 0$ are given, and the modular arithmetic is performed on the base
$M \ge 2 + 2c n + 2n \sqrt{2\sigma^2 \log(16nb/\delta)}$.
Then, we have:
\begin{enumerate}[label=(\alph*),itemsep=0em, topsep=0em, leftmargin=1.6em]
    \item  \Cref{algo:sfl:quantile-dp:new} satisfies $(1/2)\epsilon^2$-concentrated DP with 
    \[
        \epsilon = \min\left\{ 
            \sqrt{\frac{c^2 \log_2^2 b}{n\sigma^2} + \psi b}, 
            \frac{c \log_2 b}{\sqrt{n}\sigma} + \psi \sqrt{2b}
        \right\} \,,
    \]
    where $\psi = 10 \sum_{i=1}^{n-1} \exp\big(-2\pi^2\sigma^2 i / (i+1)\big) \le 10 (n-1) \exp(-2\pi^2\sigma^2)$.
    \item With probability at least $1-\delta$, the quantile error of cumulative function $\hat H$ returned by \Cref{algo:sfl:quantile-dp:new} is at most
    \[
        \Delta_\theta(\hat H, H) \le 
        R_\theta^*(\hat H) 
        + \sqrt{\frac{4 \sigma^2}{c^2 n} \log_2 b \, \log \frac{4 b}{\delta}}
    \]
    where $R_\theta^*(\hat H)$ is the error in the estimation of $(1-\theta)$-quantile of the cumulative function $\hat H$.
\end{enumerate}
\end{theorem}  

Let us interpret the result. 
    The effective noise scale is $\sigma / c$.
        Since the dominant term of the privacy error is $\epsilon \approx c \log_2 b /(\sigma \sqrt{n})$, we choose ${\sigma}/{c} \approx \log_2 b / ({\epsilon \sqrt{n}})$, 
        so that the algorithm satisfies $(1/2)\epsilon^2$-concentrated DP.
        The role of $c$ is to avoid the degeneracy of the discrete Gaussian as $\sigma \to 0$. In particular, the theorem requires $\sigma \ge 1/2$.
    The error resulting quantile error $\Delta_\theta(\hat H, H)$ is (ignoring constants and log factors)
        \[
            \Delta_\theta(\hat H, H) \lesssim
            R_\theta^*(\hat H)
            + \frac{\log^2 b}{\epsilon n} \,.
        \]
    The quantile error scales as $1/(\epsilon n)$.
    The total communication cost is $O(bn \log M)$ bits since the dimension of each hierarchical histogram is $2(b-2)$. 
    If we take $\sigma = O(1)$ and $c = O(\epsilon \sqrt{n})$, we require $M \gtrsim n^{3/2}$, so that the total communication cost is $O(bn\log n)$.

\begin{proof}[Proof of \Cref{thm:sfl:quantile:new}]
    We can show that no modular wraparound occurs anywhere in the algorithm with high probability. We assume that it holds for the proof sketch. Thus, for all valid levels $r$ and indices $j$, 
    we have $\tilde x_i(r, j) = c  x_i(r, j) + \xi_i(r, j)$ and
    \[
        \hat h(r, j) = \sum_{i=1}^n \frac{\tilde x_i(r, j)}{c }
            = \sum_{i=1}^n \left(x_i(r, j) + \frac{\xi_i(r, j)}{c } \right) \,.
    \] 
    
    The privacy analysis follows from the sensitivity of the sum query. Namely, let $X = (x_1, \ldots, x_n)$ be a sequence and define $A(X) = \sum_{i=1}^n c x_i$ as the (rescaled) sum query. In our case, each $x_i$ is a hierarchical histogram with $\log_2 b$ ones being the only non-zeros, one for each level of the tree. 
    \Cref{algo:sfl:quantile-dp:new} adds discrete Gaussian noise to the sum query to make it differentially private. That is, 
    we get the randomized algorithm $\Acal(X) = A(X) + \sum_{i=1}^n \xi_i$.  
    It was shown in \cite[Corollary 12]{kairouz2021distributed} that 
    $\Acal(X)$ is approximately distributed as  $\Ncal_{\ZZ}(A(X), n\sigma^2)$, 
    so the desired privacy guarantee follows from that of the discrete Gaussian mechanism~\cite{canonne2020discrete}. 
    In particular, for two sequences $X$ and $X'$ differing by the addition or removal of a single basis vector $x'$, we have that 
    \[
        D_\alpha(\Acal(X) \Vert \Acal(X'))
        \approx D_\alpha(\Ncal_{\ZZ}(A(X), n\sigma^2) \Vert \Ncal_{\ZZ}(A(X'), n\sigma^2))
        = \frac{\alpha c^2}{2n\sigma^2} \,.
    \]
    A rigorous analysis of the error, following the recipe of \cite{kairouz2021distributed}, leads to the first part of the theorem; the details can be found in \Cref{sec:a:sfl:privacy}.
    
    \myparagraph{Utility Analysis}
    The triangle inequality gives
    \begin{align*}
        \Delta_\theta(\hat H, H) 
        &\le \frac{1}{n}\left\vert H\big(j^*_\theta(\hat H)\big) -\hat H\big(j^*_\theta(\hat H)\big) \right\vert
        + \left\vert \frac{1}{n} \hat H\big(j^*_\theta(\hat H)\big)
            - (1-\theta) \right\vert \\
        &\le 
        \max_{j \in [b]}  \left\{ \frac{1}{n}\left\vert H(j) - \hat H(j) \right\vert \right\}
        + R_\theta^*(\hat H) \,.
    \end{align*}
    Using standard concentration arguments, we show that 
    the first term is, at most 
    $\sqrt{2 \sigma^2 n \log_2(b) \log(4b/\delta)}$, completing the proof.
\end{proof} 
\section{Discussion} \label{sec:sfl:discussion}
We discuss connections of \newfl to risk measures, fair resource allocation, and model personalization.

\myparagraph{Connection to Risk Measures}
The framework of risk measures in economics and finance formalizes the notion of minimizing the worst-case cost over a set of distributions~\cite{follmer2002convex, rockafellar2013fundamental,zbMATH06621946}. 
The superquantile $\supq_\theta(\cdot)$ is a special case of a risk measure. The \newfl framework, which minimizes the superquantile of the per-client losses, can be extended to other risk measures $\riskM$ by minimizing the objective
\[
    F_M(w) := \riskM(Z(w)) + \frac{\lambda}{2}\normsq{w}\,,
\]
where $Z(w)$ is a discrete random variable which takes value $F_i(w)$ with probability $1/n$ for $i \in [n]$. 
Another example of a risk measure is the {\em entropic risk measure}, 
which is defined as 
$\riskM_{\mathrm{ent}}^\nu(Z) = \expect[\exp(\nu Z)]/\nu$
where $\nu \in \reals_+$ is a parameter. 
The entropic risk measure is well defined provided the moment generating function $\expect[\exp(\nu Z)]$ exists, for instance, for sub-Gaussian $Z$.
The analog of \newfl with the entropic risk minimizes
\[
    F_{\mathrm{ent}}^\nu(w) = \frac{1}{\nu} \log\left(\frac{1}{n}\sum_{i=1}^n \exp\big(\nu F_i(w)\big) \right) +\frac{\lambda}{2}\normsq{w}\,.
\]
This objective $F_{\mathrm{ent}}^\nu(w)$ coincides with the one studied recently in~\cite{li2020tilted} under the name \term subsequent to the first presentation of this work~\cite{laguel2020device}. Finally, we note that 
$F_{\mathrm{ent}}^\nu$ is also related to the smoothed objective 
$F_\theta^\nu$ from~\eqref{eq:sfl-entropic-smoothing} as the limit
\[
    F_{\mathrm{ent}}^\nu(w) = \lim_{\theta \to 0} F_\theta^\nu(w)\,,
\]
\myparagraph{Maximin Strategy for Resource Allocation}
We would like to point out an interesting analogy between distributional robustness 
and proportional fairness. The superquantile-based objective in Eq.~\eqref{eq:method:obj:minmax} is a maximin-type objective
that is reminiscent of maximin objectives used in load balancing 
and network scheduling~\cite{kubiak2008proportional,stanczak2009fundamentals,pantelidou2011scheduling}.

We can draw an analogy between the two worlds, federated learning and resource allocation resp., 
by identifying errors to rates and clients to users. The maximin fair strategy to resource allocation seeks to
treat all users as fairly as possible by making their rates as large and as equal as possible 
so that no rate can be increased without sacrificing other rates that are smaller or equal~\cite{pantelidou2011scheduling}. 

Our superquantile-based \newfl framework builds off the maximin decision-theoretic foundation to frame an objective 
that we \emph{optimize} with respect to \emph{parameters} of models, and this, iteratively, over multiple rounds of client-server communication, 
while preserving the privacy of each client. 

This compositional nature of our problem, where we optimize a composition (in the mathematical sense) of
a maximin-type objective, a loss function, and model predictions differ with resource allocation in communication networks.
Further explorations of the analogy are left for future work. 

\myparagraph{Model Family and Tail Thresholds}
\update{Using a single global value of the tail threshold $\theta$ for all clients could fail to balance supporting tail clients with fitting the population average. 
To circumvent this issue, we use a similar idea to the one of~\cite{li2020fair} where a family of models is trained simultaneously for various levels, and each test client can tune its tail threshold.}

\myparagraph{\newfl vs. Model Personalization}
Consider a family of distributions $q_i(x, y)$ for $i=1, \ldots, n$ over input-output pairs. From the decomposition 
$q_i(x, y) = q_i(x) q_i(y \vert x)$, it follows that the heterogeneity of the joint distributions can be due to 
(a) heterogeneity of the marginal distributions $q_i(x)$ over the input $x$, or, 
(b) heterogeneity of the conditional distributions $q_i(y \vert x)$, or in other words, the input-output mapping. 

If two clients do not agree on their input-output mapping, a single global model cannot serve both simultaneously. Thus, when training one single global model (as in vanilla FL) or a small number of them (as in \newfl), 
there is an implicit assumption that the heterogeneity of $\{q_i(y \vert x)\,: \, i \in [n]\}$ is small. \newfl was designed to handle the heterogeneity of $q_i(x)$ better than vanilla FL by providing better worst-case performance on tail clients.  

On the other hand, the cases where the heterogeneity of the conditional distributions $q_i(y \vert x)$ is large 
requires a separate model per client, or in other words, model personalization. 
Standard approaches to model personalization still aim to minimize the average error across all clients~\cite{dinh2020moreau,pillutla2022pfl}, similar to the vanilla FL objective. Thus, it can still suffer from disparate performance across clients, including poor performance on some tail clients or data-poor clients. One solution to reduce this disparity is to combine personalization with the \newfl objective.
We refer to~\Cref{sec:sfl:expt:privacy} for numerical experiments.

\update{
\myparagraph{Quantile-based Filtering and Client Availability}
We note that the quantile-based filtering of \Cref{algo:sfl-new} implies that only $\theta m$ tail clients contribute their updates to the global model in the absence of noise (that is, the weight $\pi\pow{t}$ in line~\ref{line:sfl-new:reweigh} of \Cref{algo:sfl-new} is sparse; see also \Cref{prop:sfl-quantile}). 
In order to include the updates of $m'$ clients after filtering, \newfl would require initially sampling an initial cohort of $m = m' / \theta$ clients. On the other hand, 
clients in cross-device federated learning are typically available in a diurnal pattern~\cite{eichner2019semi,kairouz2019flsurvey}, where a large enough number of clients might not be available at certain times of the day.
This issue might be exacerbated by \newfl's requirement of $m'/\theta$ clients per round as compared to FedAvg's $m'$. 
Devising strategies to dynamically vary the tail threshold $\theta$ based on the number of available clients to overcome this issue is an interesting venue for future work. 
}

\section{Experiments} \label{sec:sfl:expt}
\begin{table}[t]
    \caption{\small{Dataset description and statistics.}}
\label{table:sfl:expt:dataset:descr}
\begin{center}
\begin{adjustbox}{max width=\linewidth}
\begin{tabular}{lccccccc}
\toprule
Task & Dataset & \#Classes & Devices & \multicolumn{2}{c}{\#Data per client} \\
\cmidrule{5-6}
& & & &  Median  & Max  \\
\midrule
Image Recognition & EMNIST &   62 & 1730 & 179 & 447 \\
Sentiment Analysis & Sent140 &  2 & 877 & 69 & 549 \\
\bottomrule
\end{tabular}
\end{adjustbox}
\end{center}
\end{table}

In this section, we demonstrate the effectiveness of \newfl in handling heterogeneity in federated learning. Our experiments were implemented in Python using automatic differentiation provided by PyTorch while the data was preprocessed using LEAF~\cite{caldas2018leaf}.  
The code to reproduce our experiments can be found online.\footnote{
\url{https://github.com/krishnap25/simplicial-fl}
}
We start by describing the datasets, tasks, and models in
\Cref{sec:sfl:expt:datasets}. 
We present numerical comparisons to several recent works -- we list them in \Cref{sec:sfl:expt:algos} and show the experimental results in \Cref{sec:sfl:expt:results}.
We demonstrate that \newfl provides the most favorable tradeoff between average error and the error on tail clients in \Cref{sec:sfl:expt:heterogeniety}.
Next, we compare \newfl with model personalization in \Cref{sec:sfl:expt:pers}.
Finally, we numerically study the privacy-utility tradeoff of the differentially private quantile computation (in \Cref{sec:sfl:expt:privacy}), and of \newfl with end-to-end differential privacy (in \Cref{sec:sfl:expt:end-to-end-dp}). 
 
Full details regarding the experiments, as well as additional results, are provided in the supplementary material.

\subsection{Datasets, Tasks and Models} \label{sec:sfl:expt:datasets}

We consider two learning tasks. The dataset and task statistics are summarized in \Cref{table:sfl:expt:dataset:descr}.

\begin{enumerate}[label=(\alph*),nolistsep,topsep=0em, leftmargin=1.6em]
\item{\textit{Character Recognition}:}
We use the EMNIST dataset~\cite{cohen2017emnist}, 
where the input $x$ is a $28\times 28$ grayscale image of a handwritten character, and the output $y$ is its label (0-9, a-z, A-Z). Each client is a writer of the character $x$. The weight $\alpha_i$ assigned to author $i$ is the number of characters written by this author. We train both a linear model and a convolutional neural network architecture (ConvNet). The ConvNet consists of two $5\!\times \! 5$ convolutional layers with max-pooling followed by one fully connected layer. Outputs are vectors of scores for each of the $62$ classes. The multinomial logistic loss is used to train both models.

\item{\textit{Sentiment Analysis}:}
We use the Sent140 dataset~\cite{go2009twitter} where the input $x$ is a tweet, and the output $y=\pm1$ is its sentiment.
Each client is a distinct Twitter user. The weight $\alpha_i$ assigned to user $i$ is the number of tweets published by this user. We train a logistic regression and a Long-Short Term Memory neural network architecture (LSTM). The LSTM is built on the GloVe embeddings of the words of the tweet~\cite{hochreiter1997long}. The hidden dimension of the LSTM is the same as the embedding dimension, i.e., $50$. We refer to the latter as ``RNN''. The loss used to train both models is the binary logistic loss.
\end{enumerate}

\subsection{Algorithms and Hyperparameters}
\label{sec:sfl:expt:algos}

We list here the competing approaches we benchmark and discuss their hyperparameters.

\myparagraph{Algorithms}
As discussed in \Cref{sec:sfl:setup}, 
a federated learning method is characterized by 
the objective function, as well as the federated optimization algorithm. We compare \newfl with the following baselines:
\begin{enumerate}[label=(\alph*),nolistsep,leftmargin=\widthof{ (a) }]
\item Vanilla FL objective: We consider two methods that attempt to minimize the vanilla FL objective: \fedavg~\cite{mcmahan2017communication} and \fedprox~\cite{li2020fedprox}. The latter augments \fedavg with a proximal term for more stable optimization.
\item Heterogeneity-aware objectives: We consider \term~\cite{li2020tilted}, which is the analogue of \newfl with the entropic risk measure (cf. \Cref{sec:sfl:discussion}) and \afl~\cite{mohri2019agnostic}, whose objective is obtained as the limit $\lim_{\theta\to 0} F_\theta(w)$ of the \newfl objective. We also consider \qffl~\cite{li2020fair}, which raises the per-client loss $F_i$ to the $(q+1)$\textsuperscript{th} power, for some $q > 0$.
We optimize \qffl and \term with the federated optimization algorithms proposed in their respective papers. We use \qffl with $q=10$ in place of \afl, as it was found to have more stable convergence with similar performance. 
\end{enumerate}
We compare to one more baseline for the vanilla FL objective. Note that \newfl the weight $\pi\pow{t}$ (see line~\ref{line:sfl-new:reweigh} of \Cref{algo:sfl-new}) is sparse, i.e., it is non-zero for only some of the $m$ selected clients, cf. \Cref{prop:sfl-quantile}. This is equivalent to a fewer number of effective clients per round, which is $\theta m$ on average. We use as baseline \fedavg with $\theta m$ clients per round, where $m$ is the number of clients per round in \newfl; we call it \fedavgsub.

Similar to \cite{mcmahan2017communication}, we consider a weighted version of the vanilla FL objective where each client's loss is weighted by $\alpha_i = N_i / N$, where $N_i$ is the number of data points on client $i$ and $N = \sum_i N_i$. Similarly, we also consider a weighted version of the \newfl objective as a superquantile of a random variable that takes value $F_i(w)$ with probability $\alpha_i$.
For a fair comparison, we run all algorithms, including \newfl, without differential privacy. We postpone a study of \newfl with differential privacy to \Cref{sec:sfl:expt:end-to-end-dp}. 

\myparagraph{Hyperparameters}
We fix the number of clients per round to be $m=100$ for each dataset-model pair except for Sent140-RNN, for which we use $m=50$.
We fixed an iteration budget for each dataset 
during which \fedavg{} converged.
We tuned a learning rate schedule using grid search 
to find the smallest terminal loss averaged over training clients for \fedavg. 
The same iteration budget and learning rate schedule were used for {\em all} other methods, including \newfl.
Each method, except \fedavgsub, 
selected $m$ clients per round for training, as specified earlier.
The regularization parameter $\lambda$, 
and the proximal weight of {FedProx}
were tuned to minimize the 
$90$\textsuperscript{th} percentile of the misclassification error
on a held-out subset of training clients. 
We run \qffl{} for $q \in \{10^{-3}, 10^{-2},\ldots, 10\}$ and
report $q$ with the 
smallest $90$\textsuperscript{th} percentile of misclassification error on {\em test} clients.
We run \term with a temperature parameter $\nu \in \{0.1, 0.5, 1, 5, 10, 50, 100, 200\}$ and also report $\nu$ with the 
smallest $90$\textsuperscript{th} percentile of misclassification error on {\em test} clients. We optimize \newfl with \Cref{algo:sfl-new}
for threshold levels $\theta \in  \{0.8, 0.5, 0.1\}$.

\begin{table*}[t]
\small
\caption{{
\textbf{$\mathbf{90}$\textsuperscript{th} percentile} 
of the distribution 
of misclassification error (in $\%$) on the test devices. Each entry is the mean over five random seeds while the standard deviation is reported in the subscript. The boldfaced/highlighted entries denote the smallest value for each dataset-model pair.
}}
\label{table:sfl:dataset:results_90}
\begin{center}
\begin{adjustbox}{max width=\linewidth}
\setlength{\tabcolsep}{8pt}
{\renewcommand{\arraystretch}{1.2}%
\begin{tabular}{lccccc}
\toprule
 		& \multicolumn{2}{c}{\textbf{EMNIST}} 				& \multicolumn{2}{c}{\textbf{Sent140}} 		\\
 		  \cmidrule(lr){2-3} \cmidrule(lr){4-5}
			& 	\textbf{Linear} 			& 	\textbf{ConvNet} 			& 	\textbf{Linear}			& 		\textbf{RNN}						\\
\midrule

\fedavg 	& $ 49.66_{0.67}$ 	& $28.46_{1.07}$ 		& $46.83_{0.54}$ 	& $49.67_{3.95}$ 	\\

\fedavgsub 	& $50.28_{0.77}$ 	& $27.57_{0.81}$ 		& $46.60_{0.38}$ 	& $46.94_{3.84}$ 	 \\

\fedprox 	& $49.15_{0.74}$ 	& $27.01_{1.86}$ 		& $46.83_{0.54}$ 	& $49.86_{4.07}$ 	\\

\qffl 		& $49.90_{0.58}$ 	& $28.02_{0.80}$ 		&\tabemph{} $\mathbf{46.39}_{0.40}$ & $48.66_{4.68}$		\\

\term 		& $48.59_{0.62}$ 	& $25.46_{1.49}$ 		&$46.69_{0.49}$& $46.54_{3.27}$		\\

\afl 		& $51.62_{0.28}$ 	& $45.08_{1.00}$ 		& $47.52_{0.32}$ 	& $57.78_{1.19}$ 		\\
\midrule

\newfl,
$\theta=0.8$& $49.10_{0.24}$ 	& $26.23_{1.15}$ 		& ${46.44}_{0.38}$ 	&\tabemph{} $\mathbf{46.46}_{4.39}$	\\

\newfl, 
$\theta=0.5$& \tabemph{} $\mathbf{48.44_{0.38}}$ & 
			\tabemph{} $\mathbf{23.69_{0.94}}$ & $46.64_{0.41}$ 	& $50.48_{8.24}$ 	 	\\

\newfl, 
$\theta=0.1$& $50.34_{0.95}$ & $25.46_{2.77}$ 			& $51.39_{1.07}$ 	& $86.45_{10.95}$ 	 \\

\bottomrule

\end{tabular}} %
\end{adjustbox}
\end{center}
\end{table*}

\begin{table*}[t]
\small
\caption{{
\textbf{Mean} of the distribution of misclassification error (in $\%$) on the test devices.  Each entry is the mean over five random seeds while the standard deviation is reported in the subscript. The boldfaced/highlighted entries denote the smallest value for each dataset-model pair.
}}
\label{table:sfl:dataset:results_mean}
\begin{center}
\begin{adjustbox}{max width=\linewidth}
\setlength{\tabcolsep}{8pt}
{\renewcommand{\arraystretch}{1.2}%
\begin{tabular}{lccccc}
\toprule
 		& \multicolumn{2}{c}{\textbf{EMNIST}} 				& \multicolumn{2}{c}{\textbf{Sent140}} 		\\
 		  \cmidrule(lr){2-3} \cmidrule(lr){4-5}
			& 	\textbf{Linear} 			& 	\textbf{ConvNet} 			& 	\textbf{Linear}			& 		\textbf{RNN}						\\
\midrule

\fedavg 	& ${34.38}_{0.38}$ 		& $16.64_{0.50}$ 			& $34.75_{0.31}$ 		& ${30.16}_{0.44}$ 	\\

\fedavgsub 	& $34.51_{0.47}$ 		& $16.23_{0.23}$ 			& $34.47_{0.03}$ 		&\tabemph{}  $\mathbf{29.86}_{0.46}$ 	\\

\fedprox 	&\tabemph{} $\mathbf{33.82_{0.30}}$& $16.02_{0.54}$ 			& $34.74_{0.31}$ 		& $30.20_{0.48}$ 	\\

\qffl 	& $34.34_{0.33}$ 		& $16.59_{0.30}$ 			& $34.48_{0.06}$ 		& $29.96_{0.56}$ 	\\

\term 		& $34.02_{0.30}$ 	& $15.68_{0.38}$ 	& $34.70_{0.31}$	&$30.04_{0.25}$ 		\\
\afl 		& $39.33_{0.27}$ 		& $33.01_{0.37}$ 			& $35.98_{0.08}$ 		& $37.74_{0.65}$  	\\
\midrule
\newfl, 
$\theta=0.8$& $34.49_{0.26}$ 		& $16.09_{0.40}$ 			& \tabemph{} $\mathbf{34.41}_{0.22}$& $30.31_{0.33}$ 	\\

\newfl, 
$\theta=0.5$& $35.02_{0.20}$ 		& \tabemph{}  $\mathbf{15.49_{0.30}}$ 	& $35.29_{0.25}$ 		& $33.59_{2.44}$ 	\\

\newfl, 
$\theta=0.1$& $38.33_{0.48}$ 		& $16.37_{1.03}$ 			& $37.79_{0.89}$ 		& $51.98_{11.81}$ 	\\

\bottomrule

\end{tabular}} %
\end{adjustbox}
\end{center}
\end{table*}

\subsection{Experimental Results} 
\label{sec:sfl:expt:results}

We measure in \Cref{table:sfl:dataset:results_90} the $90$\textsuperscript{th} 
percentile of the misclassification error across the test clients as a measure of the right tail of the per-client performance.
We also measure in \Cref{table:sfl:dataset:results_mean} the mean error, which measures the average test performance.
Our main findings are summarized below. 

\myparagraph{\newfl consistently achieves the smallest $90$\textsuperscript{th} percentile error}
\newfl achieves a $3.3\%$ absolute ($12\%$ relative) improvement over any vanilla FL objective on EMNIST-ConvNet. Among the heterogeneity-aware objectives, \newfl achieves $1.8\%$ improvement over the next best objective, which is \term. 
We note that \qffl marginally outperforms \newfl on Sent140-Linear, but the difference $0.05\%$ is much smaller than the standard deviation across runs. 

\myparagraph{\newfl is competitive at multiple values of $\theta$}
For EMNIST-ConvNet, \newfl with $\theta \in \{0.5, 0.8\}$ is better in $90$\textsuperscript{th} percentile error than {\em all} other methods we compare to, and 
\newfl with $\theta = 0.1$ is tied with \term, the next best method. 
We also empirically confirm that \newfl interpolates between \fedavg ($\theta \to 1$) and \afl ($\theta \to 0$).

\myparagraph{\newfl works best for larger threshold levels}
We observe that \newfl with $\theta=0.1$ is unstable for Sent140-RNN. 
This is consistent with \Cref{thm:sfl-main:cvx}, which requires $m$ to be much larger than $1/\theta$ (cf. \Cref{remark:sfl-cvx:rate}).
Indeed, this can be explained by
\newfl's sparse re-weighting, which only gives 
non-zero weights to $\theta m = 5$ clients on average in each round (cf. \Cref{remark:sfl-client-filtering}). 

\myparagraph{Yet, \newfl is competitive in terms of average error}
Perhaps surprisingly, \newfl gets the best test error performance on EMNIST-ConvNet and Sent140-Linear. This suggests that the average test distribution is shifted relative to the average training distribution $p_\alpha$. 
In the other cases, we find that the reduction in mean error is small relative to the gains in the $90$\textsuperscript{th} percentile error compared to Vanilla FL methods. 

\myparagraph{Minimizing superquantile loss over all clients performs better than minimizing worst error over all clients} 
Specifically, \afl which aims to minimize the worst error among all clients, as well as other
objectives which approximate it (\newfl with $\theta\to 0$, \qffl with $q\to \infty$, \term with $\nu \to 0$) tend to achieve poor performance. We find that \afl achieves the highest error both in terms of $90$\textsuperscript{th} percentile and the mean.
\newfl offers a more nuanced and more effective approach through an averaging of the tail performances rather than the straightforward pessimistic approach minimizing the worst error among all clients.

\begin{figure*}[t!]
  \centering
  {\includegraphics[width=0.6\linewidth,trim={0 5pt 0 30pt},clip=true]{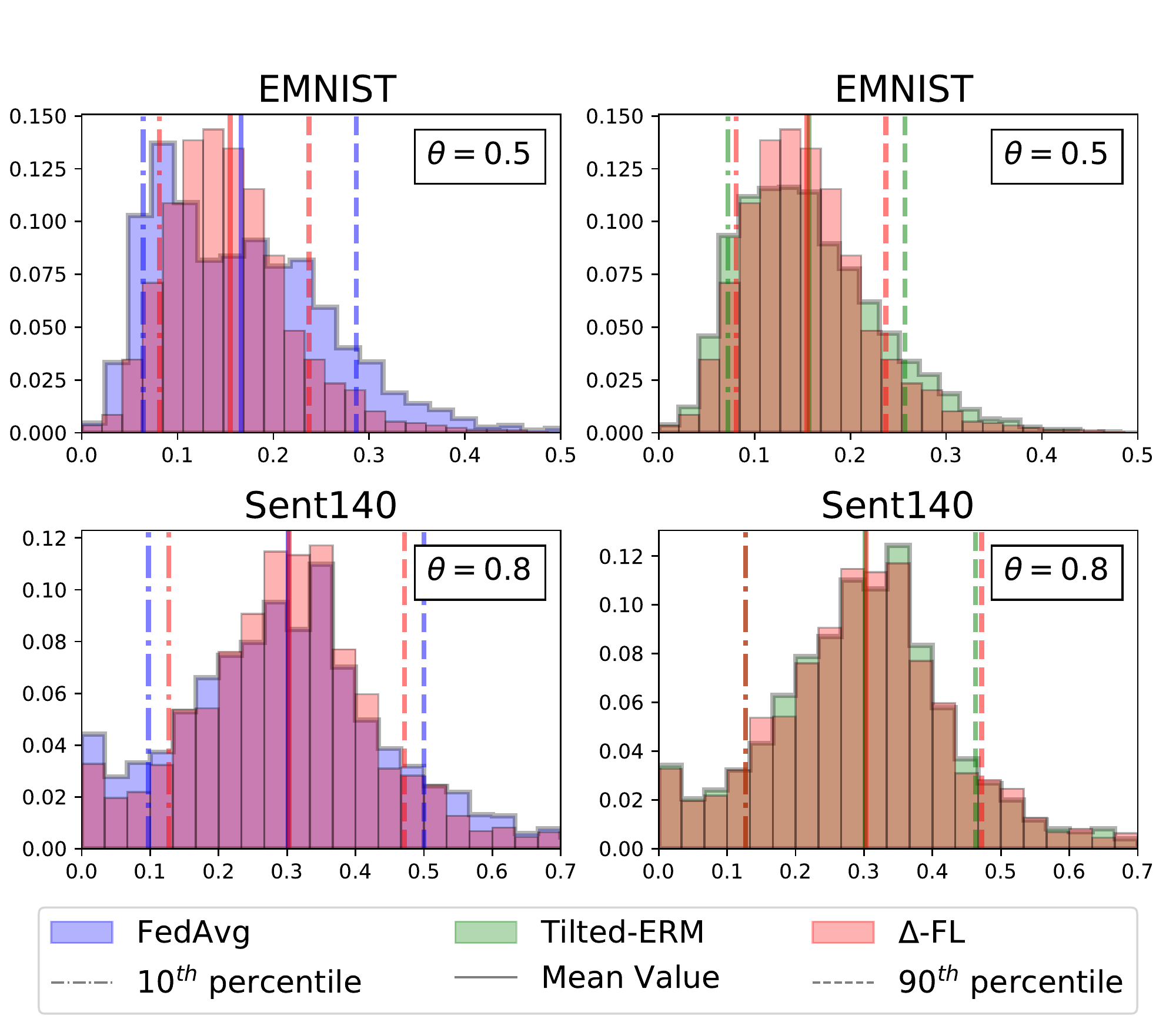}}
\caption{\small{
     Histogram of misclassification error on test clients for the EMNIST-ConvNet and Sent140-RNN.
     }}
\label{fig:main:expt:hist}
\end{figure*}

\begin{figure}[t]
\begin{minipage}[b]{0.99\linewidth}
  \centering
  \vspace{-5mm}
  {\adjincludegraphics[width=0.6\linewidth,trim={0 40pt 0 40pt},clip=true]{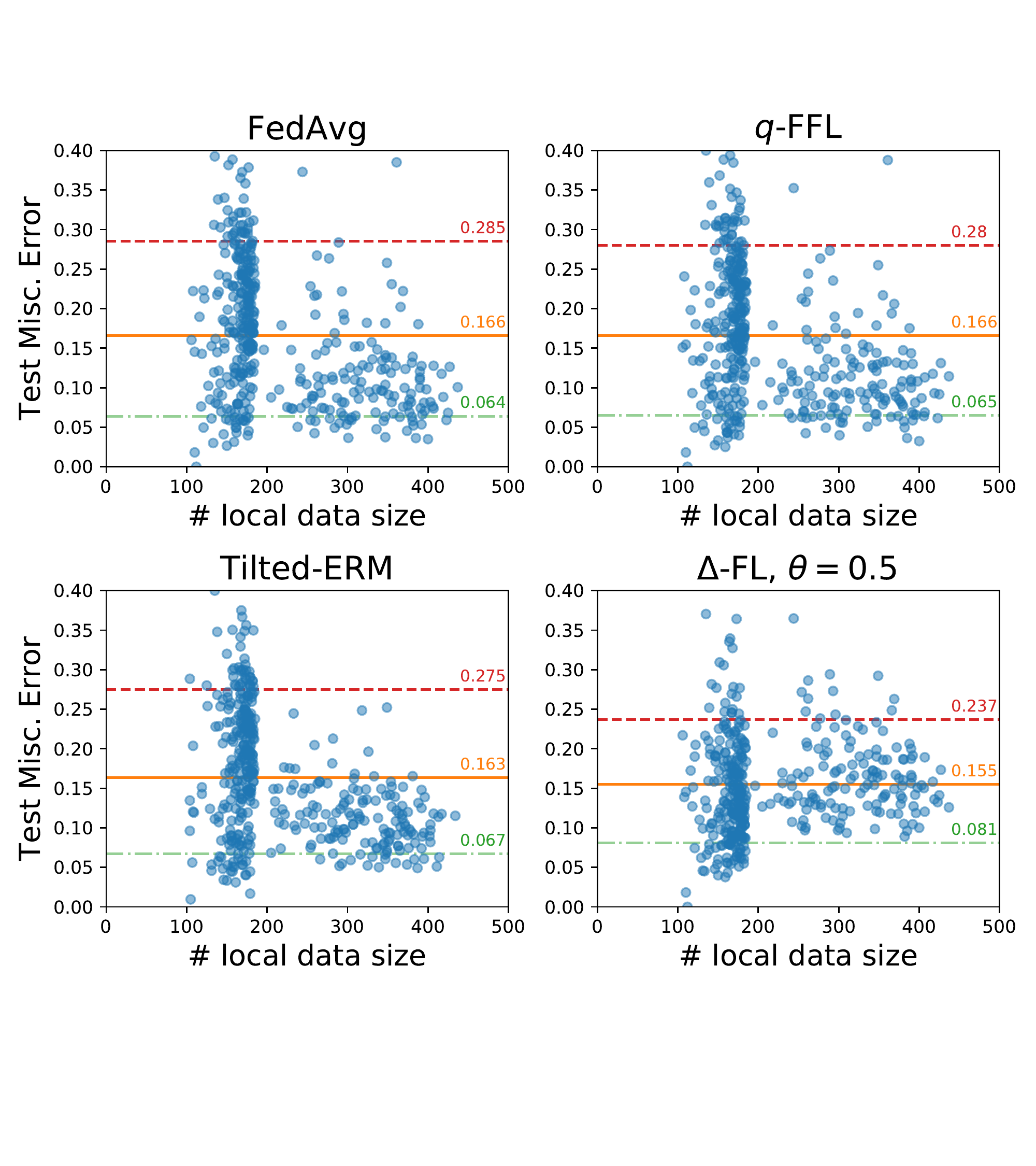}}
\end{minipage}
\vspace*{-9mm}
\caption{\small{
     Scatter plots of misclassification error on test clients against its data size for the EMNIST-ConvNet.
     }}
\label{fig:main:expt:scatter}
\end{figure}

\subsection{Exploring the Trade-off Between Average and Tail Error}
\label{sec:sfl:expt:heterogeniety}
We visualize in Figures~\ref{fig:main:expt:hist} and~\ref{fig:main:expt:scatter} the distribution of test errors to explore the trade-off various methods provide between the average error and the error on tail clients.

\myparagraph{\newfl yields improved prediction on tail clients}
This can be observed from the histogram of \newfl in \Cref{fig:main:expt:hist}, which
exhibits thinner tails than \fedavg or \term. 
We see that the vanilla FL objective of \fedavg sacrifices performance on the tail clients. \term does improve over \fedavg in this regard, but \newfl has a thinner right tail than \term, showing better handling of heterogeneity.

\myparagraph{\newfl yields improved prediction on data-poor clients}
We observe in \Cref{fig:main:expt:scatter} that
\term and \qffl mainly improve the performance on data-rich clients, that is clients with lots of data. On the other hand, 
\newfl gives a more significant reduction in misclassification error on data-poor clients, that is clients with little data ($< 200$ examples per client).

\subsection{\newfl and Model Personalization} \label{sec:sfl:expt:pers}

We now repeat the experiment of \Cref{sec:sfl:expt:results} with model personalization for the EMNIST ConvNet model. 

\myparagraph{Setup}
We personalize a model to a test client by finetuning a model trained either via FedAvg or \newfl on the particular test client's data at the end of federated training. This simple baseline is competitive with more sophisticated personalization algorithms~\cite{pillutla2022pfl}. 
Towards this end, we split the data on each test client into a training set used for the finetuning and a test set used to report the evaluation metrics. 
We finetune the model for $10$ epochs with the same local learning rate as at the end of federated training.

\begin{table*}[tb]
\small
\caption{{
Misclassification error \% of FedAvg and \newfl with model personalization on the EMNIST ConvNet model. 
Each table entry is the average
 over $5$ random seeds, while the subscript denotes the standard deviation. The boldfaced entries indicate the smallest error in each column.
}}
\label{table:sfl:pers}
\begin{center}
\begin{adjustbox}{max width=\linewidth}
\setlength{\tabcolsep}{8pt}
{\renewcommand{\arraystretch}{1.2}%
\begin{tabular}{lccccc}
\toprule
 		& \multicolumn{2}{c}{\textbf{Mean error}} 				& \multicolumn{2}{c}{\textbf{$\mathbf{90}$\textsuperscript{th} percentile error}} 		\\
 		  \cmidrule(lr){2-3} \cmidrule(lr){4-5}
			& 	\textbf{Before pers.} 			& 	\textbf{After pers.} 			& 	\textbf{Before pers.}			& 		\textbf{After pers.}						\\
\midrule

\fedavg 	& ${16.68}_{0.50}$ 		& \tabemph{} $\mathbf{5.43}_{0.12}$ 			& $28.44_{1.15}$ 		& ${8.71}_{0.19}$ 	\\

\newfl, 
$\theta=0.8$& $16.00_{0.44}$ 		& $5.44_{0.08}$ 			&  ${26.26}_{1.28}$& 
\tabemph{} $\mathbf{8.69}_{0.12}$ 	\\

\newfl, 
$\theta=0.5$& 
\tabemph{} $\mathbf{15.50_{0.31}}$ 		&  ${5.58_{0.07}}$ 	& \tabemph{} $\mathbf{23.61_{1.02}}$ 		& $8.76_{0.15}$ 	\\

\newfl, 
$\theta=0.1$& $16.05_{0.78}$ 		& $6.17_{0.11}$ 			& $24.58_{1.96}$ 		& $9.38_{0.06}$ 	\\

\bottomrule

\end{tabular}} %
\end{adjustbox}
\end{center}
\end{table*}

\myparagraph{Results}
The numerical results are given in \Cref{table:sfl:pers}. We observe that after model personalization, both FedAvg and \newfl models perform similarly, often within one standard deviation of each other. The mean error is marginally smaller for FedAvg while the $90$\textsuperscript{th} percentile error is marginally smaller for \newfl with $\theta=0.8$. The gap between these,  $0.01$ or $0.02$ percentage points, is smaller than the standard deviation, $0.1$ percentage points. 

\subsection{Differentially Private Quantile Estimation}
\label{sec:sfl:expt:privacy}

We study the privacy-utility tradeoff of \Cref{algo:sfl:quantile-dp:new}. 

\myparagraph{Setup}
We sample $n=256$ numbers from a uniform distribution over $[0, B]$ or a $\chi^2(4)$ distribution clipped to $[0, B]$ with $B=10$. We consider the performance of \Cref{algo:sfl:quantile-dp:new} by varying the number $b$ of bins and the ring size $M$. Since the communication cost of the protocol scales as the bit width $\log_2 M$, we display it instead in the plots. Recall that if our algorithm returns the $(1-\theta')$-quantile when we aim to find the $(1-\theta)$-quantile, then its quantile error is $\vert\theta - \theta'\vert$, cf. \eqref{eq:sfl:quantile-error}.  We plot the quantile error averaged over $\theta = 0.1, 0.2, \ldots, 0.9$, and the standard deviations are obtained from $10$ random runs. 

\begin{figure}[t]
  \adjincludegraphics[width=0.24\linewidth,trim={0 0pt 0 0pt},clip=true]{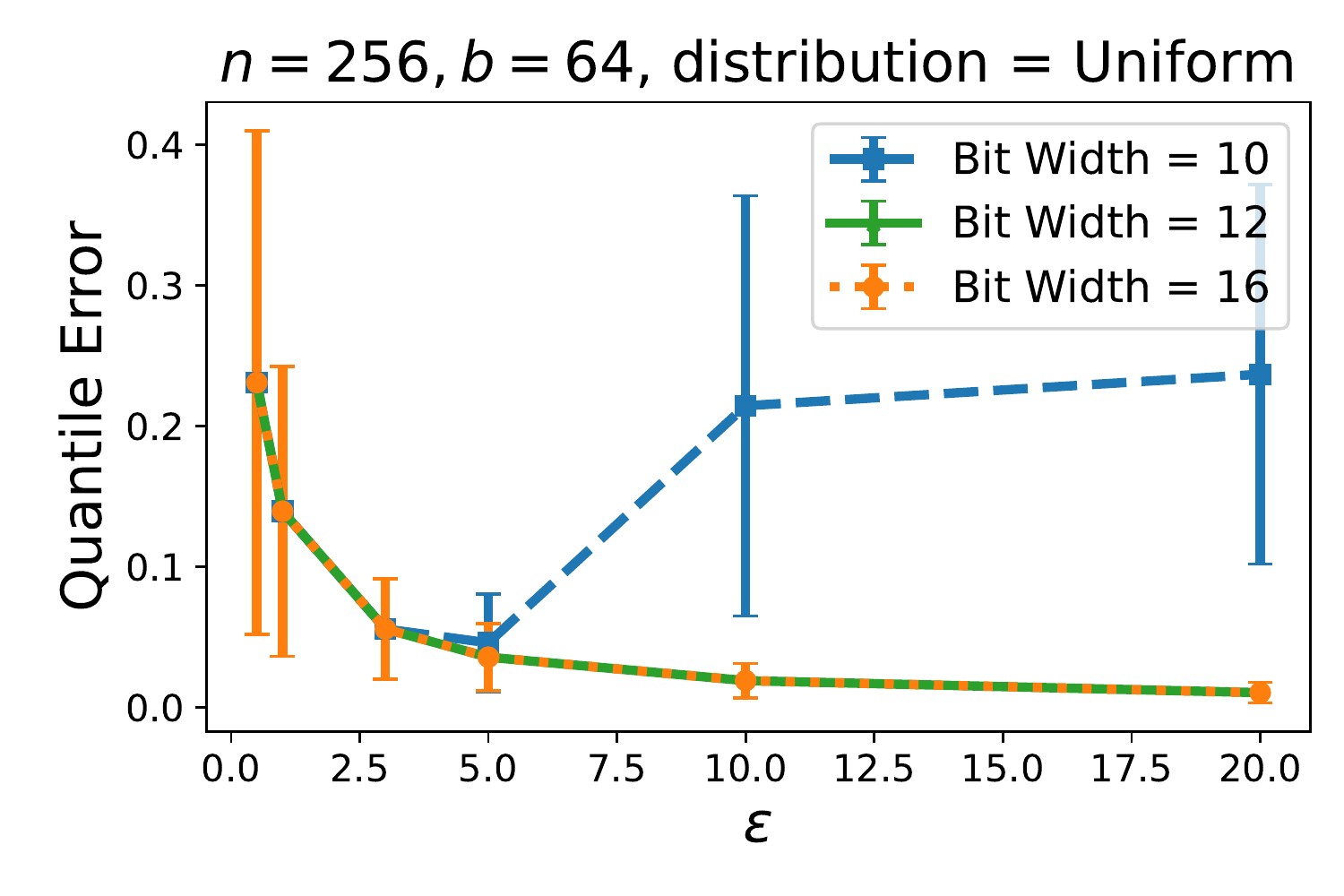} 
  \adjincludegraphics[width=0.24\linewidth,trim={0 0pt 0 0pt},clip=true]{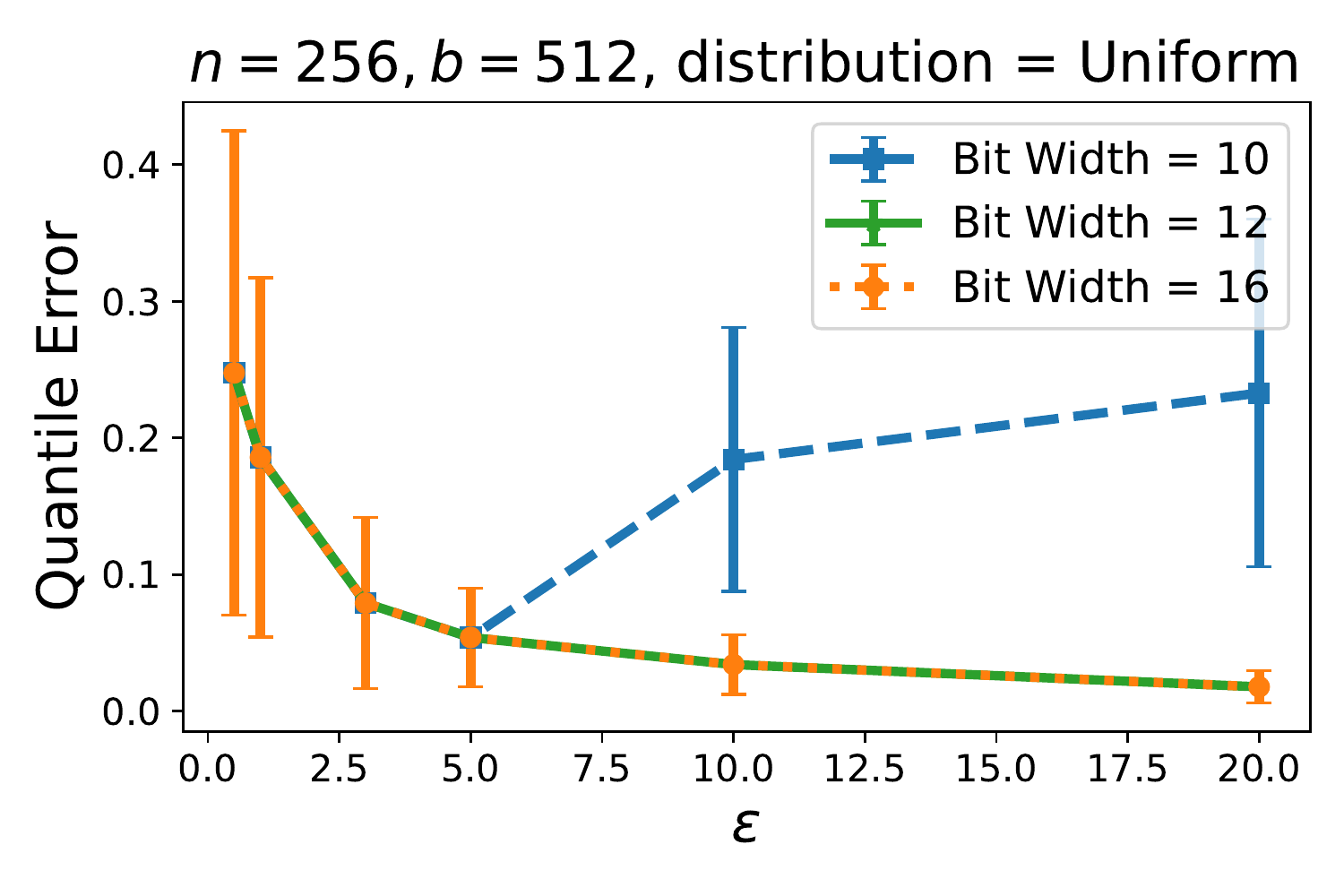}
   \adjincludegraphics[width=0.24\linewidth,trim={0 0pt 0 0pt},clip=true]{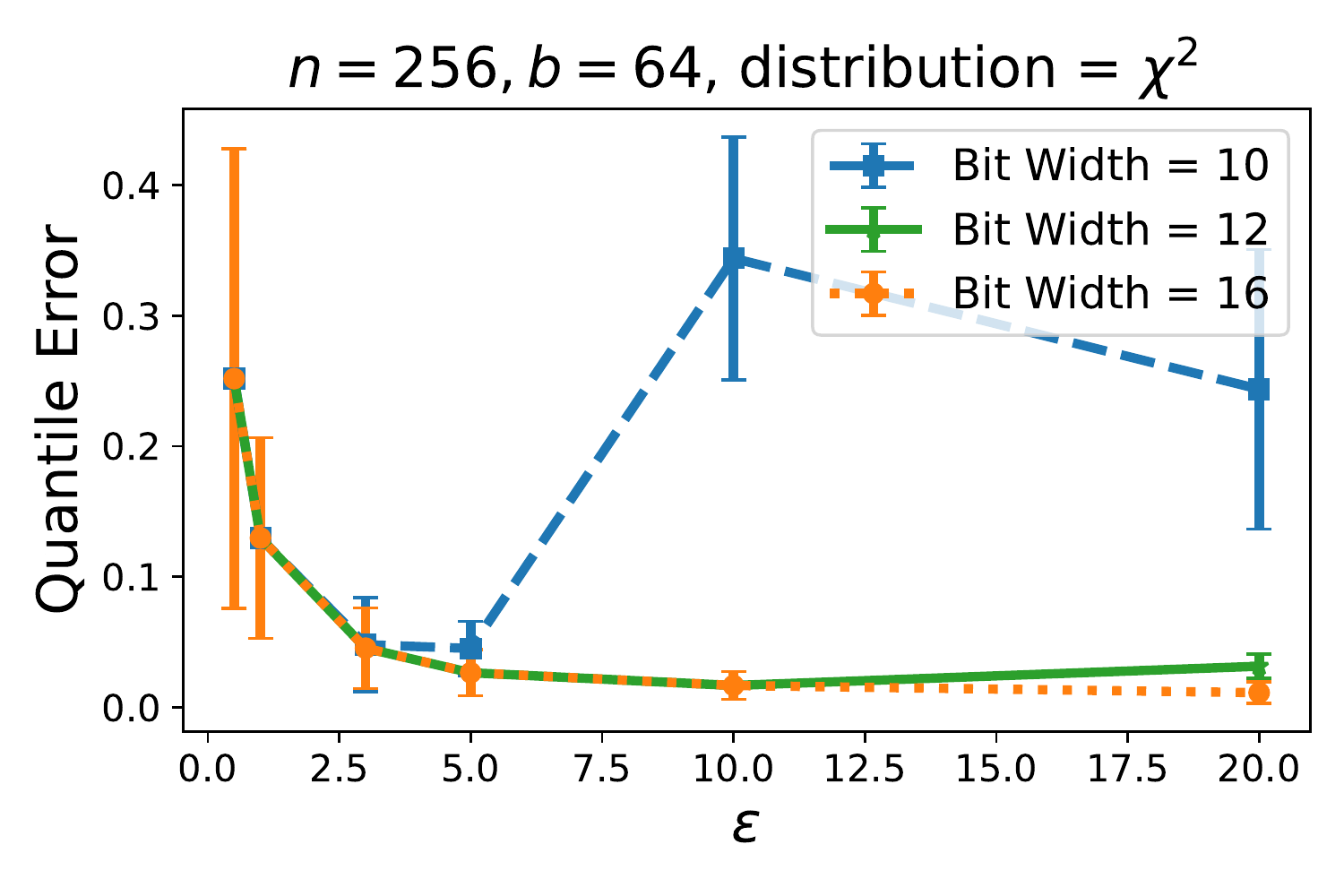} 
  \adjincludegraphics[width=0.24\linewidth,trim={0 0pt 0 0pt},clip=true]{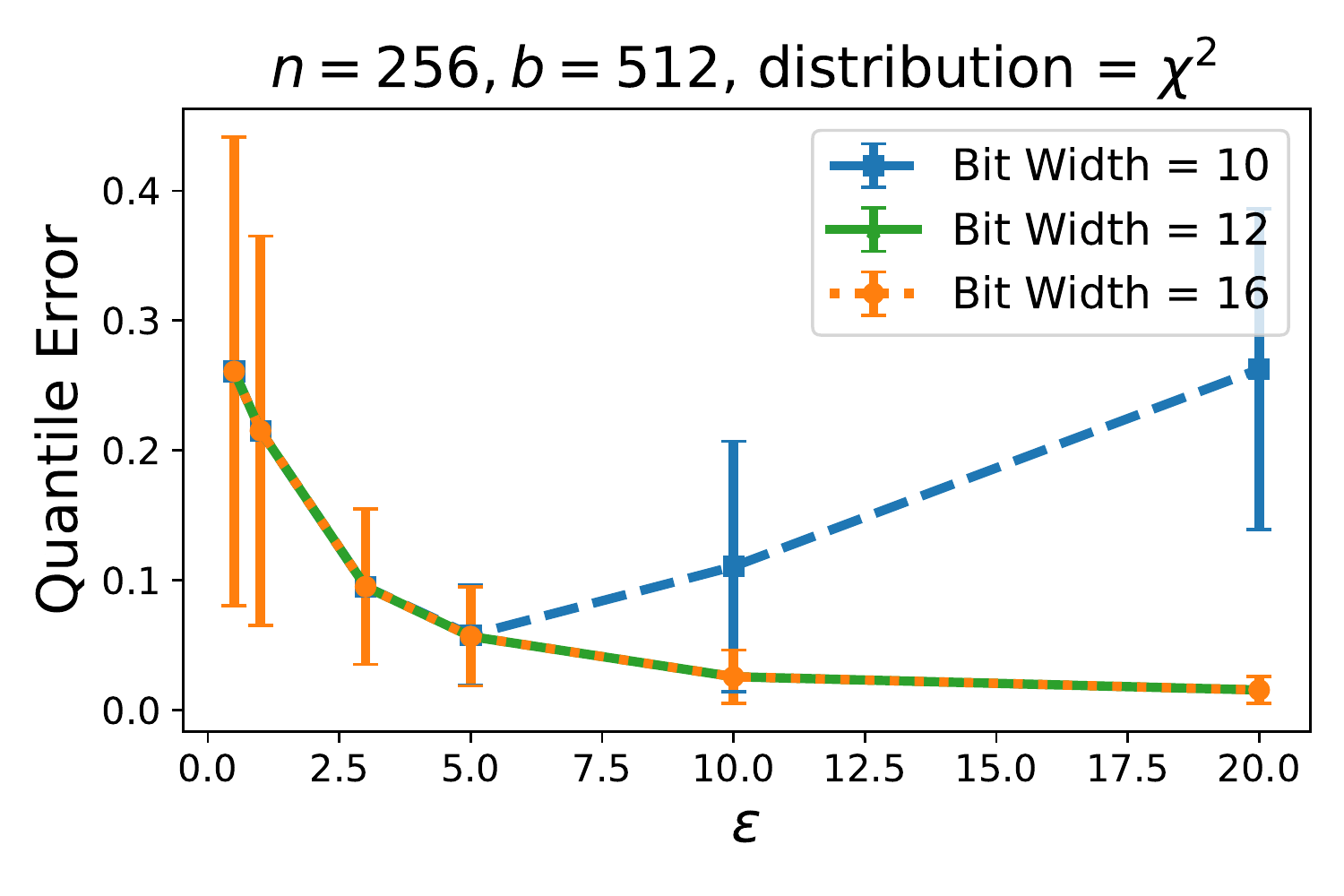}
\caption{\small{
     The quantile error (defined in \eqref{eq:sfl:quantile-error})
     incurred by \Cref{algo:sfl:quantile-dp:new} to estimate the quantile of $n=256$ numbers drawn from a uniform or $\chi^2(4)$ distribution with  $(\eps, 10^{-5})$-differential  privacy. 
     }}
\label{fig:sfl:expt:dpq}
\end{figure} 

\myparagraph{Results}
The results are given in \Cref{fig:sfl:expt:dpq}. 
For $n=256$ and $b=64$, we find that the quantile error is $0.14$ for the uniform distribution at 
$\epsilon \approx 1$; this means we might find the $36$\textsuperscript{rd} percentile or the $64$\textsuperscript{th} percentile instead of the median. 
This error quickly falls to $0.03$ at $\epsilon\approx 5$ at large enough bit widths. At a bit width of $10$, we incur errors due to the modular wraparound at $\eps \ge 5$.
The results are also qualitatively similar for other settings, although the quantile error is unsurprisingly higher at $b > n$. 

\subsection{End-to-end Differential Privacy with \newfl} 
\label{sec:sfl:expt:end-to-end-dp}
We now compare \newfl with \fedavg with end-to-end differential privacy on a synthetic classification dataset. 

\myparagraph{Dataset and Models}
The synthetic dataset contains $k=10$ classes in $d=20$ dimensions and $n=2500$ training clients. The class-conditional distribution $q(x \vert y=k) = \Ncal(\mu_k, I_d)$ is a Gaussian and is the same across all the clients while there is a label shift, i.e., $q_i(y)$ varies across clients. For the training clients, we have $q_i(y) = \text{Dir}(0.5)$ is a Dirichlet distribution with parameter $0.5$, while for validation and test clients, we have $q_i(y) = \text{Dir}(0.01)$. For each client, we sample $100$ examples from its data distribution. We refer to Appendix~\ref{sec:a:sfl:dp-expt} for details.

\myparagraph{Algorithms and Privacy Budgeting}
For the FedAvg baseline, we clip the model updates to an $\ell_2$ norm bound of $C$, which is a tunable hyperparameter. We add Gaussian noise $\Ncal(0, \sigma_w^2 I)$ --- thus, each update satisfies $\sigma_w^2/(2 C^2)$-concentrated differential privacy. To get a privacy bound across all the rounds, we use the generic bounds of \cite{zhu2019poisson} for privacy amplification by subsampling
and composing the privacy loss across the number of rounds of the algorithm. Given a fixed norm bound $C$, we select the noise scale $\sigma_w$ to get $(\eps, 1/n)$-differential privacy over the entire algorithm, where $\eps$ is provided as an input, and $n$ is the number of clients. 

Each round of \newfl involves quantile computation and weight aggregation: we use \Cref{algo:sfl:quantile-dp:new} to compute the quantile of the losses clipped to a tuned bound $B$ using a hierarchical histogram with $b$ bins.
We clip the weight updates to a norm bound $C$ and add Gaussian noise, similar to FedAvg. The total privacy loss is calculated by composing the privacy loss across both the quantile and weight updates, and the number of rounds together with amplification by subsampling using the bounds of \cite{zhu2019poisson}. 

We calculate the noise scales $\sigma_q$ of the quantile and $\sigma_w$ of the weight update so that 
(a)  the privacy budget for the quantile computation to be $r$ times the privacy budget of the weight update, where $r$ is a hyperparameter, and
(b) the overall algorithm satisfies $(\eps, 1/n)$-differential privacy.
We tune the loss bound $B$, norm bound $C$, the number of bins $b$, and the quantile privacy ratio $r$ to attain the best $90$\textsuperscript{th} percentile misclassification error across validation clients. For all experiments, we train for 1000 rounds with 100 clients per round and a fixed learning rate of 0.1. For further details on the algorithms, privacy budgeting, and hyperparameters, we refer to Appendix~\ref{sec:a:sfl:dp-expt}.

\begin{figure}[t]
\centering
\adjincludegraphics[width=0.75\linewidth,trim={0 0pt 0 0pt},clip=true]{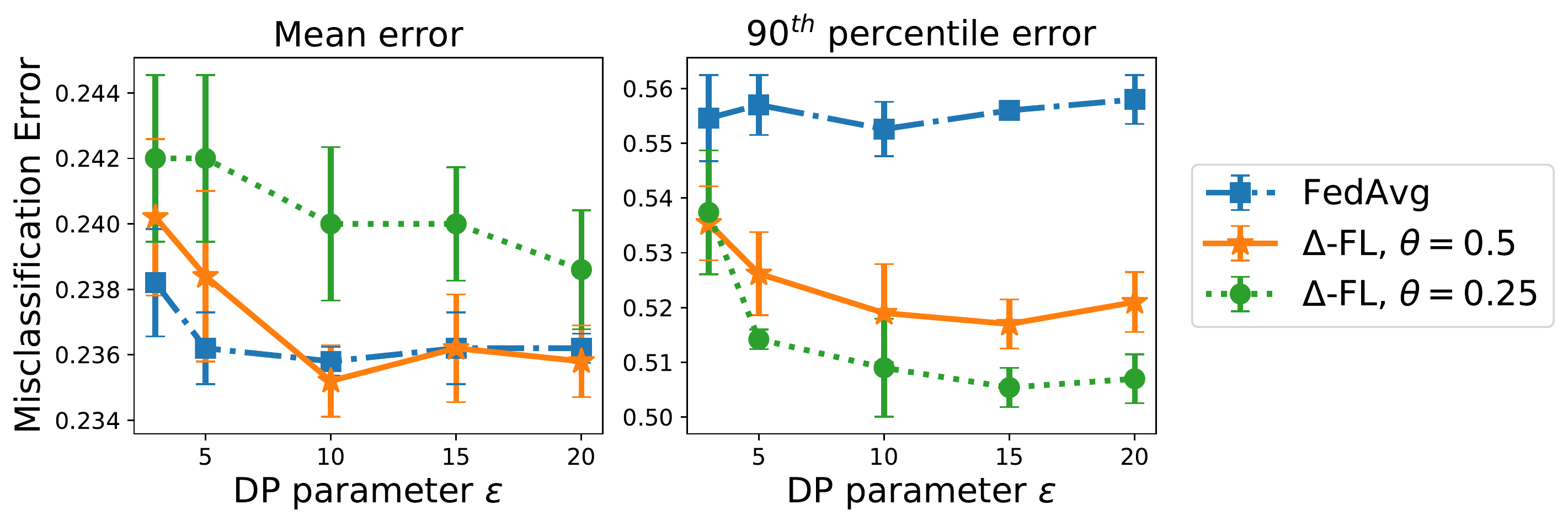}
\caption{\small{
     \newfl vs. FedAvg with $(\varepsilon, 1/n)$-differential privacy on a synthetic classification task in $\reals^{20}$ with $10$ classes and $n=2500$ clients. The error bars denote the standard deviation across $5$ random runs. 
     }}
\label{fig:sfl:expt:pri-ut-sfl}
\end{figure}

\myparagraph{Results: \newfl gives better tail performance under the same privacy budget}
The privacy-utility tradeoff of \newfl and FedAvg are shown in \Cref{fig:sfl:expt:pri-ut-sfl}. 
We see that \newfl with threshold level $\theta=0.5$ has a privacy-utility tradeoff within one standard deviation of FedAvg on the mean misclassification error while being $3.1$ percentage points better on the tail misclassification error as measured by its $90$\textsuperscript{th} percentile: $55.7\%$ for FedAvg versus $52.6\%$ for \newfl at $\eps = 5$. Smaller values of $\theta$, such as $\theta=0.25$ are $0.6$ percentage points worse on the mean error while being $1.2$ and $4.3$ percentage points better than $\theta=0.5$ and FedAvg respectively on the tail error. We note that the utility of \newfl degrades more at smaller $\eps$ when compared to FedAvg: $1.64$ percentage points for $\theta=0.5$ versus $0.2$ percentage points for FedAvg from $\eps=10$ to $\eps=3$ for the tail error. Despite this effect, the tail error for \newfl is smaller than FedAvg even at $\eps=3$.

\section{Conclusion} \label{sec:sfl:conclusion}

We present the \newfl framework that operates with heterogeneous clients while guaranteeing a minimal predictive performance to each client. 
\update{\newfl relies on a superquantile-based objective, parameterized by a tail threshold level, to optimize the tail statistics of the prediction errors on the client data distributions.}
We present a federated optimization algorithm that combines differentially private quantile estimation to filter out clients to run federated averaging steps. We derive finite time convergence guarantees 
of $O(1/\sqrt{T})$ in $T$ communication rounds in the nonconvex case and $O(\exp(-T/\kappa^{3/2}) + \kappa/T)$ in the strongly convex case with local condition number $\kappa$. We establish a utility bound of $O(\log^2 b /(\eps n))$ for $(\eps, \delta)$-differentially private quantile computation. 
Experimental results on federated learning benchmarks demonstrate the superior performance of \newfl over state-of-the-art baselines on the upper quantiles of the error on test clients, with particular improvements on data-poor clients, while being competitive on the mean error with and without differential privacy.

\subsection*{Acknowledgements}
The authors thank Peter Kairouz, Sewoong Oh, and Lun Wang for fruitful discussions.
The authors acknowledge support from NSF DMS 2023166,
DMS 1839371, CCF 2019844, the CIFAR program “Learning
in Machines and Brains”, faculty research awards, and a JP
Morgan Ph.D. fellowship. This work has been partially supported
by MIAI – Grenoble Alpes, (ANR-19-P3IA-0003).
The work was mainly performed while Krishna Pillutla was at the University of Washington, and Yassine Laguel was at the Universit\'{e} Grenoble Alpes.

\appendix

\section*{Appendix}

The outline of the appendix is as follows: 
\begin{itemize}
  \item \Cref{sec:a:sfl:convergence}: Convergence analyses and proofs of \Cref{thm:sfl:nonsmooth-nonconvex:main,thm:sfl-main:cvx}.
  \item \Cref{sec:a:sfl:expt}: Privacy analysis and proof of \Cref{thm:sfl:quantile:new}.
  \item \Cref{sec:a:sfl:expt}: Full experimental details and additional plots.
  \item \Cref{sec:a:sfl:dp-expt}: End-to-end \newfl --- details and numerical results. 
\end{itemize}

\section{Convergence Analysis} \label{sec:a:sfl:convergence}
Below, we restate and prove 
\Cref{thm:sfl:nonsmooth-nonconvex:main} as \Cref{thm:sfl:nonsmooth-nonconvex} in \Cref{sec:a:sfl:expt:ncvx}
and
\Cref{thm:sfl-main:cvx} as \Cref{thm:a:sfl-cvx} in \Cref{sec:a:sfl:expt:cvx},

\subsection{Review of Notation}
Here, we review the notation of the variants of the functions $F_i$ and $F_\theta$ in \Cref{tab:notation}.

\begin{table}[t]
\caption{Review of notation.} 
\label{tab:notation}
\begin{center}
\setlength{\tabcolsep}{8pt}
{\renewcommand{\arraystretch}{1.4}%
\begin{tabular}{p{5em} p{30em}}
     \toprule 
     \textbf{Function} & \textbf{Description} \\
     \midrule  
     $F_i$ & Loss function of client $i$ \\
     $\tilde F_i$ & Loss plus regularization on client $i$: $\tilde F_i(w) = F_i(w) + \frac{\lambda}{2}\normsq{w}$ 
     \\ 
     \midrule
     $F_\theta$ & The main objective of \newfl, defined in \eqref{eq:method:obj:minmax} \\
     $F_{\theta, S}$ & The analogue of $F_\theta$ defined on only on a sample $S$ of clients \\
     $\overline F_\theta$ & 
     Averaged minibatch objective:
     $\overline F_\theta(w) = \expect_S [F_{\theta, S}(w)]$ where the expectation is over uniform subsamples of clients of size $\vert S \vert =m$ \\
     $\overline \Phi_\theta^\mu$ & 
     The Moreau envelope of $\overline F_\theta$; see \eqref{eq:sfl:moreau-env} \\
     $\hat F_\theta$ & The variant of the \newfl objective computed with a tail mean, and used to formalize the connection between \Cref{algo:sfl-private,algo:sfl-new} \\
     $F_\theta^\nu$ & 
     Smoothing of $F_\theta$ using the KL divergence; see \eqref{eq:sfl-entropic-smoothing}\\
     \bottomrule
\end{tabular}} %
\end{center}
\end{table}

\subsection{Convergence Analysis: Non-convex Case}
\label{sec:a:sfl:expt:ncvx}

We review some definitions of subdifferentials and weak convexity before we get to the main theorem.

\myparagraph{Nonconvex Subdifferentials}
We start by recalling the definition of subgradients for nonsmooth functions (in finite dimension), following the terminology of\;\cite{rockafellar2009variational}.
Consider a function $\psi\colon\reals^d\rightarrow\reals\cup\{+\infty\}$
and a point $\bar{w}$ such that $\psi(\bar w) <+\infty$.
The regular (or Fr{\'e}chet) subdifferential of $\psi$ at $\bar w$ is defined by
\[
\partial \psi(\bar{w}) =\left\{s\in \mathbb{R}^d: ~\psi(w)\geq
\psi(\bar{w})+ \inp{s}{w-\bar{w})} + o(\|{w-\bar{w}}\|)\right\}.
\]
The regular subdifferential thus corresponds to the set of gradients of smooth functions that are below $\psi$ and coincide with it at $\bar\x$.
These notions generalize (sub)gradients of both smooth functions and convex functions: it reduces to the singleton $\{\grad \psi(\bar \x)\}$ when $\psi$ is smooth and to the standard subdifferential from convex analysis when $\psi$ is convex.

\myparagraph{Weak Convexity}
We recall the notion of weak convexity, which is one way of characterizing functions that are ``close'' to convex. 
A function $\psi:\reals^d\to \reals$ is said to be $\eta$-weakly convex if the function $w \mapsto \psi(w) + (\eta/2)\normsq*{w}$ is convex~\cite{nurminskii1973quasigradient}. 
The class of weakly convex functions includes all convex functions (with $\eta =0$) and all $L$-smooth functions (with $\eta=L$). 

Weak convexity also admits an equivalent first-order condition: for any $w, z \in \reals^d$ and $s \in \partial \psi(w)$, we have, 
\begin{align} \label{eq:weak-cvx-first-order}
    \psi(z) \ge \psi(w) + \inp{s}{z-w}  - \frac{\eta}{2}\normsq*{z-w} \,.
\end{align}

\noindent Weak convexity will feature in our developments in two ways: 
\begin{itemize}[nolistsep,leftmargin=\widthof{(a) }]
    \item In our case, both $F_\theta$ as well as $F_{\theta, S}$ are $L$-weakly convex, since each can be written as the maximum of a family of $L$-smooth functions~\cite[Lemma~4.2]{drusvyatskiy2019efficiency}.
    \item The prox operator for weakly convex functions is well-defined. Let $\psi$ be a $\eta$-weakly convex function. Its proximal or prox operator, with parameter $\mu > 0$, is defined as
    \[
        \prox_{\psi/\mu}(w) = 
        \argmin_{z} \left\{\psi(z) + \frac{\mu}{2}\normsq*{w-z}\right\} \,.
    \]
    It is well-defined (i.e., the argmin exists and is unique) for $\mu > \eta$, since the function inside the argmin is $(\mu - \eta)$-strongly convex. 
\end{itemize}

\noindent In nonsmooth and nonconvex optimization of weakly convex functions, we are interested in finding stationary points w.r.t. the regular subdifferential, i.e., points $w$ satisfying $0 \in \partial \psi(w)$. A natural measure of near-stationarity is, therefore,
\[
    \mathrm{dist}(0, \partial\psi(w)) = \inf_{s \in \partial  \psi(w)} \norm{s} \,.
\]

\myparagraph{Moreau Envelope}
Given a parameter $\mu > 0$, we define the Moreau envelope of $\overline F_\theta$ as 
\[
    \overline \Phi_\theta^\mu(w) = \inf_{z} \left\{
    \overline F_\theta(z) + \frac{\mu}{2} \normsq*{w-z} 
    \right\} \,.
\]
The Moreau envelope is well-defined since $\overline F_\theta$ is bounded from below by our assumptions. 
We will use two standard properties of the Moreau envelope:
\begin{itemize}[nolistsep]
    \item Since $\overline F_{\theta, S}$ is $L$-weakly convex, we have that its Moreau envelope $\overline \Phi_\theta^\mu(w)$ is continuously differentiable for $\mu > L$ with
    \begin{align} \label{eq:sfl:sudiff:prox-grad}
        \grad \overline \Phi_\theta^\mu(w) = \mu\left(w - \prox_{\overline F_\theta / \mu}(w)\right) \,.
    \end{align}
    \item The stationary points of $\overline \Phi_\theta^\mu$ and $\overline F_\theta$ coincide and $\inf \overline \Phi_\theta^\mu = \inf \overline F_\theta$ for $\mu > L$.
    \item We have for all $\mu > 0$ that $\overline \Phi_\theta^\mu(w) \le \overline F_\theta(w)$.
\end{itemize}

\myparagraph{Notation}
Let $S=S\pow{t}$ denote the random set of clients selected in round $t$ of \Cref{algo:sfl-new}. 
We define 
\begin{align}  \label{eq:sfl:subdiff:defn}
    \tilde \grad F_{\theta, S}(w\pow{t}) = \sum_{i \in S} \pi_i\pow{t} \grad F_i(w\pow{t}) \,,
\end{align}
where $\pi_i\pow{t} \in \argmax_{\pi \in \Pcal_{\theta, S}} \sum_{i \in S} \pi_i F_i(w\pow{t})$ is selected as in line~\ref{line:sfl-new:reweigh} of \Cref{algo:sfl-new}.
 A key consequence of the chain rule~\cite[][Thm. 10.6]{rockafellar2009variational} is
\begin{align} \label{eq:sfl:subdiff:chain}
    \tilde \grad F_{\theta, S}(w\pow{t}) \in \partial F_{\theta, S}(w\pow{t}) \,.
\end{align}

\myparagraph{Convergence Analysis}
We now state and prove the convergence result in the nonconvex case. 

\begin{theorem} \label{thm:sfl:nonsmooth-nonconvex}
    Fix the number of local steps $\tau$ and the number of rounds $T$, fix $\mu = 2L$ and set the learning rate 
    \[
        \gamma = \min \left\{ \frac{1}{4\tau L}, \frac{1}{\tau \sqrt{T}} \sqrt{\frac{\Delta F_0}{L G^2}}, \frac{1}{\tau T^{1/3}} \left( 
            \frac{\Delta F_0}{32 L^2 G^2 (1-\tau^{-1})}
        \right)^{1/3}   \right\} \,,
    \]
    where we denote $\Delta F_0 = \overline \Phi^\mu_\theta(w\pow{0}) - \inf \overline \Phi^\mu_\theta \le \overline F_\theta(w\pow{0}) - \inf \overline F_\theta$.
   Let $\hat w$ be sampled uniformly at random from $\{w\pow{0}, \ldots, w\pow{T-1}\}$.
    Ignoring absolute constants, we have the bound, 
    \[
        \expect\normsq*{\grad \overline \Phi^\mu_\theta(\hat w)} 
        \le 
        \sqrt{\frac{\Delta F_0 LG^2}{T}}
        + \left(\frac{\Delta F_0 L G (1-\tau^{-1})^{1/2} }{T}\right)^{2/3}
        + \frac{\Delta F_0 L}{T}\,.
    \]
\end{theorem}

\begin{proof} %
We start with some notation. 
Throughout, we denote $z\pow{t}$ as the proximal point of $w\pow{t}$: 
\[
    z\pow{t} = \prox_{\overline F_\theta/\mu}(w\pow{t})
    = \argmin_z \left\{ \overline F_\theta(z) + \frac{\mu}{2} \normsq*{z - w\pow{t}}\right\} \,.
\]
Let $\Fcal\pow{t}$ denote the sigma algebra generated by $w\pow{t}$ and define $\expect_t[\cdot] = \expect[\cdot \mid \Fcal\pow{t}]$. By definition, we have that $z\pow{t}$ is also $\Fcal\pow{t}$-measurable. 

We use the update $w\pow{t+1}=w\pow{t}-\gamma \sum_{i\in S} \pi_i\pow{t} \sum_{k=0}^{\tau-1} \grad F_i(w_{i, k}\pow{t})$ to get
\begin{align} \nonumber
    \overline \Phi_\theta^\mu & (w\pow{t+1})
    = 
    \min_z \left\{ \overline F_\theta(z) + \frac{\mu}{2}\normsq*{z - w\pow{t+1}} \right\} \\ \nonumber
    &\le \overline F_\theta(z\pow{t}) + \frac{\mu}{2} \normsq*{z\pow{t} - w\pow{t+1}} \\ \nonumber
    &  \, 
    \begin{aligned} 
    =   \overline F_\theta(z\pow{t}) + \frac{\mu}{2} \normsq*{z\pow{t} - w\pow{t}} 
    &+ \mu \gamma\inp*{z\pow{t} - w\pow{t}}{\sum_{i \in S}\pi_i\pow{t} \sum_{k=0}^{\tau-1} \grad F_i(w_{i, k}\pow{t})} \\  & +  \frac{\mu\gamma^2}{2} \normsq*{\sum_{i \in S} \pi_i\pow{t} \sum_{k=0}^{\tau-1} \grad F_i(w_{i, k}\pow{t})}  
    \end{aligned}\\
    & \, 
    \begin{aligned} 
    = \overline \Phi_\theta^\mu(w\pow{t})
    & + \underbrace{\mu \gamma\inp*{z\pow{t} - w\pow{t}}{\sum_{i \in S}\pi_i\pow{t} \sum_{k=0}^{\tau-1} \grad F_i(w_{i, k}\pow{t})}}_{=: \Tcal_1} 
    \\ & + \underbrace{\frac{\mu\gamma^2}{2} \normsq*{\sum_{i \in S} \pi_i\pow{t} \sum_{k=0}^{\tau-1} \grad F_i(w_{i, k}\pow{t})}  }_{=: \Tcal_2} \,.
    \end{aligned}
    \label{eq:sfl:nonsmooth-pf:main}
\end{align}
For $\Tcal_1$, we consider the effect of a single update with a learning $\tau\gamma$: 
\[
    \Tcal_1' := \mu \tau \gamma\inp*{z\pow{t} - w\pow{t}}{\sum_{i \in S}\pi_i\pow{t} \grad F_i(w\pow{t})} \,,
\]
so that the difference $\Tcal_1 - \Tcal_1'$ is the effect of the drift introduced by taking multiple local steps. We bound the first order term $\Tcal_1'$, the drift term $\Tcal_1 - \Tcal_1'$ and the second order term $\Tcal_2$ separately. 

\myparagraph{Bounding the first order term $\Tcal_1'$}
By definition of the weights $\pi_i\pow{t}$, we have 
$\sum_{i\in S} \pi_i\pow{t} \grad F_i(w\pow{t}) = \tilde \grad F_{\theta, S}(w\pow{t}) \in \partial F_{\theta, S}(w\pow{t})$, see also \eqref{eq:sfl:subdiff:defn}. 
This allows us to invoke the weak convexity of $F_{\theta, S}$, in particular \eqref{eq:weak-cvx-first-order}, to bound
\begin{align*}
    \frac{\Tcal_1'}{\mu\tau\gamma}
    = \inp*{z\pow{t} - w\pow{t}}{\tilde\grad F_{\theta, S}(w\pow{t})}
    \le
    F_{\theta, S}(z\pow{t}) - F_{\theta, S}(w\pow{t}) + \frac{L}{2} \normsq*{z\pow{t}-w\pow{t}} \,.
\end{align*}
Taking an expectation conditioned on $\Fcal\pow{t}$ (i.e., over the randomness in $S$), 
we get $\expect_t[F_{\theta,S}(w\pow{t})] = \overline F_\theta(w\pow{t})$. 
Further, since
 $z\pow{t}$ is $\Fcal\pow{t}$-measurable, 
 we also have $\expect_t[F_{\theta,S}(z\pow{t})] = \overline F_\theta(z\pow{t})$. 
 That gives, 
\begin{align*}
\frac{1}{\mu\tau\gamma} \expect_t[\Tcal_1']
    &\le
    \left(\overline F_\theta(z\pow{t}) + \frac{\mu}{2}\normsq*{z\pow{t}-w\pow{t}}\right)
    - 
    \overline F_\theta(w\pow{t}) 
    - \frac{\mu-L}{2}\normsq*{z\pow{t}-w\pow{t}} \,.
\end{align*}
Note that the function
\[
    h(z) := \overline F_\theta(z) + \frac{\mu}{2} \normsq*{z - w\pow{t}}
\]
is $(\mu-L)$-strongly convex and $z\pow{t}$ is its minimizer. This gives, 
\[
    h(w\pow{t}) - h(z\pow{t}) \ge \frac{\mu - L}{2} \normsq*{z\pow{t} - w\pow{t}} \,,
\]
so that we have the bound
\begin{align} \label{eq:sfl:nonsmooth-pf:A}
    \frac{1}{\mu\tau\gamma}\expect_t[\Tcal_1']
    &\le 
    - (\mu-L)\normsq*{z\pow{t}-w\pow{t}} 
    \stackrel{\eqref{eq:sfl:sudiff:prox-grad}}{=} - \frac{\mu - L}{\mu^2} \normsq*{\grad \overline \Phi_\theta^\mu(w\pow{t})}
    \,.
\end{align}
\myparagraph{Bounding the effect of the drift $\Tcal_1 - \Tcal_1'$}
The contribution of $k$\textsuperscript{th} local step to the drift $\Tcal_1 - \Tcal_1'$ can be bounded as
\begin{align*}
    \Bigg\vert &\inp*{z\pow{t} - w\pow{t}}{\sum_{i \in S}\pi_i\pow{t} \left(\grad F_i(w_{i, k}\pow{t}) - \grad F_i(w\pow{t})\right)}\Bigg\vert  \\
    &\stackrel{\text{(i)}}{\le} \frac{\mu-L}{2} \normsq*{z\pow{t}-w\pow{t}} + \frac{1}{2(\mu-L)} \normsq*{\sum_{i \in S} \pi_i\pow{t}\left(\grad F_i(w_{i, k}\pow{t}) - \grad F_i(w\pow{t}) \right) } \\
    &\stackrel{\text{(ii)}}{\le} \frac{\mu-L}{2} \normsq*{z\pow{t}-w\pow{t}} + \frac{1}{2(\mu-L)} \sum_{i \in S} \pi_i\pow{t} \normsq*{\grad F_i(w_{i, k}\pow{t}) - \grad F_i(w\pow{t}) } \\
    &\stackrel{\text{(iii)}}{\le} \frac{\mu-L}{2} \normsq*{z\pow{t}-w\pow{t}} + \frac{L^2}{2(\mu-L)} \sum_{i \in S} \pi_i\pow{t} \normsq*{w_{i, k}\pow{t} - w\pow{t} } \,.
\end{align*}
Here, we first used (i) the Cauchy-Schwarz inequality, (ii) Jensen's inequality, and (iii) the smoothness of $F_i$. 
Summing this over $k$, we get the bound
\begin{align}\nonumber
    \expect_t \left\vert \Tcal_1 - \Tcal_1'\right\vert
    &\le \frac{\tau \gamma(\mu - L)}{2\mu} \normsq*{\grad \overline \Phi_\theta^\mu(w\pow{t})} + \frac{\mu \gamma L^2}{2(\mu-L)} d\pow{t}  \\
    &\le \frac{\tau \gamma(\mu - L)}{2\mu} \normsq*{\grad \overline \Phi_\theta^\mu(w\pow{t})} + \frac{4\mu \tau^3\gamma^3 G^2}{\mu-L} (1-\tau\inv) \,,
     \label{eq:sfl:nonsmooth-pf:B}
\end{align}
where we bounded 
$d\pow{t} := \expect_t \left[\sum_{i \in S} \sum_{k=0}^{\tau-1} \pi_i\pow{t} \normsq*{w_{i, k}\pow{t} - w\pow{t} } \right]$ 
by \Cref{prop:sfl:client-drift}. 

\myparagraph{Bounding the second order term $\Tcal_2$}
Next, we bound $\Tcal_2$ as 
\begin{align} \label{eq:sfl:nonsmooth-pf:C}
    \Tcal_2 &= \frac{\mu\gamma^2}{2}\normsq*{\sum_{i \in S} \pi_i\pow{t} \sum_{k=0}^{\tau-1} \grad F_i(w_{i, k}\pow{t})} \le \frac{\mu\gamma^2 \tau}{2} \sum_{i \in S} \pi_i\pow{t} \sum_{k=0}^{\tau-1} \normsq*{\grad F_i(w_{i, k}\pow{t})} 
    \nonumber
    \\& \le \frac{\mu\gamma^2 \tau^2 G^2}{2} \,,
\end{align}
where we used Jensen's inequality and $\normsq*{\grad F_i(w_{i, k}\pow{t})} \le G^2$ since $F_i$ is $G$-Lipschitz. 

\myparagraph{One step update and telescoping the bound}
Plugging \eqref{eq:sfl:nonsmooth-pf:A} to \eqref{eq:sfl:nonsmooth-pf:C} into \eqref{eq:sfl:nonsmooth-pf:main}, we have, 
\begin{align*}
    \expect_t\left[\overline \Phi_\theta^\mu(w\pow{t+1})\right]
    \le 
    \overline \Phi_\theta^\mu(w\pow{t})
    &- \frac{\tau\gamma(\mu-L)}{2\mu} \normsq*{\grad \overline \Phi_\theta^\mu(w\pow{t})}
    \\& + \frac{\mu \gamma^2 \tau^2 G^2}{2}\left( 1 + \frac{8 L^2 \gamma}{\mu - L}(\tau-1) \right) \,.
\end{align*}
Finally, 
taking an unconditional expectation,
summing this up over $t=0$ to $T-1$ and rearranging gives us the bound
\[
        \expect\normsq*{\grad \overline \Phi^\mu_\theta(\hat w)} 
        \le \frac{4 \Delta F_0}{\tau \gamma T} + 4 \tau \gamma L G^2\left( 1 + 8L \gamma(\tau-1)\right) \,,
\]
where we plugged in $\mu = 2L$.
Plugging in the choice of $\gamma$ (cf. \Cref{lem:sfl:best-params-2}) completes the proof.
\end{proof}

\subsection{Convergence Analysis: Strongly Convex Case} \label{sec:a:sfl:expt:cvx}
The fully specified version of \Cref{thm:sfl-main:cvx} is the following.  

\begin{theorem}[Convergence rate, Strongly Convex Case] \label{thm:a:sfl-cvx}
Suppose that each $F_i$ is convex and the regularization parameter satisfies $0 < \lambda < L$. Define notation $\kappa = (L + \lambda) / \lambda$, $w^\star = \argmin_w F_\theta(w)$ and $\Delta_0 = \normsq{w\pow{0} - w^\star}$. Assume also that the number of rounds is $T \ge 16 \kappa^{3/2}$. 
Fix a smoothing parameter $\nu > 0$ as 
\[
    \nu = \frac{8 G^2 \delta}{\lambda} \big(1 \vee 32 \kappa^2 \delta\big) \,,
\]
where $\delta > 0$ is given by
\[
    \delta = \min\left\{  
        \frac{1}{16 \kappa^{3/2}}, 
        \frac{1}{T} \left( 1 \vee \log \frac{C T}{\log m}\right), 
        \frac{1}{T} \left( 1 \vee \log \frac{CT^2}{\kappa^2 \log m}\right) 
    \right\} \,,
\]
and $C = \lambda^2 \Delta_0 / G^2$. 
Letting $L' = L + \lambda + G^2/\nu$, fix a learning rate 
\begin{align*}
    \gamma = \min \Bigg\{
        \frac{1}{4\tau L'}, \,
        \frac{1}{8\tau \kappa \sqrt{2 \lambda L'}},  \,
        \frac{1}{\lambda \tau T} \left(1 \vee \log C \theta m T \right), 
        \frac{1}{\lambda \tau T} \left(1 \vee \log \frac{C T^2}{\kappa^2(1-\tau^{-1})} \right)^2
    \Bigg\}\,.
\end{align*}
Consider the sequence $(w\pow{t})_{t=0}^T$ produced by \Cref{algo:sfl-new} run with smoothing parameter $\nu$ and learning rate $\gamma$ chosen as above, and the corresponding averaged iterate
\[
    \overline w\pow{T} := \frac{\sum_{t=0}^T  w\pow{t} \left(1 - \frac{\lambda\tau\gamma}{2}\right)^{-(1 + t)}}{\sum_{r=0}^T \left(1 - \frac{\lambda\tau \gamma}{2}\right)^{-(1 + r)}} \,.
\]
Then, ignoring absolute constants, we have, 
\begin{align*}
    \expectation{F_\theta(\overline w\pow{T}) - F_\theta(w^\star)}{} \leq 
    & \,\,  \lambda \normsq{w\pow{0} - w^\star} \exp\left(- \frac{T}{16 \kappa^{3/2}} \right) + \frac{B}{\sqrt{\theta m}} \\
    &+ \frac{G^2}{\lambda T}\left( 
        \frac{1}{\theta m} + \log m \right)  \, \left(1 \vee \log \frac{\lambda^2 \Delta_0 \theta m T }{G^2} 
    \right) \\
    &+ \frac{G^2 \kappa^2}{\lambda T^2}\left(
        1-\tau^{-1} + \log m  \right)\left(1 \vee \log \frac{\lambda^2 \Delta_0 T^2}{G^2 \kappa^2}
    \right)^2 \,.
\end{align*}
\end{theorem}

We review some notation before giving the proof.
\myparagraph{Notation}
Analogous to the smoothing $F_\theta^\nu$ of $F_\theta$, we define the smoothing of the sample version $F_{\theta, S}$ as
\begin{align} \label{eq:sfl-entropic-smoothing-sample}
    F_{\theta, S}^\nu(w) = \max_{\pi \in \mathcal{P}_{\theta, S}} 
        \left\{ \sum_{i \in S} \pi_i F_i(w) - \nu D_S(\pi) \right\} + \frac{\lambda}{2}\normsq{w} \,,
\end{align}
From Danskin's theorem~\citep[Proposition B.25]{bertsekas1999nonlinear}, we get the expression of its gradient as 
\begin{align} \label{eq:sfl:smooth-grad:sample}
    \grad F_{\theta, S}^\nu(w\pow{t}) = \sum_{i \in S} \pi_i\pow{t} \grad \tilde F_i(w\pow{t}) \,,
\end{align}
where $\pi\pow{t}$ attains the unique argmax in \eqref{eq:sfl-entropic-smoothing-sample} (see also \eqref{eq:sfl:smooth-argmax} for the definition).

We define the averaged superquantile as 
\begin{align} \label{eq:def_avg_superquantile}
    \overline F_\theta^\nu(w) = \expect_{S \sim U_m}[F_{\theta, S}^\nu(w)] \,,
\end{align}
where $U_m$ is the uniform distribution over subsets of $[n]$ of size $m$. 
Finally, let $\overline w^\star = \argmin_w \overline F_{\theta}^\nu(w)$.

We also define the notion of \emph{client drift} as 
\begin{align} \label{eq:sfl:pf:client-drift}
    d\pow{t} := \expect_{S \sim U_m} \left[ \sum_{i \in S} \pi_i\pow{t} \sum_{k=0}^{\tau-1} \normsq{w_{i, k}\pow{t} - w\pow{t}} \, \middle\vert \, \Fcal_t \right] \,.
\end{align}

\begin{proof}[Proof of \Cref{thm:a:sfl-cvx}]
We denote $\expect_t[\cdot] := \expect[\,\cdot \mid \Fcal_t]$.
We expand the update $w\pow{t+1} = w\pow{t} - \gamma \sum_{i\in S} \pi_i\pow{t} \sum_{k=0}^{\tau-1} \grad \tilde F_i(w_{i, k}\pow{t})$
to get 
\begin{align} \label{eq:sfl:cvx-pf:main}
\begin{aligned}
    \|w\pow{t+1} - \overline w^\star\|^2
    =& \,  \normsq{w\pow{t} - \overline w^\star}
    - \underbrace{2 \gamma 
            \sum_{i \in S} \pi_i\pow{t} \sum_{k=0}^{\tau-1} \dotproduct{\grad \tilde{F}_i(\localmodel{t}{i}{k})}{w\pow{t} - \overline w^\star}
        }_{=:\Tcal_1}
        \\ & 
    + \underbrace{\gamma^2
            \squarednorm{
                \sum_{i\in S} \pi_i\pow{t} \sum_{k=0}^{\tau-1} \grad \tilde{F}_i(\localmodel{t}{i}{k})
            }
        }_{=: \Tcal_2} \,.
\end{aligned}
\end{align}
In order to bound the first order term $\Tcal_1$, we analyze the effect of a single local step of learning rate $\tau\gamma$ rather than $\tau$ local steps of learning rate $\gamma$. The analogue of the first order term $\Tcal_1$, in this case, would be
\[
    \Tcal_1' := 2 \tau\gamma \, \sum_{i \in S} \pi_i\pow{t} \inp*{\grad \tilde F_i(w\pow{t})}{w\pow{t} - \overline w^\star}
    \stackrel{\eqref{eq:sfl:smooth-grad:sample}}{=} 2 \tau\gamma\, \inp*{\grad F_{\theta, S}^\nu(w\pow{t})}{w\pow{t} - \overline w^\star} \,.
\]
The difference $\Tcal_1 - \Tcal_1'$ is the effect of the drift from taking multiple local steps. From here, the proof consists of the following steps:
\begin{enumerate}[nosep]
    \item \label{item:sfl-cvx:1}
    bound the first order term $\Tcal_1'$,
    \item \label{item:sfl-cvx:2}
    bound the drift $\Tcal_1 - \Tcal_1'$,
    \item \label{item:sfl-cvx:3}
    bound the second order term $\Tcal_2$,
    \item \label{item:sfl-cvx:4}
    combine these to get the effect of one communication round $t$, 
    \item \label{item:sfl-cvx:5}
    unroll the bound over all communication rounds $t=1,\ldots, T$, 
    \item \label{item:sfl-cvx:6}
    connect optimization on the surrogate $\overline F_{\theta}^\nu$ to the original $F_\theta$, 
    \item \label{item:sfl-cvx:7}
    optimize the choices of the learning rate $\gamma$ and smoothing parameter $\nu$.  
\end{enumerate}

\myparagraph{\ref{item:sfl-cvx:1}. Bounding the first order term $\Tcal_1'$}
We use the $\lambda$-strong convexity (cf. \eqref{lem:str_convex}) of $F_{\theta, S}^\nu$ to get
\[
    \Tcal_1' \ge  2\tau\gamma\, \left(F_{\theta,S}^\nu(w\pow{t}) -  F_{\theta,S}^\nu (\overline w^\star) 
            + \frac{\lambda}{2} \normsq{w\pow{t} - \overline w^\star} \right) \,.
\]
Taking an expectation w.r.t. the sampling $S$ (i.e., conditioned on $\Fcal_t$) gives
\begin{align} \label{eq:sfl:cvx-pf:1}
    \expect_t[\Tcal_1'] \ge  2\tau\gamma\, \left(\overline F_{\theta}^\nu(w\pow{t}) - \overline F_{\theta}^\nu (\overline w^\star) 
            + \frac{\lambda}{2} \normsq{w\pow{t} - \overline w^\star} \right)  \,.
\end{align}

\myparagraph{\ref{item:sfl-cvx:2}. Bounding the effect of the drift $\Tcal_1 - \Tcal_1'$}
The contribution of $k$\textsuperscript{th} local step to the drift $\Tcal_1 - \Tcal_1'$ can be bounded as
\begin{align*}
    \Bigg\vert &\inp*{\sum_{i \in S}\pi_i\pow{t} \left(\grad \tilde F_i(w_{i, k}\pow{t}) - \grad \tilde F_i(w\pow{t})\right)}{w\pow{t} - \overline w^\star}\Bigg\vert  \\
    &\stackrel{\text{(i)}}{\le} \frac{\lambda}{4} \normsq*{w\pow{t}- \overline w^\star} + \frac{1}{\lambda} \normsq*{\sum_{i \in S} \pi_i\pow{t}\left(\grad \tilde F_i(w_{i, k}\pow{t}) - \grad \tilde F_i(w\pow{t}) \right) } \\
    &\stackrel{\text{(ii)}}{\le} \frac{\lambda}{4} \normsq*{w\pow{t}- \overline w^\star} + \frac{1}{\lambda} \sum_{i \in S} \pi_i\pow{t} \normsq*{\grad \tilde F_i(w_{i, k}\pow{t}) - \grad \tilde F_i(w\pow{t}) } \\
    &\stackrel{\text{(iii)}}{\le} \frac{\lambda}{4} \normsq*{w\pow{t}- \overline w^\star} + \frac{(L+\lambda)^2}{\lambda} \sum_{i \in S} \pi_i\pow{t} \normsq*{w_{i, k}\pow{t} - w\pow{t} } \,.
\end{align*}
Here, we first used (i) the Cauchy-Schwarz inequality, (ii) Jensen's inequality, and (iii) the $(L+\lambda)$-smoothness of $\tilde F_i$. 
Summing this over $k$, we get the bound
\begin{align} \label{eq:sfl:cvx-pf:2}
    \expect_t\vert \Tcal_1 - \Tcal_1'\vert
    \le \frac{\lambda \tau \gamma}{2}\normsq{w\pow{t}-\overline w^\star} + \frac{2 \gamma(L+\lambda)^2}{\lambda} d\pow{t} \,,
\end{align}
where we use the definition of $d\pow{t}$ from \eqref{eq:sfl:pf:client-drift}.

\myparagraph{\ref{item:sfl-cvx:3}. Bounding the second order term $\Tcal_2$}
By using the expression \eqref{eq:sfl:smooth-grad:sample} of $\grad F_{\theta, S}^\nu$, we get
\begin{align*}
    &\squarednorm{\sum_{i\in S} \pi_i\pow{t} \sum_{k=0}^{\tau-1} \grad \tilde F_i(w_{i, k}\pow{t})} \\
    &\leq 2\squarednorm{
            \sum_{i \in S} \pi_i\pow{t} \sum_{k=0}^{\tau - 1} \left(\grad \tilde{F}_i(\localmodel{t}{i}{k}) - \grad \tilde{F}_i(w\pow{t}) \right)
            }  
        +  2\squarednorm{
            \sum_{i \in S} \pi_i\pow{t} 
            \sum_{k=0}^{\tau - 1}  \grad \tilde{F}_i(w\pow{t})
            }  
       &\\
    & \leq 
         2\tau \sum_{i \in S} \pi_i\pow{t} \sum_{k=0}^{\tau-1} \squarednorm{\grad \tilde{F}_i(\localmodel{t}{i}{k}) - \grad \tilde{F}_i(w\pow{t})} 
         + 2\tau^2 \squarednorm{\grad F_{\theta,S}^\nu(w\pow{t})}  \,. &
        \\
\end{align*}
For the first term, 
we invoke $(L+\lambda)$-smoothness of $\tilde F_i$ and take an expectation to get $2 \tau (L + \lambda)^2 d\pow{t}$.
For the second term, 
we have from the definition \eqref{eq:def_avg_superquantile} of 
$\overline F_\theta^\nu$ that
$\expect_t \left[ \grad F_{\theta, S}^\nu(w\pow{t})\right] = \grad \overline F_\theta^\nu(w\pow{t})$. Therefore, we can write
\begin{align*}
    \expect_t \squarednorm{\grad F_{\theta,S}^\nu(w\pow{t})}
    &= \expect_t \squarednorm{
                    \grad F_{\theta,S}^\nu(w\pow{t})  - \grad \overline F_{\theta}^\nu(w\pow{t})
                    }  
                    + \squarednorm{\grad \overline F_{\theta}^\nu(w\pow{t})}  \\
    &\le \frac{8G^2}{\theta m} + 
    2 L' \left( \overline F_\theta^\nu(w\pow{t}) - \overline F_\theta^\nu(\overline w^\star) \right) \,,
\end{align*}
where we invoked 
\Cref{lem:bounded_variance} to bound the variance of the partial superquantile and $L'$-smoothness of $\overline F_\theta^\nu$.  
Overall, this gives us 
\begin{align} \label{eq:sfl:cvx-pf:3}
    \expect_t [\Tcal_2]
    &\leq 2 \gamma^2 \tau (L+\lambda)^2 \, d\pow{t}
         + \frac{16 \tau^2 \gamma^2  G^2}{\theta m}
         + 4 \tau^2 \gamma^2 L' \left( \overline F_{\theta}^\nu(w\pow{t}) - \overline F_{\theta}^\nu(\overline w^\star) \right) \,.
\end{align} 

\myparagraph{\ref{item:sfl-cvx:4}. One-step update}
Plugging \eqref{eq:sfl:cvx-pf:1} 
to \eqref{eq:sfl:cvx-pf:3} into
\eqref{eq:sfl:cvx-pf:main}, we get, 
\[
\begin{split}
    \expect_t \normsq{\globalmodel{t+1} - \overline w^\star} \leq 
        & \, \left(1 - \frac{\lambda \tau \gamma}{2}\right)  \normsq{w\pow{t} - \overline w^\star} 
        \\&
        - (2 \tau \gamma - 4 \gamma^2 \tau^2 L') \left(\overline F_{\theta}^\nu(w\pow{t}) - \overline F_{\theta}^\nu(\overline w^\star)\right) 
        \\ 
        & + \frac{16 \tau^2 \gamma^2 G^2}{\theta m}
         + 2\gamma (L+\lambda)^2 (\tau\gamma + \lambda^{-1}) d\pow{t} \,.
\end{split}
\]
Next, we plug in the bound on $d\pow{t}$ from \Cref{prop:sfl:client-drift} and simplify some coefficients.
First, since $\gamma \leq (4 \tau L')^{-1}$ we have $2 \tau \gamma - 4 \gamma^2 \tau^2 L' \geq \tau \gamma$.
Likewise, the same condition on $\gamma$ also implies $\tau\gamma + 1/\lambda \le 2/\lambda$. 
Finally, 
$\gamma \le \big( 8\tau \kappa \sqrt{2 \lambda L'}\big)^{-1}$ implies ${64 L' (L+\lambda)^2 \tau^2 \gamma^2}/{ \lambda} \leq 1/2$.
As a result, we arrive at the bound
\begin{align*}
    \overline F_{\theta}^\nu(w\pow{t}) - \overline F_{\theta}^\nu(\overline w^\star)
        \leq \frac{2}{\tau \gamma} 
        &\left(1 - \frac{\lambda\tau \gamma}{2}\right) \normsq{w\pow{t} - \overline w^\star} 
            - \frac{2}{\tau \gamma} \expect_t\normsq{\globalmodel{t+1} - \overline w^\star} \\
        &+ \underbrace{\frac{32\tau                      \gamma G^2 }{\theta m}
            + \frac{64 G^2  (L+\lambda)^2 \tau^2(1-\tau^{-1}) \gamma^2}{ \lambda}  \left(4 + \frac{8}{\theta m} \right)
            }_{=: \totalnoise} \,.
\end{align*}

\myparagraph{\ref{item:sfl-cvx:5}. Telescoping the bound}
By telescoping the one-step improvement and convexity, we get, 
Next, we use convexity to get
\begin{align*}
    &\expectation{\overline F_{\theta}^\nu(\overline w\pow{T}) - \overline F_{\theta}^\nu(\overline w^\star)}{}  \\
    &\le \frac{1}{\sum_{t=0}^T \left(1 - \frac{\lambda\gamma         \tau}{2}\right)^{-(1 + t)}} \sum_{t=0}^T \left(1 - \frac{\lambda\tau \gamma}{2}\right)^{-(1 + t)}  \expectation{\overline F_{\theta}^\nu(w\pow{t}) - \overline F_{\theta}^\nu(\overline w^\star)}{}  \\
    &\leq \frac{2 \squarednorm{w\pow{0} - \overline w^\star}}{\tau \gamma \sum_{t=0}^T \left(1 - \frac{\lambda\tau \gamma}{2}\right)^{-(1 + t)}} 
         + \totalnoise \,.
\end{align*}
Now, we can bound the denominator from below with
\begin{align*}
    \sum_{t=0}^T \left(1 - \frac{\lambda\tau \gamma}{2}\right)^{-(1 + t)} 
    \ge \frac{2}{\lambda \tau \gamma} \left( \left(1 - \frac{\lambda \tau \gamma}{2}\right)^{-(T+1)} - 1\right)
    \geq \frac{2}{\tau \gamma \lambda} \left(e^{(T+1) \lambda \tau \gamma} - 1\right) \,.
\end{align*}
This gives us the final bound
\begin{align} \label{eq:sfl-cvx:proof:1}
    \expectation{\overline F_{\theta}^\nu(\overline w\pow{T}) - \overline F_{\theta}^\nu(\overline w^\star)}{} \leq \frac{\lambda}
    {e^{T \lambda \tau \gamma}  - 1}
    \normsq{w\pow{0} - \overline w^\star}  + \totalnoise  \,.
\end{align}

\myparagraph{\ref{item:sfl-cvx:6}. Translating the results from the surrogate $\overline F_\theta^\nu$ to the original $F_\theta$}
We optimize the surrogate $\overline F_\theta^\nu$ defined on a sample $S$ of clients rather than the full superquantile. 
The effect of this shows up in both sides of \eqref{eq:sfl-cvx:proof:1}. 
We bound the left-hand side by noting that the bias introduced by the surrogate is bounded as in \Cref{lem:sfl:bouded_bias}.
For the right hand side, we use the $\lambda$-strong convexity of $F_\theta$ and \Cref{lem:sfl:bouded_bias} to get
\begin{align*}
     \normsq{w\pow{0} - \overline w^\star} &\leq 2 \normsq{w\pow{0} - w^\star} + 2 \normsq{\overline w^\star - w^\star} 
     \\ & \leq 2 \normsq{w\pow{0} - w^\star} + \frac{4}{\lambda} \left( F_\theta(\overline w^\star) - F_\theta(w^\star)  \right) \\
     & \, 
     \begin{aligned}
     \leq & \,  2 \normsq{w\pow{0} - w^\star} 
        \\ & + \frac{4}{\lambda} \left( F_\theta(\overline w^\star) - \overline F_{\theta}^\nu(\overline w^\star) 
        + \overline F_{\theta}^\nu(\overline w^\star) - \overline F_{\theta}^\nu(w^\star) 
        + \overline F_{\theta}^\nu(w^\star) -  F_\theta(w^\star) \right) 
    \end{aligned} \\
    &\, \leq 2 \normsq{w\pow{0} - w^\star} 
        + \frac{4}{\lambda} \left(\frac{2B}{\sqrt{\theta m}} +  4 \nu \log m \right)\quad \,,
\end{align*}
since $\overline F_{\theta}^\nu(\overline w^\star) - \overline F_{\theta}^\nu(w^\star)  \leq 0$.
Plugging this into \eqref{eq:sfl-cvx:proof:1} 
gives us the bound
\begin{align} \label{eq:sfl-cvx:proof:final-old}
\begin{aligned}
    \expectation{F_\theta(\overline w\pow{T}) - F_\theta(w^\star)}{} \leq 
    & \;\;\frac{2\lambda}
    {e^{T \lambda \tau \gamma}  - 1} \normsq{w\pow{0} - w^\star}  + \frac{32 \tau \gamma G^2}{\theta m} 
            + 
    \\ &
            \frac{64 G^2  (L+\lambda)^2 \tau^2 (1-\tau^{-1}) \gamma^2}{ \lambda} 
     \left(4 + \frac{8}{\theta m} \right) + 
     \\ &
     \left(\frac{2B}{\sqrt{\theta m}} +  4 \nu \log m \right) \left(1 + \frac{8}
    {e^{T \lambda \tau \gamma}  - 1} \right)\, .
\end{aligned}
\end{align}
\myparagraph{\ref{item:sfl-cvx:7}. Hyperparameter optimization}
To complete the proof from here, it remains to optimize the learning rate $\gamma$ and the smoothing parameter $\nu$ by repeated invocations of \Cref{lem:sfl:best-params-1}.

We start with the learning rate $\gamma$. Ignoring absolute constants gives us the bound
\begin{equation}
\label{eq:sfl-cvx:proof:final}
\begin{split}
    \expectation{F_\theta(\overline w\pow{T})}{} - F_\theta(w^\star) \leq 
    & \;\; \lambda \Delta_0 \exp(-\lambda \tau \Gamma T) + 
    \frac{G^2}{\theta m \lambda T} 
    \left(1 \vee \log\frac{\lambda^2 \Delta_0 \theta m }{G^2} T \right)
    \\ & +
        \frac{G^2\kappa^2}{\lambda T^2}(1-\tau^{-1}) \left(
            1 \vee \log \frac{\lambda^2 \Delta_0 T^2}{G^2 \kappa^2}
        \right)^2 
     \\ & + 
     \frac{B}{\sqrt{\theta m}} +   \nu \log m \,,
\end{split}
\end{equation}
where we take 
\[
    \Gamma = \min\left\{ \frac{\sqrt{\lambda}}{8\tau (L+\lambda) \sqrt{2 L'}},  \,
        \frac{1}{4\tau L'}
    \right\}
\]
This application of \Cref{lem:sfl:best-params-1} requires $\lambda \tau \Gamma T \ge 1$, which we will ensure later, based on the choice of $\nu$. Recall that $\Gamma$ depends on $L'$, which itself depends on $\nu$ as $L' = L + \lambda + G^2/\nu$. 

Next, we set $\nu$. We will require that $\nu \le G^2 / (\lambda \kappa)$, so that the two terms from \eqref{eq:sfl-cvx:proof:final} that depend on $\nu$ can be bounded as 
\begin{align} \label{eq:sfl-cvx:proof:final-nu}
    \lambda \Delta_0 \exp(-\lambda \tau \Gamma T) + \nu \log m 
    \le \lambda \Delta_0\exp\left( - \frac{T}{
    16\kappa\sqrt{\frac{G^2}{\lambda \nu}} \vee \frac{8G^2}{\lambda \nu}} \right)
    + \nu \log m \,.
\end{align}
To simplify the expression, we substitute 
\[
    \frac{1}{\delta} = \max\left\{ 16 \kappa\sqrt{\frac{G^2}{\lambda \nu}},\, 
    \frac{8 G^2}{\lambda \nu} 
    \right\}
    \quad
    \iff
    \quad
    \nu = \max\left\{
    \frac{256\kappa^2 G^2 \delta^2}{\lambda} , \, \frac{8G^2\delta}{\lambda}
    \right\} \,.
\]
The bound $\nu \le G^2/ (\lambda \kappa)$ translates to the upper bound $\delta \le (16 \kappa^{3/2})^{-1}$. 
Therefore, the right hand side of \eqref{eq:sfl-cvx:proof:final-nu} can 
be further upper bounded by using $\max\{a, b\} \le a + b$ as
\[
    \lambda \Delta_0 \exp(-\delta T) 
    + \frac{8G^2 \log m}{\lambda} \,\, \delta 
    + \frac{256 G^2 \kappa^2 G^2 \log m}{\lambda} \,\, \delta^2 \,.
\]
We now invoke
\Cref{lem:sfl:best-params-1} under the condition $T \ge 16\kappa^{3/2}$. We set $\delta$ as specified by \Cref{lem:sfl:best-params-1} --- this gives us the choices of the smoothing parameter $\nu$ and learning rate $\gamma$. Plugging this into \eqref{eq:sfl-cvx:proof:final} gives the bound of the theorem. Finally, to complete the proof, it can be verified that the condition $\lambda \tau \Gamma T \ge 1$ is guaranteed by $T \ge 16 \kappa^{3/2}$. 
\end{proof}

\subsection{Intermediate Results}
We present some prerequisites and some intermediate results which are required in the convergence proofs. 

Note that for any $S \subset [n]$ of size $m$, the partial superquantile is differentiable at $w$ with :
\begin{align} \label{lem:stochastic_gradient_formula}
 \grad F_{\theta,S}^\nu(w) = \sum_{i \in S} \pi_i^\star \grad \tilde{F}_i(w)
\end{align} 
where $\pi^\star$ denotes the solution to the maximization 
\begin{equation*}
    F_{\theta,S}^\nu(\x)
        = \max_{\pi \in \mathcal{P}_{\theta, S}} \sum_{i \in S} \pi_i \tilde{F}_i(\x) - \nu D_{S}(\pi)
\end{equation*}

\myparagraph{Bias and variance of the partial superquantile}
We use the partial superquantile defined on a subset $S \subset [n]$ to approximate the full superquantile. We start with the quality of this approximation.

\begin{property}\label{lem:sfl:bouded_bias} \label{lem:bounded_variance}
Let $U_m$ denote the uniform distribution over all subsets of $[n]$ of size $m$.
For any $w \in \reals^d$, we have
\[
\begin{split}
 \left\vert \overline F_{\theta}^\nu(w) - F_\theta(w) \right\vert  \leq \frac{B}{\sqrt{\theta m}} + 2 \nu \log m \,, \\
 \expect_{S \sim U_m} \squarednorm{\grad F_{\theta,S}^\nu(w) - \grad \overline F_{\theta}^\nu(w)} \leq \frac{8G^2}{\theta m} \,.
\end{split}
\]
\end{property}

\myparagraph{Smoothing and smoothness constants}
The following result is standard~\cite[][Theorem 4.1, Lemma 4.2]{beck2012smoothing}.

\begin{property}\label{lem:smoothness_partial_superquantile}
For every $\nu > 0$, we have that $F_{\theta, S}^\nu$ and $\overline F_{\theta, S}^\nu$ 
are $L'$-smooth with $L' = L + \lambda + \frac{G^2}{\nu}$.
\end{property}

\myparagraph{Bounding the gradient dissimilarity} 
Bounding of the variance of gradient estimators is a key assumption in the analysis of stochastic gradient methods (see e.g. the textbook~\cite{Bottou2018optimization}). In the centralized setting, when a stochastic objective $\expect_\xi[f(w, \xi)]$, it is standard to assume for a given estimator ${g_w}$ of $\grad_w \expect{f(w,\xi)}$ that there exists some constants $M_1, M_2 >0$ such that for all $w \in \Rd$,
\[
    \squarednorm{\expectation{{g_w}}{}} \leq M_1 \quad \text{or} \quad \squarednorm{\expectation{{g_w}}{}} \leq M_1 + M_2 \squarednorm{\grad_w \expectation{f(w,\xi}{}} \,.
\]
In the federated setting, the use of a subset $S \subset [n]$ of clients in each round induces noise on the estimation of the average gradient over the whole network. Thus, such assumption translates into a \textit{bound on the gradient dissimilarity} among the clients~\cite{karimireddy2020scaffold,wang2019adaptive}: 
\[
    \frac{1}{n}  \sum_{i \in [n]} \squarednorm{\grad \tilde{F}_i(w)} \leq M_1 + M_2 \squarednorm{\frac{1}{n} \sum_{i \in [n]} \grad \tilde{F}_i(w)} \,.
\]
In this work, we also consider the minimization of the global loss $F_\theta^\nu$ by a stochastic algorithm based on partial participation of the clients, with the additional difficulty that we only have access to a biased estimator $\overline F_{\theta}^\nu$ of the loss $F_\theta^\nu$ and its gradient. In particular, the adaptive reweighting of the clients selected at each round does not permit the direct use of such an assumption. We show instead in the next lemma that the variance of the stochastic gradient estimator can also be bounded, thanks to the Lipschitz assumption.

\begin{proposition}[Gradient Dissimilarity]\label{lem:adversarial_dissimilarity}
Consider the quantities $\pi\pow{t}, w\pow{t}$ from \Cref{algo:sfl-new}.
We have,
\[
    \mathbb{E} \left[ \sum_{i \in S} \pi_i\pow{t} \squarednorm{\grad \tilde{F}_i(w\pow{t})} \, \middle\vert \Fcal_t \right]
    \leq \left(4 +\frac{8}{\theta m}\right) G^2 + \squarednorm{\grad \overline F_{\theta}^\nu(w\pow{t})} \,.
\]
\end{proposition}
\begin{proof}
We drop the superscript $t$ throughout this proof. 
By centering the second moment (cf. \eqref{eq:techn:mean-var}), we have:
\[
\begin{split}
    \sum_{i \in S} \pi_i & \squarednorm{\grad \tilde{F}_j(w)} = \sum_{i \in S} \pi_i \squarednorm{\grad \tilde{F}_i(w) - \grad F_{\theta,S}^\nu(w)} + \squarednorm{\grad F_{\theta,S}^\nu(w)}\\
    &= \sum_{i \in S} \pi_i\squarednorm{\grad F_i(w) - \sum_{j \in S} \pi_j \grad F_j(w)} + \squarednorm{\grad F_{\theta,S}^\nu(w)}\,.
\end{split}
\]
Now since the weights $\pi_i$ sum to one, we may use the convexity of $\squarednorm{\cdot}$ to get:
\[
\begin{split}
     \sum_{i \in S} \pi_i \squarednorm{\grad \tilde{F}_j(w)} 
    &\leq \sum_{i, j \in S} 
    \pi_i \pi_j
        \squarednorm{\grad F_j(w) - \grad F_i(w)} 
    + \squarednorm{ \grad F_{\theta,S}^\nu(w)} \, . \\
\end{split}
\]
The squared triangle inequality (cf. \eqref{eq:ineq:young})
together with the Lipschitz assumption on the functions $F_i$ yields:
\[
\begin{split}
     \sum_{i \in S} \pi_i \squarednorm{\grad \tilde{F}_i(w)} 
    &\leq  \,  2 \sum_{i, j \in S} 
   \pi_i \pi_i
        \left(\squarednorm{\grad F_i(w)} + \squarednorm{\grad F_j(w)} \right) +   \squarednorm{\grad F_{\theta,S}^\nu(w)} 
    \\&\leq 4\; G^2 + \squarednorm{\grad F_{\theta,S}^\nu(w)}\,.
\end{split}
\]

Thus, taking an expectation over $S \sim U_m$ gives
\[
\begin{split}
     \mathbb{E} \left[\sum_{i \in S} \pi_i\squarednorm{\grad \tilde{F}_j(w)} \, \middle\vert\, \Fcal_t \right] \leq 4\; G^2 + \mathbb{E}_{S \sim U_m} \left[ \squarednorm{\grad F_{\theta,S}^\nu(w)}\right] \,.
\end{split}
\]

By centering (cf. \eqref{eq:techn:mean-var}), we get,
\begin{equation}\label{eq:adversarial_final_step}
\begin{split}
    \mathbb{E} \left[ \sum_{i \in S} \pi_i \squarednorm{\grad \tilde{F}_i(w)} \, \middle\vert\, \Fcal_t \right] 
    \leq\,  & 4\; G^2 
    + \squarednorm{\grad \overline F_{\theta}^\nu(w)} 
    \\ & 
    + \mathbb{E} \left[ \squarednorm{\grad F_{\theta,S}^\nu(w) - \grad \overline F_{\theta}^\nu(w)} \, \middle\vert\, \Fcal_t \right]
    \,.
\end{split}
\end{equation}
Finally, substituting the variance bound from \Cref{lem:bounded_variance} into~\eqref{eq:adversarial_final_step} yields the stated result.
\end{proof}

\myparagraph{Bounding the Client Drift}
During federated learning, each client takes multiple local steps.
This causes the resulting update to be a biased estimator of a descent direction for the global objective. This phenomenon has been referred to as ``client drift''~\cite{li2020convergence_fedavg,karimireddy2020scaffold}. Current proof techniques rely on treating this as a ``noise'' term that is to be controlled. 
In the context of this work, the reweighting by $\pi\pow{t}$ requires us to adapt this typical definition of client drift to our setting. 
In particular, recall that we define the client drift $d\pow{t}$ in outer iteration $t$ of the algorithm as 
\begin{align*} 
    d\pow{t} := \expect_{S \sim U_m} \left[ \sum_{i \in S} \pi_i\pow{t} \sum_{k=0}^{\tau-1} \normsq{w_{i, k}\pow{t} - w\pow{t}} \, \middle\vert \, \Fcal_t \right] \,.
\end{align*}

\begin{proposition}[Client Drift] \label{prop:sfl:client-drift}
If $\gamma \leq \frac{1}{4 \tau (L+\lambda)}$, we have for any $t\geq 0$ that
\begin{align*}
    d\pow{t} &\leq 8\tau^2(\tau-1) \gamma^2 \left( 
        \left(4 + \frac{8}{\theta m} \right)  G^2 
        + 2 L' 
            \left(\overline F_{\theta}^\nu(w\pow{t}) 
            - \overline F_{\theta}^\nu(\overline w^\star) \right)
    \right) \,.
\end{align*}
Furthermore, if $\lambda = 0$, we have the bound
\[
    d\pow{t} \le 8 \tau^2(\tau-1) \gamma^2  G^2 \,.
\]
The last bound also works without smoothing, i.e., $\nu = 0$.
\end{proposition}
\begin{proof}
We absorb the regularization into the superquantile by defining $\tilde F_i(w) = F_i(w) + (\lambda/2)\normsq{w}$. 
If $\tau = 1$, there is nothing to prove as both sides of the inequality are $0$. We assume now that $\tau > 1$. 
Let us first fix $S \subset [n]$ of size $\vert S\vert  = m$. For any $k \in S$ and $j \in \range{\tau-1}$, by the squared triangle inequality (cf. \eqref{eq:ineq:young}), we have:
\[
\begin{split}
    \squarednorm{\localmodel{t}{i}{k} - w\pow{t}} &= \squarednorm{\localmodel{t}{i}{k-1} - \gamma \grad \tilde{F}_i(\localmodel{t}{i}{k-1}) -  w\pow{t}} \\
    & \leq \left(1 + \frac{1}{\tau-1}\right) \squarednorm{\localmodel{t}{i}{k-1} - w\pow{t}} + \tau \gamma^2 \squarednorm{\grad \tilde{F}_i(\localmodel{t}{i}{k-1})}\,.
\end{split}
\]
The squared triangle inequality (cf. \eqref{eq:ineq:young}) 
together with the smoothness of the local losses gives:
\[
\begin{split}
    \begin{aligned}
    &\squarednorm{\localmodel{t}{i}{k} - w\pow{t}} \\
    & \textcolor{white}{\left(\squarednorm{\localmodel{t}{i}{k} - w\pow{t}}\right)} %
    \end{aligned}
    & 
    \begin{aligned}
    \leq & \left(1 + \frac{1}{\tau-1}\right) \squarednorm{\localmodel{t}{i}{k-1} - w\pow{t}} 
       \\ & + 2 \tau \gamma^2 
            \left( \squarednorm{\grad \tilde{F}_i(\localmodel{t}{i}{k-1}) - \grad \tilde{F}_i(w\pow{t})}
            + \squarednorm{\grad \tilde{F}_i(w\pow{t})}
            \right)
    \end{aligned}\\
    & \begin{aligned}
    \leq & \left(1 + \frac{1}{\tau-1}\right) \squarednorm{\localmodel{t}{i}{k-1} - w\pow{t}} 
       \\ & + 2 \tau \gamma^2 (L+\lambda)^2
             \squarednorm{\localmodel{t}{i}{k-1} - w\pow{t}}
            + 2 \tau \gamma^2 \squarednorm{\grad \tilde{F}_i(w\pow{t})}
            \,.
    \end{aligned}
\end{split}
\]
Hence, for $\gamma \leq \frac{1}{4 \tau (L+\lambda)}$, we get:
\[
\begin{split}
    \squarednorm{\localmodel{t}{i}{k} - w\pow{t}} 
    &\leq \left(1 + \frac{2}{\tau-1}  \right) \squarednorm{\localmodel{t}{i}{k-1} - w\pow{t}} + 2 \tau \gamma^2 \squarednorm{\grad \tilde{F}_i(w\pow{t})} \,.
\end{split}
\] 
Unrolling this recursion yields for any $j \le \tau - 1$
\[
\begin{split}
    \squarednorm{\localmodel{t}{i}{k} - w\pow{t}} 
    &\leq \sum_{i=0}^{k-1} \left(1 + \frac{2}{\tau-1}  \right)^{i} \left( 2 \tau \gamma^2 \squarednorm{\grad \tilde{F}_i(w\pow{t})} \right)  \\
    &\leq \frac{\tau-1}{2}  \left(1 + \frac{2}{\tau -1}\right)^k  \left( 2 \tau \gamma^2 \squarednorm{\grad \tilde{F}_i(w\pow{t})} \right) \\
    &\leq \frac{\tau-1}{2} \left(1 + \frac{2}{\tau -1}\right)^{\tau-1} \left( 2 \tau \gamma^2 \squarednorm{\grad \tilde{F}_i(w\pow{t})} \right) \\
    & \leq 8\tau(\tau-1) \gamma^2  \squarednorm{\grad\tilde{F}_i(w\pow{t})} \,,
\end{split}
\] 
where we use $(1 + 2/x)^x \le e^2 < 8$ for any $x > 0$.
If $\lambda = 0$ we have that $\squarednorm{\grad\tilde{F}_i(w\pow{t})}
= \squarednorm{\grad F_i(w\pow{t})} \le G^2$ since $F_i$ is $G$-Lipschitz; this gives us the final bound in the statement. When $\lambda \neq 0$, this does not hold.
In this case, we apply \Cref{lem:adversarial_dissimilarity} to get

\begin{align*}
    d\pow{t} 
    &\leq  \meanoverselectedsets{8\tau^2(\tau-1) \gamma^2\, } \left[ \sum_{i \in S} \pi_i\pow{t} \squarednorm{\grad \tilde{F}_i(w\pow{t})} \,\middle\vert\, \Fcal\pow{t} \right] \\
    &\leq 8\tau^2(\tau-1) \gamma^2 \left( 
        \left(4 + \frac{8}{\theta m} \right)  G^2 
        + \squarednorm{\grad \overline F_{\theta}^\nu(w\pow{t})} 
        \right) \,.
\end{align*}
Invoking smoothness (cf.~\eqref{lem:smoothness}) completes the proof.
\end{proof}

\subsection{Useful Inequalities and Technical Results}

We recall a few standard inequalities:
\begin{itemize}
    \item Squared Triangle inequality: For any
        $x, y\in \reals^d$ and $\alpha>0$ we have:
        \begin{align} \label{eq:ineq:young}
            \squarednorm{x + y} \leq (1 + \alpha) \squarednorm{x} + \left(1 + \frac{1}{\alpha}\right) \squarednorm{y} \,.
        \end{align}
    \item Centering the second moment: For any $\reals^d$-valued random vector $X$ such that $\expect \normsq{X} < \infty$, 
    \begin{align} \label{eq:techn:mean-var}
    \expect\normsq{\rv}
        = \expect{
        \squarednorm{\rv - \expectation{\rv}{} }
        }{} + \squarednorm{\expectation{\rv}{}}
    \end{align}
    \item Strong convexity: Let $F: \Rd \rightarrow \R$ be $\mu$-strongly convex. Then for any $x, y \in \Rd$, we have:
    \begin{align} \label{lem:str_convex}
    \dotproduct{\grad F(x)}{x-y} \geq F(x) - F(y) + \frac{\mu}{2} \squarednorm{x - y} 
    \end{align}
    \item Smoothness:
        Let $F: \Rd \rightarrow \R$ be $L$-smooth and let $F^\star$ be the minimum value of $F$ (assuming it exists). Then for any $x \in \Rd$, we have:
        \begin{align} \label{lem:smoothness}
        \squarednorm{\grad F(x)} \leq 2L \left(F(x) - F^\star\right) 
        \end{align}
\end{itemize}

\begin{lemma} \label{lem:sfl:best-params-1}
    Consider the maps $\varphi, \psi: (0, \Gamma] \to \reals_+$ given by
    \[
        \varphi(\gamma) = 
        \frac{A}{\exp(\lambda \gamma T) - 1} + B \gamma + C \gamma^2 \,,
        \quad
        \psi(\gamma) = 2A \exp(-\lambda \gamma T) + B\gamma + C\gamma^2 \,,
    \]
    where $\lambda, \Gamma, A, B, C, T > 0$ are given and $\lambda \Gamma \le 1$. 
    If $T \ge (\lambda \Gamma)^{-1}$, then, we have, 
    \[
        \varphi(\gamma^\star) \le \psi(\gamma^\star) \le 
        2A \exp(-\lambda \Gamma T) 
        + \frac{3B}{\lambda T} 
        \left(1 \vee \log \frac{A \lambda T}{B}\right) 
         + \frac{3C}{\lambda^2 T^2} \left(1 \vee\log\frac{A \lambda^2 T^2}{C} \right)^2  \,,
    \]
    where $\gamma^\star$ is given by
    \[
        \gamma^\star = \min\left\{ 
        \Gamma, \frac{1}{\lambda T} \left( 1 \vee \log \frac{A \lambda T}{B}\right) ,
        \frac{1}{\lambda T} \left(1 \vee\log \frac{A \lambda^2 T^2}{C} \right) 
        \right\} \,.
    \]
Furthermore, we also have that $\big(\exp(\lambda \gamma^\star T) - 1\big)^{-1} \le 1$. 
\end{lemma}
\begin{proof}
Since $\lambda \Gamma T \ge 1$, we have that $\lambda \gamma^\star T \ge 1$. Then, 
$\exp(-\lambda \gamma^\star T) \le \exp(-1) < 1/2$ so that
\[
    \frac{1}{\exp(\lambda \gamma^\star T) - 1}
    = 
    \frac{\exp(-\lambda \gamma^\star T)}{1 - \exp(-\lambda \gamma^\star T)}
    \le 2 \exp(-\lambda \gamma^\star T) \le 1 \,.
\]
Therefore, we have, 
\[
    \varphi(\gamma^\star) \le 2 A \exp(-\lambda \gamma^\star T) + B \gamma^\star + C (\gamma^\star)^2 = \psi(\gamma^\star)\,.
\]

Next, define $\gamma_1 = (\lambda T)^{-1}\log(1 \vee A\lambda T/B)$ and $\gamma_2 = (\lambda T)^{-1} \log(1 \vee A\lambda^2 T^2/C)$, so that $\gamma^\star = \min\{\Gamma, \gamma_1, \gamma_2\}$.
We have three cases:
\begin{itemize}
    \item If $\gamma^\star = \Gamma$, we have that $\Gamma \le \gamma_1$ and $\Gamma \le \gamma_2$ so that
    \[
        \psi(\gamma^\star) = 
        2A \exp(-\lambda \Gamma  T) + B \Gamma + C\Gamma^2 
        \le 2A \exp(-\lambda \Gamma  T) + B \gamma_1 + C\gamma_2^2 \,.
    \]
    \item If $\gamma^\star = \gamma_1$, we have $\gamma_1 \le \gamma_2$. In this case, 
    \[
        \psi(\gamma^\star)
        = A \exp(-\lambda \gamma_1 T) + B\gamma_1 + C \gamma_1^2
        \le \frac{2B}{\lambda T} + \frac{B}{\lambda T} \left(1 \vee \log \frac{A\lambda T}{B} \right) + C\gamma_2^2 \,.
    \]  
    \item If $\gamma^\star = \gamma_2$, we have $\gamma_2 \le \gamma_1$, so that 
    \[
        \psi(\gamma^\star)
        = 2A \exp(-\lambda \gamma_2 T) + B\gamma_2 + C \gamma_2^2
        \le \frac{2C}{\lambda^2 T^2} + 
        B\gamma_1 + \frac{C}{\lambda^2 T^2} \left(1 \vee \log \frac{A\lambda^2 T^2}{C} \right)^2 \,.
    \]  
\end{itemize}
\end{proof}

The proof of the next lemma is elementary and is omitted. 
\begin{lemma} \label{lem:sfl:best-params-2}
    Consider the map $\varphi: (0, \Gamma] \to \reals_+$ given by
    \[
        \varphi(\gamma) = 
        \frac{A}{\gamma T}  + B \gamma + C \gamma^2 \,,
    \]
    where $\Gamma,  A, B, C > 0$ are given. 
    Then, we have, 
    \[
        \varphi(\gamma^\star) \le 
        \frac{A}{\Gamma T} + 2\left(\frac{AB}{T}\right)^{1/2}
             + 2 C^{1/3} \left( \frac{A}{T} \right)^{2/3} \,,
    \]
    where $\gamma^\star$ is given by
    \[
        \gamma^\star = \min\left\{ 
        \Gamma, \sqrt{\frac{A}{BT}}, \left(\frac{A}{CT}\right)^{1/3} \right\} \,.
    \]
\end{lemma}

\section{Privacy Analysis} \label{sec:a:sfl:privacy}

\subsection{Preliminaries}
The discrete Gaussian mechanism was introduced in \cite{canonne2020discrete} as an extension of the Gaussian mechanism to integer data. 
A random variable $\xi$ is said to satisfy the discrete Gaussian distribution with mean $\mu$ and variance proxy $\sigma^2$ if 
\[
    \prob(\xi = n) = C \, \exp\left( - \frac{(n-\mu)^2}{2\sigma^2} \right) \quad \text{for all } n \in \ZZ \,,
\]
where $C$ is an appropriate normalizing constant.
We denote it by $\Ncal_{\ZZ}(\mu, \sigma^2)$. 
We need the following property of the discrete Gaussian.
\begin{property} \label{property:techn:discrete-gaussian}
    Let $\xi$ be distributed according to $\Ncal_{\ZZ}(\mu, \sigma^2)$. Then, $\expect[\xi] = \mu$. 
    Furthermore, if $\mu = 0$, then $\xi$ is sub-Gaussian with variance proxy $\sigma^2$, i.e., $\expect[\exp(\lambda \xi)] \le \exp(\lambda^2\sigma^2/2)$ for all $\lambda > 0$. 
\end{property}

\subsection{Privacy-Utility Analysis of Quantile Computation}
We now give the full proof of \Cref{thm:sfl:quantile:new}. 

\begin{proof}[Proof of \Cref{thm:sfl:quantile:new}]
We start by defining and controlling the probabilities of some events. Throughout, let $\delta > 0$ be fixed.
    Define the event 
    \begin{align}
        E_{\mathrm{mod}} = \bigcap_{i=1}^n \bigcap_{r=0}^{\log_2 b - 1} \bigcap_{j=1}^{b/2^r} \left\{ 
            - \frac{M-2}{2n} \le c  x_i(r, j) + \xi_i(r, j) \le \frac{M-2}{2n} 
        \right\} \,.
    \end{align}
    Note that under $E_{\mathrm{mod}}$, no modular wraparound occurs in the algorithm. Thus, for all valid levels $r$ and indices $j$, 
    we have $\tilde x_i(r, j) = c  x_i(r, j) + \xi_i(r, j)$ and
    \[
        \hat h(r, j) = \sum_{i=1}^n \frac{\tilde x_i(r, j)}{c }
            = \sum_{i=1}^n \left(x_i(r, j) + \frac{\xi_i(r, j)}{c } \right) \,.
    \] 
    
    Next, we define the event 
    \begin{align}
         E_{\mathrm{diff}} = 
         \bigcap_{j=1}^b \left\{ \left\vert H(j) - \hat H(j) \right\vert \le \sqrt{2 \sigma^2 n \log_2(b) \log(4b/\delta)} \right\} \,.
    \end{align}
    We will show later that $E_{\mathrm{mod}}$ and $E_{\mathrm{diff}}$ holds with high probability;
    for now, we assume that they hold.
    
    \myparagraph{Privacy Analysis}
    We start by establishing the sensitivity of the sum query over $x_i$'s as $\log_2 b$, one for each level in the hierarchical histogram. Define the input space $\Xcal$ to be the space of hierarchical histograms with one non-zero entry in the leaf nodes with consistent counts (i.e., the count of a parent node in the hierarchical histogram equals the sum of its child nodes). 
    Let $\Xcal^* = \cup_{m=1}^\infty \Xcal^m$ denote the set of all sequences of elements of $\Xcal$. 
    We consider the rescaled sum query $A((x_1, \ldots, x_N)) = \sum_{i=1}^n c  x_i$. 
    The $L_2$ sensitivity $S(A)$ of this query $A$ is supremum over all $X \in \Xcal^*$ and $X'$ which is obtained by concatenating $x'$ to $X$: 
    \[
        S(A) = \sup_{X, X'} \norm{A(X) - A(X')}_2
        = \sup_{x' \in \Xcal} c  \norm{x'}_2  = c \log_2 b  \,.
    \]
    We invoke the privacy bound of sums of discrete Gaussians (\Cref{lem:techn:sum-of-disc-gauss}) to claim that 
    an algorithm $\Acal$ returning
    $A(x) + \sum_{i=1}^n \xi_i$ satisfies $(1/2)\epsilon^2$-concentrated DP  where $\epsilon$ is as in the theorem statement. 
    The fact that the quantile and all further functions of it remain private follows from the post-processing property of DP (also known as the data-processing inequality). 
    
    \myparagraph{Utility Analysis}
    Using the triangle inequality, we get, 
    \begin{align*}
        \Delta_\theta(\hat H, H) 
        &= \left\vert \frac{H\big(j^*_\theta(\hat H)\big)}{n} - (1-\theta) \right\vert \\
        &\le \frac{1}{n}\left\vert H\big(j^*_\theta(\hat H)\big) -\hat H\big(j^*_\theta(\hat H)\big) \right\vert
        + \left\vert \frac{1}{n} \hat H\big(j^*_\theta(\hat H)\big)
            - (1-\theta) \right\vert \\
        &\le 
        \max_{j \in [b]}  \left\{ \frac{1}{n}\left\vert H(j) - \hat H(j) \right\vert \right\}
        + R_\theta^*(\hat H) \,.
    \end{align*}
    The first term is bounded under $E_{\mathrm{diff}}$, and this gives the utility bound. 

    \myparagraph{Bounding the Failure Probability}
    The algorithm fails when at least one of 
    $E_{\mathrm{mod}}$ or $E_{\mathrm{diff}}$ fail to hold.
    We have from~\Cref{claim:sfl:quantile:pf:1}
    that $\prob(E_{\mathrm{mod}}) \ge 1- \delta / 4$ under the 
    given assumptions. 
    From, \Cref{claim:sfl:quantile:pf:2}, we have
    $\prob(E_{\mathrm{diff}} \vert E_{\mathrm{mod}}) \ge 1- \delta / 2$.
    We bound the total failure probability of the algorithm with a union bound as 
    \begin{align*}
        \prob(\bar E_{\mathrm{diff}} \cup \bar E_{\mathrm{mod}})
        &\le \prob(\bar E_{\mathrm{diff}} \vert E_{\mathrm{mod}}) \, \prob(E_{\mathrm{mod}}) + \prob(\bar E_{\mathrm{diff}} \vert \bar E_{\mathrm{mod}})\, \prob(\bar E_{\mathrm{mod}})  + \prob(\bar E_{\mathrm{mod}})  \\
        &\le  \prob(\bar E_{\mathrm{diff}} \vert E_{\mathrm{mod}}) 
        + 2 \, \prob(\bar E_{\mathrm{mod}}) \le \delta \,.
    \end{align*}
\end{proof}

We state and prove bounds on probabilities of the events
$E_{\mathrm{mod}}, E_{\mathrm{diff}}$ defined above.
\begin{claim} \label{claim:sfl:quantile:pf:1}
    If $M \ge 2 + 2c n + 2n \sqrt{2\sigma^2 \log(16nb/\delta)}$, then 
    $\prob(E_{\mathrm{mod}}) \ge 1 - \delta / 4$.
\end{claim}
\begin{proof}
    Each discrete Gaussian random variable $\xi_{i}(r, j)$ is centered and sub-Gaussian with variance proxy $\sigma^2$ (cf. \Cref{property:techn:discrete-gaussian}). 
    A Cram\'er-Chernoff bound (cf. \Cref{lem:techn:chernoff}) gives us the exponential tail bound
    \[
        \prob\left(\vert\xi_i(r, j)\vert > \sqrt{2\sigma^2 \log(16nb/\delta)} \right) \le \frac{\delta}{8nb} \,.
    \]
    Applying the union bound over $i = 1, \ldots, n$ and the $2b - 2$ nodes in each hierarchical histogram $x_i$ (each node corresponding to one $(r, j)$ pair) completes the proof. 
\end{proof}

\begin{claim} \label{claim:sfl:quantile:pf:2}
    We have $\prob(E_{\mathrm{diff}} \vert E_{\mathrm{mod}}) \ge 1-\delta/2$. 
\end{claim}
\begin{proof}
    Under $E_{\mathrm{mod}}$, we have that 
    $\hat H(j) = H(j) + \sum_{i=1}^n \sum_{(r, o) \in P_j} \xi_i(r, o)$, where $P_j$ is the maximal dyadic partition of $[1, j]$ with $\vert P_j\vert \le  \log_2 b$. 
    Thus, $\zeta_j := \hat H(j) - H(j)$ is sub-Gaussian with variance proxy $n \vert P_j\vert \sigma^2 \le n \sigma^2 \log_2 b$. 
    A Cram\'er-Chernoff bound (cf. \Cref{lem:techn:chernoff}) gives us
    \[
        \prob\left( \vert \zeta_j \vert  > \sqrt{2\sigma^2 n \log_2(b) \log(4b/\delta)} \right) \le \frac{\delta}{2b} \,.
    \]
    Applying a union bound over $j = 1, \ldots, b$ completes the proof. 
\end{proof}

\subsection{Useful Results}

The distributed discrete Gaussian mechanism gets privacy guarantees by adding a sum of discrete Gaussian random variables. We give a bound on its privacy. The following lemma is due to~\cite{kairouz2021distributed}.
\begin{lemma}[Privacy of Sum of Discrete Gaussians]
\label{lem:techn:sum-of-disc-gauss}
    Fix $\sigma \ge 1/2$. Let $A : \Xcal \to \real^d$ be a deterministic algorithm with $\ell_2$-sensitivity $S$ for some input space $\Xcal$. Define a randomized algorithm $\Acal$, which when given an input $x \in \Xcal$, samples $\xi_1, \ldots, \xi_n \sim \Ncal_{\ZZ}(0, \sigma^2 I_d)$ and returns $A(x) +  \sum_{i=1}^n \xi_i$.
    Then, $\Acal$ satisfies $\epsilon^2/2$-concentrated DP with 
    \[
        \epsilon = \min\left\{ 
            \sqrt{\frac{S^2}{n\sigma^2} + \frac{\psi d}{2}}, 
            \frac{S}{\sqrt{n}\sigma} + \psi \sqrt{d}
        \right\} \,,
    \]
    where $\psi = 10 \sum_{i=1}^{n-1} \exp\big(-2\pi^2\sigma^2 i / (k+1)\big) \le 10 (n-1) \exp(-2\pi^2\sigma^2)$.
\end{lemma}

Next, we record a standard concentration result. 
\begin{lemma}[Cram\'er-Chernoff] \label{lem:techn:chernoff}
    Let $\xi$ be a real-valued and centered sub-Gaussian random variable with variance proxy $\sigma^2$, i.e.,  $\expect[\xi] = 0$ and $\expect[\exp(\lambda \xi)] \le \exp(\lambda^2 \sigma^2 / 2)$ for all $\lambda > 0$.
    Then, we have for any $t > 0$,
    \[
        \prob(\vert\xi\vert > t) \le 2 \exp\left( - \frac{t^2}{2\sigma^2}\right) \,.
    \]
\end{lemma}

\section{Numerical Experiments: Extended Results} \label{sec:a:sfl:expt}

We conduct our experiments on two datasets from computer vision and natural language processing.
These datasets contain a natural, non-iid split of data which is reflective of data heterogeneity encountered in federated learning.
In this section, we describe in detail the experimental setup and the results. Here is its outline:
\begin{itemize}
    \item \Cref{sec:a:sfl:expt:datasets} describes the datasets and tasks. 
    \item \Cref{sec:a:sfl:expt:hyperparameters} presents the algorithm and the hyperparameters used. 
    \item \Cref{sec:a:sfl:evaluation_strategy} details the evaluation methodology.
    \item \Cref{sec:a:sfl:exp_results} gives the experimental comparison of \newfl to baselines.
\end{itemize}

Since each client has a finite number of datapoints
in the examples below, we let its probability distribution $q_i$ to be the empirical distribution 
over the available examples, and the weight
$\alpha_i$ to be proportional to the number of 
datapoints available on the client. 

\subsection{Datasets and Tasks}\label{sec:a:sfl:expt:datasets}
We use the two following datasets, described in detail below. 
The data was preprocessed using
LEAF~\cite{caldas2018leaf}.
\par\leavevmode\par

\myparagraph{EMNIST for handwritten-letter recognition}\newline
\vspace{-2ex}
\myparagraph{Dataset}
EMNIST~\cite{cohen2017emnist} is a character recognition dataset. This dataset contains images of handwritten digits or letters, labeled with their identification (a-z, A-Z, 0-9).
The images are grey-scaled pictures of $28 \times 28 = 784$ pixels. 

\myparagraph{Train and Test Devices}
Each image is also annotated with the ``writer'' of the image, i.e.,
the human subject who hand-wrote the digit/letter during the data collection process. 
Each client corresponds to one writer. 
From this set of clients, we discard all clients containing less than 100 images. The remaining clients 
were partitioned into two groups --- $1730$ training and $1730$ testing clients.
For each experiment, we subsampled $865$ training and $865$ testing 
clients for computational tractability, where the sampled clients
vary based on the random seed of each experiment.

\myparagraph{Model}
We consider the following models for this task.

\begin{itemize}
    \item \textbf{Linear Model}:
    We use a linear softmax regression model.
    In this case, each $F_i$ is convex.
    We train parameters $w \in \mathbb{R}^{62 \times 784}$. Given an input image $x\in \mathbb{R}^{784}$, the score of each class $c\in [62]$ is the dot product $\langle w_c,  x \rangle$. The probability $p_c$ assigned to each class is then computed as a softmax: $p_c = \exp{\langle w_c,  x \rangle}/ \sum_{c'} \exp{\langle w_{c'},  x \rangle}$. The prediction for a given image is then the class with the highest probability.
    
    \item \textbf{ConvNet}: We also consider a convolutional neural network with 
    two convolutional layers with max-pooling and one fully connected layer (F.C) which outputs a vector in
    $\reals^{62}$. The outputs of the ConvNet are scores with respect to each class. They are also used with a softmax operation to compute probabilities.
\end{itemize}

The loss used to train both models is the multinomial logistic loss $L(p,y) = - \log p_y$ where $p$ denotes the vector of probabilities computed by the model and $p_y$ denotes its $y$\textsuperscript{th} component. 
In the convex case, we add a quadratic regularization term of the form $({\lambda}/{2}) \|w\|_2^2$. 
\vspace{1ex}
\myparagraph{Sent140 for Sentiment Analysis}\newline
\vspace{-2ex}
\myparagraph{Dataset}
Sent140~\citep{go2009twitter} is a text dataset of 1,600,498 tweets produced by 660,120 Twitter accounts. Each tweet is represented by a character string with emojis redacted. Each tweet is labeled with a binary sentiment reaction (i.e., positive or negative), which is inferred based on the emojis in the original tweet. 

\myparagraph{Train and Test Devices}
Each client represents a Twitter account and contains only tweets published by this account. From this set of clients, we discarded all clients containing less than 50 tweets and split the 877 remaining clients into a train set and a test set of sizes $438$ and $439$ respectively.
This split was held fixed for all experiments.
Each word in the tweet is encoded by 
its $50$-dimensional GloVe embedding~\citep{pennington2014glove}.

\myparagraph{Model}
We consider the following models.

\begin{itemize}
\item \textbf{Linear Model}: We consider a 
    $l_2$-regularized linear logistic 
    regression model where the parameter vector $w$  is of dimension $50$. In this case, each $F_i$ is convex. We summarize each tweet by the average of the GloVe embeddings of the words of the tweet. 
    
\item \textbf{RNN}: The nonconvex model is a Long Short Term Memory (LSTM) model~\citep{hochreiter1997long} 
    built on the GloVe embeddings of the words of the tweet. 
    The hidden dimension of the LSTM is the same as the embedding dimension, i.e., $50$. We refer to it as ``RNN''.
\end{itemize}

The loss function is the binary logistic loss.

\subsection{Algorithms and Hyperparameters}\label{sec:a:sfl:expt:hyperparameters}

\myparagraph{Algorithm and Baselines}\newline
The proposed \newfl %
is run for three values of $\theta\in \{0.8, 0.5, 0.1\}$. We compare it to the following baselines:
\begin{itemize}
    \item \fedavg~\citep{mcmahan2017communication}: It is the de facto standard for the vanilla federated learning objective. 
    
    \item \fedavg, $\theta$: We also consider \fedavg{} with a random client subselection step: local updates are run on a fraction of the initial number of clients randomly selected  per round. For each dataset, we try three values, corresponding to the average number of clients selected by \newfl for the three values of $\theta$ used. 
    In the main paper, we report as \fedavgsub{} the performance of
    \fedavg, $\theta$ with $\theta \in \{0.8,0.5,0.1\}$ 
    which gives the best performance on \newfl (i.e.,
    lowest $90$\textsuperscript{th} percentile of test 
    misclassification error). Here we report numbers for all 
    values of $\theta$ considered.
    
    \item \fedprox~\citep{li2020fedprox}: It augments \fedavg{} with 
    a proximal term but still minimizes the vanilla federated learning
    objective.
    
    \item \qffl~\citep{li2020fair}: 
    It raises the per-client losses to the power $(1+q)$, where
    $q \ge 0$ is a parameter, in order to focus on clients with higher loss. We run $q$-FFL for values of $q$ in $\{10^{j}, j\in \{-3, \ldots, 1\}\}$.
    \item \afl~\citep{mohri2019agnostic}: It aims to minimize the worst per-client loss. We implement it as an asymptotic version of $q$-FFL, using a large value of $q$, as this was found to yield better convergence with comparable performance~\citep{li2020fair}. In the experiments, we take $q=10.0$.
\end{itemize}

The experiments are conducted on the datasets described in \Cref{sec:a:sfl:expt:datasets}.

\myparagraph{Hyperparameters}\newline
\vspace{-2ex}    
\myparagraph{Rounds}    
We measure the progress of each algorithm by the 
number of calls to secure aggregation routine for weight vectors,
i.e., the 
number of communication rounds.

    For the experiments, we choose the number of 
    communication rounds depending on the convergence of the optimization for \fedavg. %
    For the EMNIST dataset, we run the algorithm for $3000$ communication rounds with the linear model and $1000$ for the ConvNet. For the Sent140 dataset, we run the $1000$ communication rounds for the linear model and $600$ for the RNN. 
 
\myparagraph{Devices per Round}   
    We choose the same number of clients per round for each 
    method, with the exception of $\fedavg, \theta$.
    All clients are assumed to be available and selections are made uniformly at random. 
    In particular, we select $100$ clients per round
    for all experiments with the exception of Sent140 RNN
    for which we used $50$ clients per round.

\myparagraph{Local Updates and Minibatch Size}
    Each selected client locally runs $1$ epoch of mini-batch stochastic gradient descent locally.
    We used the default mini-batch of $10$ for all experiments~\cite{mcmahan2017communication}, except for $16$ for EMNIST ConvNet. This is because the latter experiments were run using on a GPU, as we describe in the section on the hardware.
    
\myparagraph{Learning rate scheme} 
    We now describe the learning rate $\gamma_t$ used during \textit{LocalUpdate}.
    For the linear model, we used a constant fixed learning rate $\gamma_t \equiv \gamma_0$, while 
    for the neural network models, 
    we used a step decay scheme of the learning rate $\gamma_t = \gamma_0 c^{-\lfloor t/t_0 \rfloor}$ for some $t_0$ where $\gamma_0$ and $0 < c \le 1$ are tuned. 
    We tuned 
    the learning rates only for the baseline \fedavg{}
    and used the same learning rate for the other baselines and \newfl at all values of $\theta$. 
    
    For the neural network models,
    we fixed $t_0$ so that the learning rate decayed once or twice during the fixed time horizon $T$. 
    In particular, we used $t_0 = 400$ for EMNIST ConvNet (where $T=1000$) and
    $t_0 = 200$ for Sent140 RNN (where $T=600$).
    We tuned $c$ from the set $\{2^{-3}, 2^{-2}, 2^{-1}, 1\}$,
    while the choice of the range of $\gamma_0$ 
    depended on the dataset-model pair. 
    The tuning criterion we used was the mean of
    the loss distribution over the training clients
    (with client $i$ weighted by $\alpha_i$)
    at the end of the time horizon.
    That is, we chose the $\gamma_0, c$ 
    which gave the best terminal training loss. 

\myparagraph{Tuning of the regularization parameter}
    The regularization parameter $\lambda$ for linear models was tuned with cross validation from the set $\{10^{-k}\,:\, k \in \{3,\ldots, 8\} \}$. This was performed as described below.
    
    For each dataset, we held out half the training clients as validation clients. Then, for different values of the regularization parameter, 
    we trained a model with the (smaller subset of) training clients and evaluate its performance on the validation clients. 
    We selected the value of the regularization parameter
    as the one which gave the smallest $90$\textsuperscript{th} percentile of the misclassification error on the validation clients. 

\myparagraph{Baselines Parameters}
    We tune the proximal parameter of \fedprox{} with cross validation. The procedure we followed is identical 
    to the procedure we described above for the regularization 
    parameter $\lambda$. The set of parameters tested is $\{10^{-j}, j \in \{0,\dots, 3\}\}$. 
    We did not attempt to tune the parameter $q$ of \qffl{} and 
    report the performance of all values of $q$ which we tried.

\myparagraph{Hyperparameters of \newfl}
    We optimize \newfl via 
    \Cref{algo:sfl-new}
    with a fixed number of local steps, corresponding to one epoch.
    For simplicity, we calculate the quantile exactly, assuming 
    client losses are available to the server.

\begin{table*}[t!]
\caption{\small{
Metrics for the test misclassification error for EMNIST (Linear Model).
 }}\label{table:a:sfl:ext_comp_emnist_linear}
\begin{center}
\begin{adjustbox}{max width=0.99\linewidth}
\begin{tabular}{lcccccccccc}
\toprule
Method  &  Mean  &  Standard Deviation  &  $10^{th}$ Percentile  &  Median  &  $90^{th}$ Percentile \\ 
\midrule 
 FedAvg  & $34.38\pm0.38$  &  $18.39\pm0.33$  &  $21.54\pm0.35$  &  $32.61\pm0.39$  &  $49.65\pm0.67$ \\ 
 FedAvg $ \theta=0.8$  & $34.20\pm0.45$  &  $18.25\pm0.22$  &  $\mathbf{21.37}\pm0.26$  &  $32.10\pm0.34$  &  $49.92\pm1.16$ \\ 
 FedAvg $\theta=0.5$  & $34.51\pm0.47$  &  $18.21\pm0.30$  &  $21.40\pm0.36$  &  $32.36\pm0.59$  &  $50.28\pm0.77$ \\ 
 FedAvg $\theta=0.1$  & $34.60\pm0.46$  &  $18.58\pm0.31$  &  $21.71\pm0.37$  &  $32.54\pm0.37$  &  $50.33\pm1.28$ \\ 
 FedProx  & $\mathbf{33.82} \pm0.30$  & $18.25\pm0.23$  &  $21.37\pm0.35$  &  $\mathbf{31.75}\pm0.20$  &  $49.15\pm0.74$ \\ 
 $q$-FFL (Best $q=1.0$)  & $34.71\pm0.27$  & $19.34\pm0.30$  &  $22.33\pm0.41$  &  $32.80\pm0.23$  &  $49.90\pm0.58$ \\ 
 Tilted-ERM (Best $t=1.0$)  & $34.15 \pm 0.25$ & $10.78 \pm 0.30$ & $22.43 \pm 0.29$ & $32.36 \pm 0.23$ & $48.59 \pm 0.62$ \\
 AFL  & $39.32\pm0.27$  & $25.42\pm0.27$  &  $28.64\pm0.43$  &  $38.16\pm0.34$  &  $51.62\pm0.28$ \\ 
 \midrule
 $\Delta$-FL $ \theta =0.8$  & $34.48\pm0.26$  & $19.16\pm0.32$  &  $22.24\pm0.32$  &  $32.85\pm0.31$  &  $49.10\pm0.24$ \\ 
 $\Delta$-FL $ \theta =0.5$  & $35.01\pm0.20$  & $20.46\pm0.34$  &  $23.64\pm0.22$  &  $33.83\pm0.34$  &  $\mathbf{48.44}\pm0.38$ \\ 
 $\Delta$-FL $ \theta =0.1$  & $38.32\pm0.48$  & $23.86\pm0.59$  &  $27.27\pm0.64$  &  $37.52\pm0.67$  &  $50.34\pm0.95$ \\ 
 \bottomrule
\end{tabular}
\end{adjustbox}
\end{center}
\end{table*}

\begin{table*}[th!]
\caption{\small{
Metrics for the test misclassification error for EMNIST  (ConvNet Model).
 }}\label{table:a:sfl:ext_comp_emnist_convnet}
\begin{center}
\begin{adjustbox}{max width=0.99\linewidth}
\begin{tabular}{lccccc}
\toprule
Method  &  Mean  &  Standard Deviation  &  $10^{th}$ Percentile  &  Median  &  $90^{th}$ Percentile \\ 
\midrule
 FedAvg  & $16.63\pm0.50$  & $\mathbf{4.94}\pm0.14$  &  $\mathbf{6.43}\pm0.24$  &  $15.34\pm0.37$  &  $28.46\pm1.07$ \\ 
 FedAvg $\theta=0.8$  & $15.95\pm0.42$  & $5.25\pm0.19$  &  $6.86\pm0.38$  &  $14.84\pm0.24$  &  $26.82\pm1.28$ \\ 
 FedAvg $\theta=0.5$  & $16.22\pm0.23$  & $5.06\pm0.17$  &  $6.47\pm0.28$  &  $15.05\pm0.25$  &  $27.56\pm0.81$ \\ 
 FedAvg $ \theta=0.1$  & $15.97\pm0.43$  & $5.40\pm0.42$  &  $7.10\pm0.64$  &  $14.76\pm0.20$  &  $26.35\pm2.08$ \\ 
 FedProx  & $16.01\pm0.54$  & $5.16\pm0.32$  &  $6.68\pm0.44$  &  $14.88\pm0.29$  &  $27.01\pm1.86$ \\ 
 $q$-FFL (Best $q=0.001) $  & $16.58\pm0.30$  & $5.05\pm0.21$  &  $6.53\pm0.20$  &  $15.40\pm0.43$  &  $28.02\pm0.80$ \\ 
 Tilted-ERM (Best $t=1.0$)  & $15.69 \pm 0.38$ & $7.31 \pm 0.68$ & $7.26 \pm 0.51$ & $\textbf{14.66} \pm 0.16$ & $25.46 \pm 1.49$ \\
AFL  & $33.00\pm0.37$  & $20.38\pm0.23$  &  $22.92\pm0.23$  &  $31.58\pm0.27$  &  $45.07\pm1.00$ \\ 
\midrule
 $\Delta$-FL $ \theta =0.8$  & $16.08\pm0.40$  & $5.60\pm0.14$  &  $7.31\pm0.29$  &  $14.85\pm0.48$  &  $26.23\pm1.15$ \\ 
 $\Delta$-FL $ \theta =0.5$  & $\mathbf{15.48}\pm0.30$  & $6.13\pm0.15$  &  $8.08\pm0.16$  &  $14.73\pm0.22$  &  $\mathbf{23.69\pm0.94}$ \\ 
 $\Delta$-FL $ \theta =0.1$  & $16.37\pm1.03$  & $6.61\pm0.42$  &  $8.28\pm0.65$  &  $15.49\pm1.03$  &  $25.45\pm2.77$ \\ 
 \bottomrule
\end{tabular}
\end{adjustbox}
\end{center}
\end{table*} 
\begin{table*}[th!]
\caption{\small{
Metrics for the test misclassification error for Sent140 (Linear Model).
 }}\label{table:a:sfl:ext_comp_sent140_linear}
\begin{center}
\begin{adjustbox}{max width=0.99\linewidth}
\begin{tabular}{lccccc}
\toprule
Method  &  Mean  &  Standard Deviation  &  $10^{th}$ Percentile  &  Median  &  $90^{th}$ Percentile \\ 
\midrule 
 FedAvg  & $34.74\pm0.31$  &  $12.16\pm0.15$  &  $21.89\pm0.24$  &  $34.81\pm0.38$  &  $46.83\pm0.54$ \\ 
 FedAvg $\theta=0.8$  & $34.47\pm0.03$  &  $12.08\pm0.16$  &  $21.69\pm0.26$  &  $34.62\pm0.17$  &  $46.59\pm0.38$ \\ 
FedAvg $\theta=0.5$  & $34.46\pm0.07$  &  $12.11\pm0.24$  &  $\mathbf{21.55}\pm0.51$  &  $\mathbf{34.48}\pm0.20$  &  $47.00\pm0.40$ \\ 
FedAvg $\theta=0.1$  & $34.79\pm0.32$  &  $11.97\pm0.37$  &  $22.08\pm0.75$  &  $34.93\pm0.35$  &  $46.69\pm0.84$ \\ 
FedProx  & $34.74\pm0.31$  &  $12.16\pm0.15$  &  $21.89\pm0.24$  &  $34.82\pm0.39$  &  $46.83\pm0.54$ \\ 
 $q$-FFL (Best $q=1.0$)  & $34.48\pm0.06$  &  $11.96\pm0.14$  &  $21.61\pm0.24$  &  $34.57\pm0.16$  &  $\mathbf{46.38}\pm0.40$ \\ 
 Tilted-ERM (Best $t=1.0$)  & $34.71 \pm 0.31$ & $12.00 \pm 0.14$ & $21.83 \pm 0.34$ & $34.91 \pm 0.39$ & $46.70 \pm 0.50$ \\
AFL  & $35.97\pm0.08$  &  $11.83\pm0.09$  &  $23.58\pm0.28$  &  $36.09\pm0.17$  &  $47.51\pm0.32$ \\ 
\midrule
 $\Delta$-FL $\theta =0.8$  & $\mathbf{34.41}\pm0.22$  &  $12.17\pm0.11$  &  $21.77\pm0.34$  &  $34.64\pm0.25$  &  $46.44\pm0.38$ \\ 
$\Delta$-FL $\theta =0.5$  & $35.28\pm0.25$  &  $\mathbf{11.68}\pm0.40$  &  $23.03\pm0.38$  &  $35.55\pm0.53$  &  $46.64\pm0.41$ \\ 
$\Delta$-FL $\theta =0.1$  & $37.78\pm0.89$  &  $12.86\pm0.52$  &  $23.93\pm0.99$  &  $37.80\pm1.30$  &  $51.38\pm1.07$ \\ 
\bottomrule
\end{tabular}
\end{adjustbox}
\end{center}
\end{table*} 

\begin{table*}[th!]
\caption{\small{
Metrics for the test misclassification error for Sent140  (RNN Model).
 }}\label{table:a:sfl:ext_comp_sent140_rnn}
\begin{center}
\begin{adjustbox}{max width=0.99\linewidth}
\begin{tabular}{lccccc}
\toprule
Method  &  Mean  &  Standard Deviation  &  $10^{th}$ Percentile  &  Median  &  $90^{th}$ Percentile \\ 
\midrule 
 FedAvg  & $30.16\pm0.44$  & $4.36\pm1.26$  &  $10.06\pm2.06$  &  $29.51\pm0.33$  &  $49.66\pm3.95$ 1\\ 
 FedAvg $\theta=0.8$  & $\mathbf{29.85}\pm0.46$  & $5.39\pm1.32$  &  $11.90\pm2.27$  &  $29.57\pm0.31$  &  $46.93\pm3.84$ 1\\ 
 FedAvg $\theta=0.5$  & $31.06\pm1.01$  & $\mathbf{4.33}\pm2.73$  &  $\mathbf{9.69}\pm4.89$  &  $30.14\pm0.71$  &  $53.10\pm7.22$ 1\\ 
 FedAvg $\theta=0.1$  & $31.96\pm1.47$  & $4.82\pm2.09$  &  $11.65\pm4.83$  &  $31.55\pm1.13$  &  $52.87\pm8.41$ 1\\ 
 FedProx  & $30.20\pm0.48$  & $4.35\pm1.23$  &  $10.37\pm2.08$  &  $\mathbf{29.51}\pm0.32$  &  $49.85\pm4.07$ \\ 
 $q$-FFL (Best $q=0.01$)  & $29.99\pm0.56$  & $4.90\pm1.66$  &  $10.98\pm2.88$  &  $29.56\pm0.39$  &  $48.65\pm4.68$ \\ 
 Tilted-ERM (Best $t=1.0$)  & $30.13 \pm 0.49$ & $14.17 \pm 2.10$ & $13.18 \pm 3.33$ & $29.96 \pm 0.84$ & $46.54 \pm 3.27$ \\
 AFL  & $37.74\pm0.65$  & $9.90\pm1.46$  &  $18.19\pm1.99$  &  $36.95\pm1.03$  &  $57.78\pm1.19$ \\ 
 \midrule
 $\Delta$-FL $\theta =0.8$  & $30.30\pm0.33$  & $6.75\pm2.68$  &  $13.05\pm3.87$  &  $29.92\pm0.38$  &  $\mathbf{46.46}\pm4.39$ \\ 
 $\Delta$-FL $\theta =0.5$  & $33.58\pm2.44$  & $8.74\pm3.98$  &  $16.77\pm6.62$  &  $33.28\pm2.27$  &  $50.47\pm8.24$ \\ 
 $\Delta$-FL $\theta =0.1$  & $51.97\pm11.81$  & $9.11\pm5.47$ &  $16.67\pm9.15$  &  $52.44\pm13.21$  &  $86.44\pm10.95$ \\ 
 \bottomrule
\end{tabular}
\end{adjustbox}
\end{center}
\end{table*}

\begin{figure}[t!]
   \includegraphics[width=0.97\linewidth]{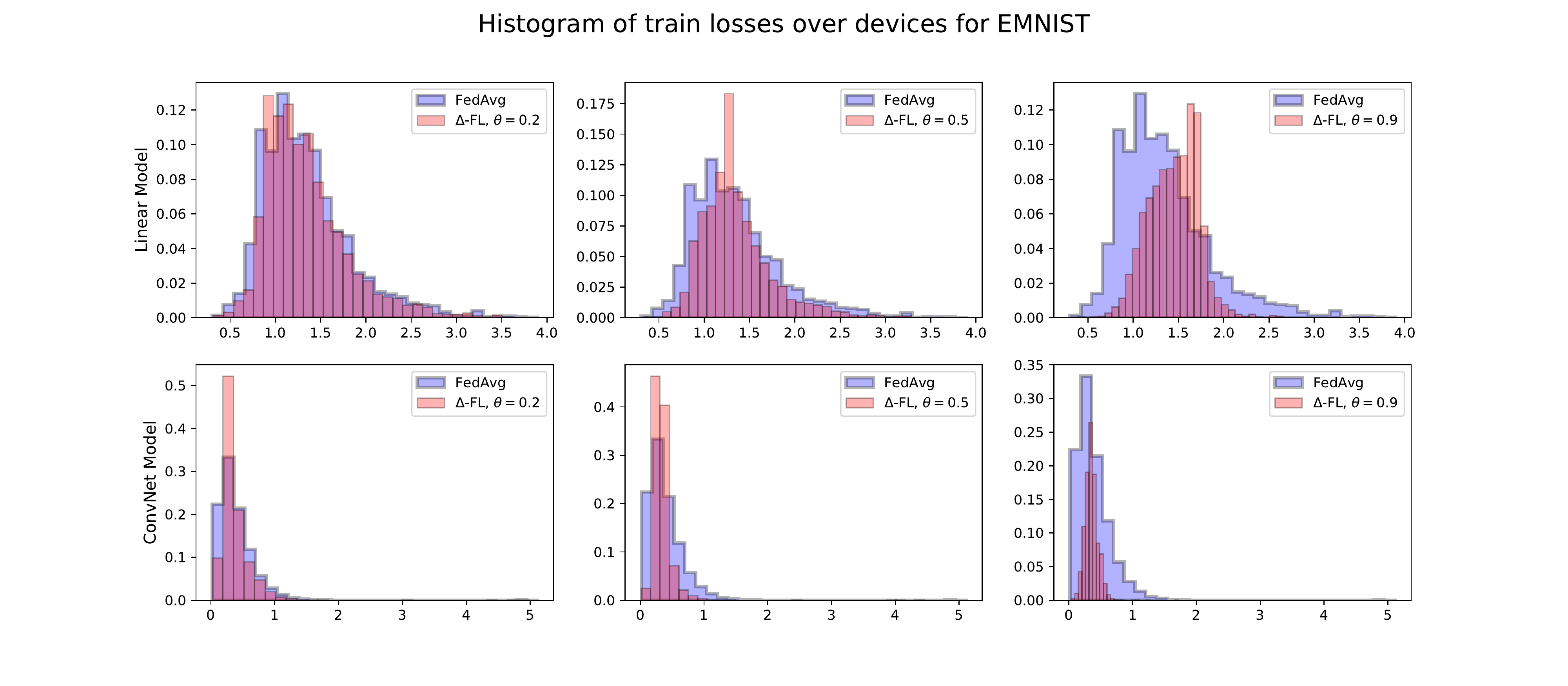}
   \includegraphics[width=0.97\linewidth]{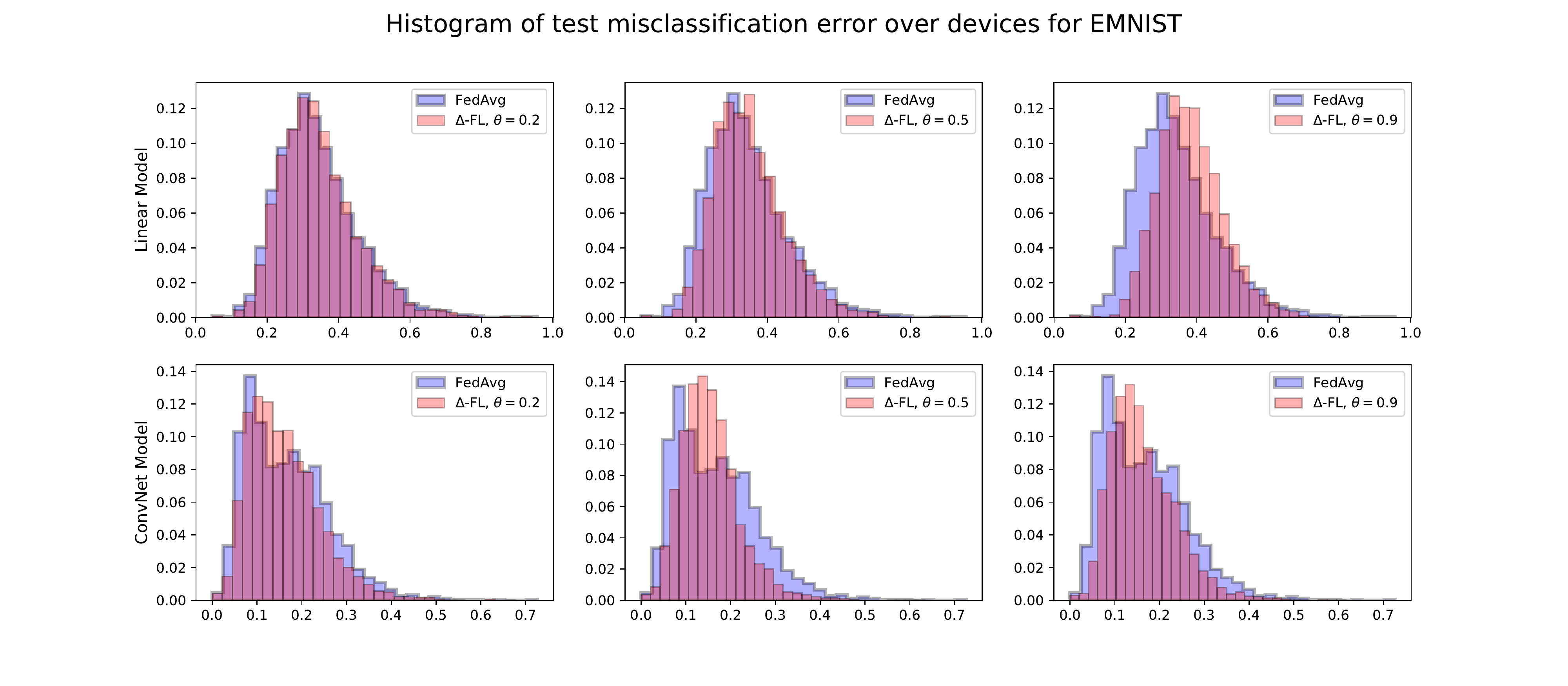}
   \caption{Histogram of loss distribution over training clients and misclassification error distribution over testing clients for EMNIST. The identification of the model (linear or ConvNet) is given 
   on the $y$-axis of the histograms.}
   \label{fig:a:sfl:expt:hist-emnist}
\end{figure}

\begin{figure}[t!]
    \centering
   \includegraphics[width=0.97\linewidth]{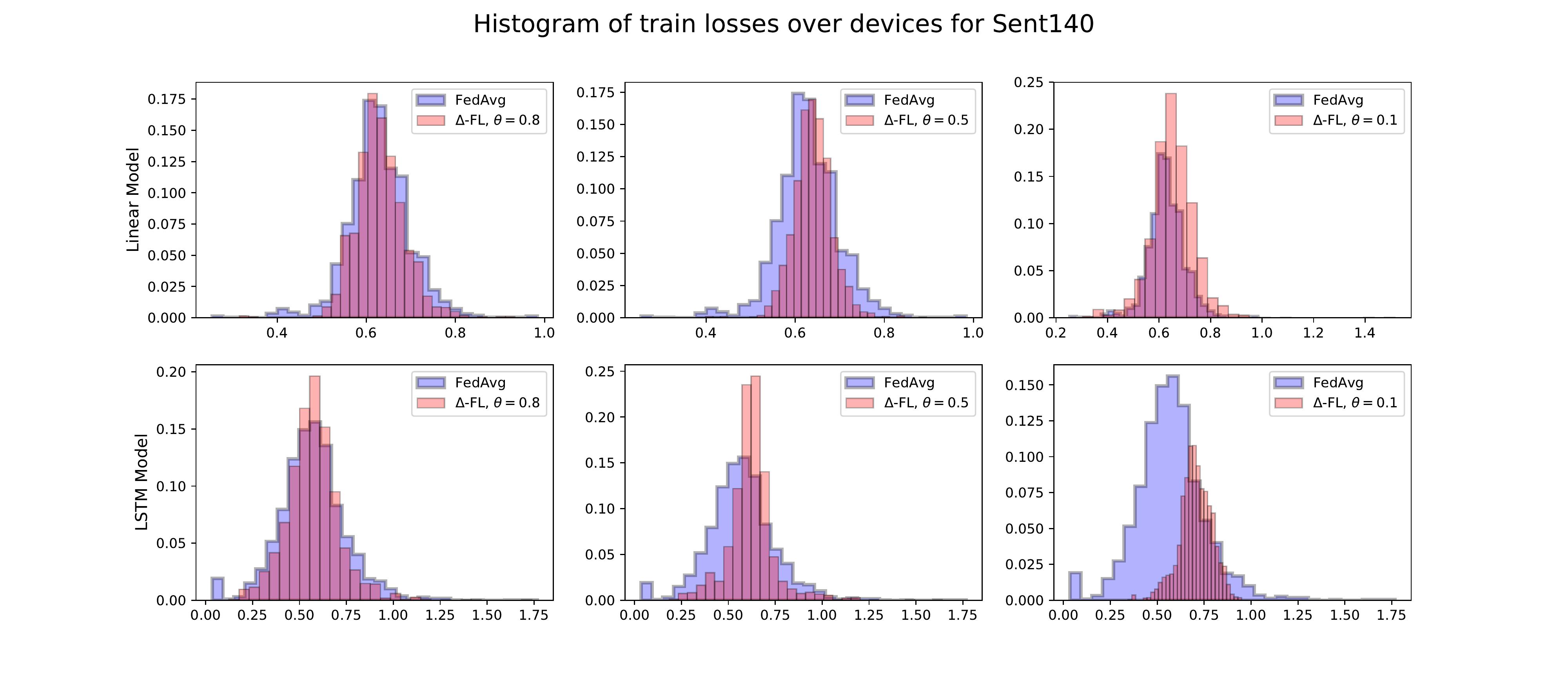}

   \includegraphics[width=0.97\linewidth]{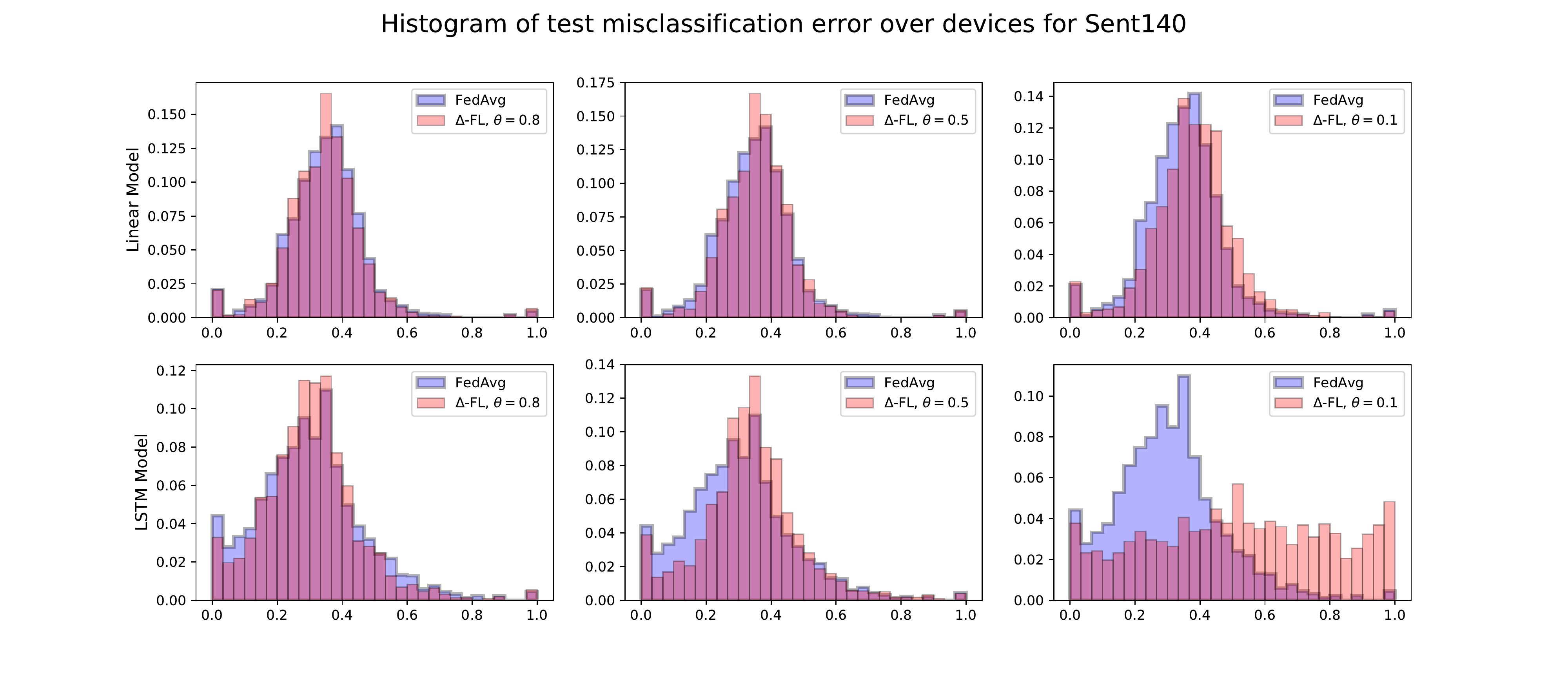}

   \caption{Histogram of loss distribution over training clients and misclassification error distribution over testing clients for Sent140. The identification of the model (linear or RNN) is given on the $y$-axis of the histograms.}
   \label{fig:a:sfl:expt:hist-sent-shake}
\end{figure}

\begin{figure}[t!]
\centering
   \includegraphics[width=0.9\linewidth]{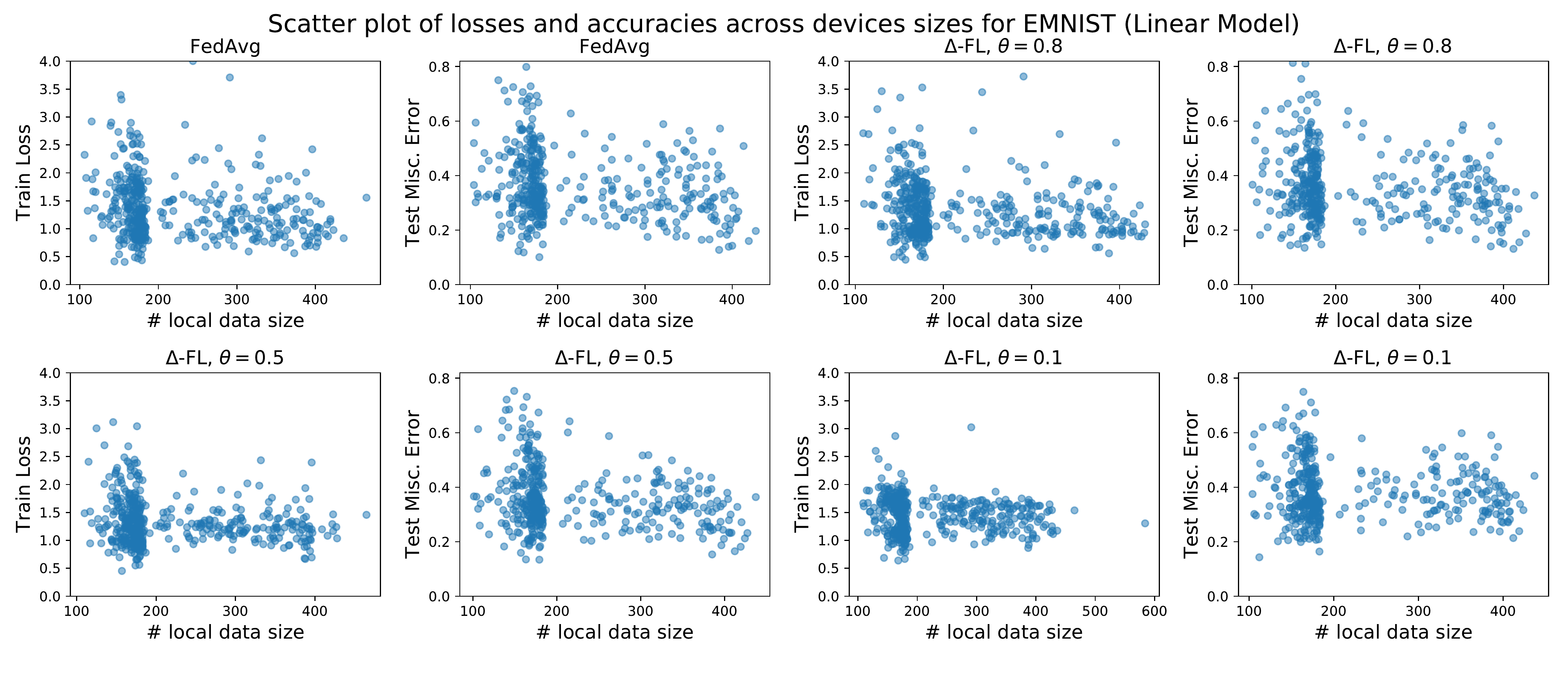}
   
   \includegraphics[width=0.9\linewidth]{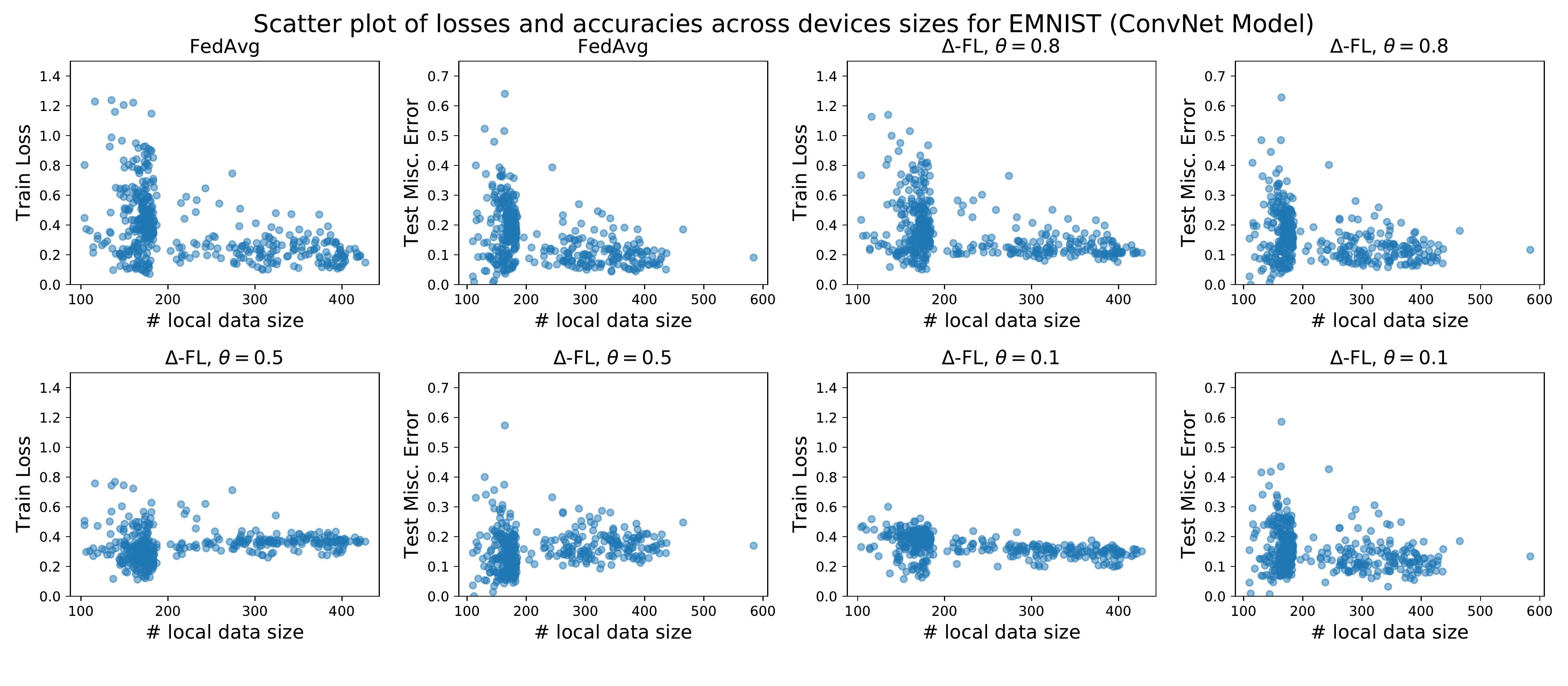}
   \caption{Scatter plot of (a) loss on training client vs. amount of local data, and (b) misclassification error on testing client vs. amount of local data for EMNIST.}
   \label{fig:a:sfl:expt:scatter-emnist}
\end{figure}

\begin{figure}[t!]

   \includegraphics[width=0.96\linewidth]{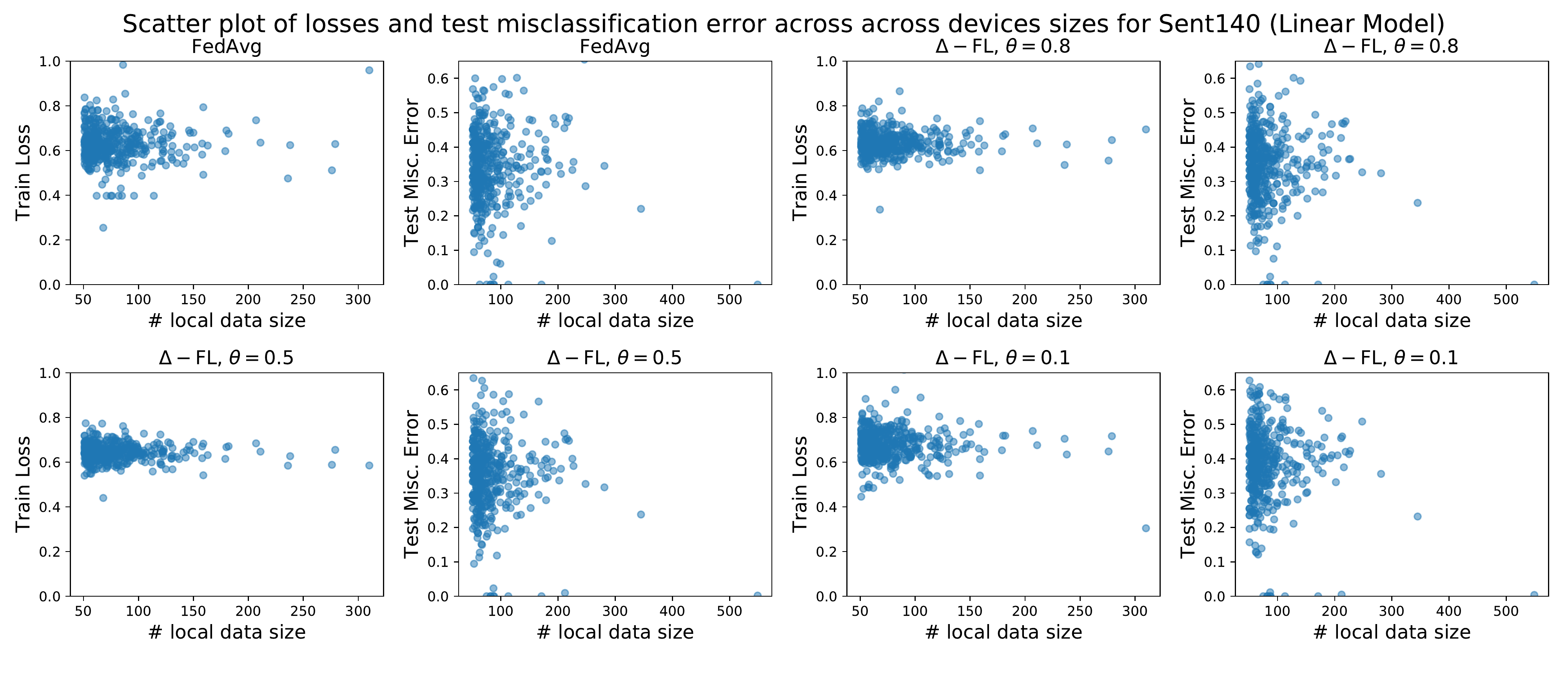}
   
   \includegraphics[width=0.96\linewidth]{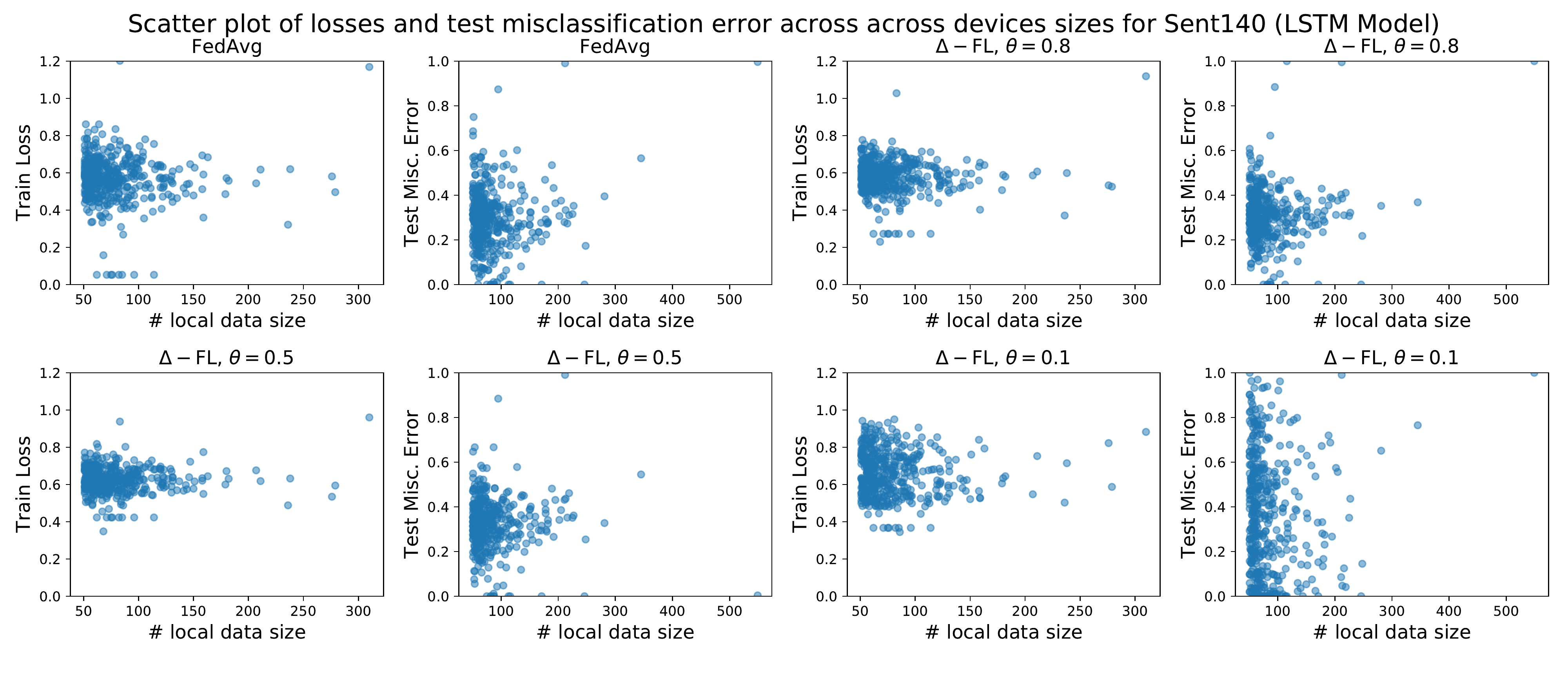}

   \caption{Scatter plot of (a) loss on training client vs. amount of local data, and (b) misclassification error on testing client vs. amount of local data for Sent140.}
   \label{fig:a:sfl:expt:scatter-sent-shake}
\end{figure}

\subsection{Evaluation Strategy and Other Details}\label{sec:a:sfl:evaluation_strategy}

\myparagraph{Evaluation metrics}
    We record the loss of each training client and the misclassification error of each testing client, as measured on its local data. 
    
    The evaluation metrics noted in \Cref{sec:a:sfl:exp_results} are the following: the weighted mean of the loss distribution over the training clients, the (unweighted) mean misclassification error over the testing clients, the weighted $\tau$-percentile of the loss over the training client and the (unweighted) $\tau$-percentile of the misclassification error over the testing clients for values of $\tau$ among $\{20, 50, 60, 80, 90, 95 \}$. We also present the $90$\textsuperscript{th} and $95$\textsuperscript{th} superquantile of the test misclassification error (i.e., average misclassification error
    of the worst $10\%$ and $5\%$ of the clients respectively),
    as well as the
    average test misclassification error of the best $10\%$ clients. The weight $\alpha_i$ used for training
    client $i$ was set as proportional 
     to the number of datapoints on the client.
     
\myparagraph{Evaluation times}    
    We evaluate the model during the training process once every $l$ communication rounds. The value of $l$ used was $l=50$ for EMNIST linear model, 
    $l=10$ for EMNIST ConvNet, $l=20$ for Sent140 linear model and $l=25$ for Sent140 RNN.

\myparagraph{Hardware}
    We run each experiment as a simulation as a single process. 
    The linear models were trained on 
    m5.8xlarge AWS instances, each with an Intel Xeon Platinum 8000 series processor with $128$ GB of memory running at most $3.1$ GHz. 
    The neural network experiments were trained on workstation with 
    an Intel i9 processor with $128$ GB of memory at $1.2$ GHz, and two Nvidia Titan Xp GPUs. The Sent140 RNN experiments were run on a CPU
    while the other neural network experiments were run using GPUs.

\myparagraph{Software Packages}
    Our implementation is based on NumPy using the Python language. 
    In the neural network experiments, we use PyTorch to implement the 
    \textit{LocalUpdate} procedure, i.e., the model itself and the 
    automatic differentiation routines provided by PyTorch to make
    SGD updates. 
    
\myparagraph{Randomness}
    Since several sampling routines appear in the procedures such as the selection of clients or the local stochastic gradient, we carry out our experiments with five different seeds and plot the average metric value over these seeds. Each simulation is run on a single process. 
    Where appropriate, we report one standard deviation from the mean.
    
\subsection{Experimental Results}\label{sec:a:sfl:exp_results}

We now present the experimental results of the paper. 
\begin{itemize}[topsep=0.2pt]
    \setlength\itemsep{0.2 pt}
    \item We present different metrics on the distribution of test misclassification error over the clients, comparing \newfl
    to baselines.

    \item We study the convergence of \Cref{algo:sfl-new} for 
    \newfl over the course of the optimization,
    and compare it with \fedavg.
    
    \item We plot the histograms of the distribution of losses over train clients as well as the test misclassification errors over test clients at the end of the training process.
    
    \item We present in the form of scatter plots the training loss and test misclassification error across clients achieved at the end of the training, versus the number of local data points on the client.
    
    \item We present the number of clients having a loss greater than the quantile at each communication round for \newfl. This gives the effective number of clients selected in each round, cf. \Cref{prop:sfl-quantile} and \Cref{remark:sfl-client-filtering}. 
    
\end{itemize}

\myparagraph{Comparison to Baselines}
We now present a detailed comparison of 
various statistics of the test misclassification error distribution
for different methods in \Cref{table:a:sfl:ext_comp_emnist_linear}-
For each column, the smallest mean over five random runs is highlighted
in bold. Further, if no other method is within one standard deviation of this method, the entire entry (i.e., mean\,$\pm$\,std) 
is highlighted in bold.

\myparagraph{Histograms of Loss and Test Misc. Error over Devices}
Here, we plot the histograms of the loss distribution over training clients and the misclassification error distribution over testing clients.
We report the losses and errors obtained at the end of the training 
process. Each metric is averaged per client over 5 runs of the random seed. \Cref{fig:a:sfl:expt:hist-emnist} shows the histograms for EMNIST,
while \Cref{fig:a:sfl:expt:hist-sent-shake} shows the histograms for Sent140 dataset.
for Sent140. We note that \newfl tends to exhibit thinner upper tails at multiple values of $\theta$ and a lower variance of the distribution in most of the cases. This is also confirmed by the figures in Tables \ref{table:a:sfl:ext_comp_emnist_linear} to~\ref{table:a:sfl:ext_comp_sent140_rnn}. This shows the benefit of using \newfl over vanilla \fedavg.

\myparagraph{Performance compared to local data size}
Next, we plot the loss on training clients versus the amount of local data on the client
and the misclassification error on the test clients versus the 
amount of local data on the client.
See \Cref{fig:a:sfl:expt:scatter-emnist} for EMNIST and 
\Cref{fig:a:sfl:expt:scatter-sent-shake} for Sent140.

Observe firstly that improvement over the worst cases is achieved regardless of the local data size of the clients. Indeed, the client re-weighting step operates a sorting of the loss of the clients which does not prevent small clients from being selected. In contrast, \fedavg, by averaging with respect to the weights of the clients is likely to put more weight on the clients with larger local data size. Secondly, \newfl appears to reduce the variance of 
of the loss on the train clients.
Lastly, note that amongst test clients with a small number 
of data points (e.g., $<200$ for EMNIST or $<100$ for Sent140), \newfl reduces the variance of the misclassification error. Both effects are more pronounced on 
the neural network models.

\section{Numerical Experiments: End-to-End Differential Privacy}
\label{sec:a:sfl:dp-expt}
We consider a synthetic classification dataset to evaluate the privacy-utility tradeoff of \newfl under end-to-end differential privacy.

\subsection{End-to-End Differential Privacy with \newfl}

To obtain an end-to-end differentially private version of \newfl, we modify the weight aggregation step of \Cref{algo:sfl-private} (line~\ref{line:sfl-private:aggregation}).
Specially, we clip the weight updates and add Gaussian noise to obtain differential privacy via the Gaussian mechanism.
The overall algorithm is given in \Cref{algo:sfl-private:end-to-end}.

\myparagraph{Privacy Accounting}
We now discuss the privacy spent in each communication round. For simplicity, we assume the number $m\pow{t} = \sum_{i \in S} \mathbb{I}(F_i(w\pow{t}) \ge Q\pow{t})$ of selected clients is publicly known. 

\begin{claim}
    Consider the setting of \Cref{algo:sfl-private:end-to-end} with noise scale $\sigma_w$, norm bound $C$
    and \Cref{algo:sfl:quantile-dp:new} with $b$ bins and noise scale $\sigma = \sigma_q$.
    Each round of \Cref{algo:sfl-private:end-to-end} satisfies $(1/2)\eps^2$-concentrated DP where 
    \[
        \frac{1}{2}\eps^2 
        = 
          \frac{1}{2} \min \left\{ \frac{c^2 \log_2^2 b}{m\sigma_q^2} + \psi b \,, \, \left( \frac{c \log_2 b}{\sqrt{m} \sigma_q} + \psi \sqrt{2b} \right)^2  \right\}  + \frac{\sigma_w^2}{2 C^2}  \,,
    \]
    where $\psi = 10 \sum_{i=1}^{m-1} \exp\left( - 2 \pi^2 \sigma_q^2 i/(i+1)  \right)$.
\end{claim}
\begin{proof}
    The privacy bound of the quantile computation from  \Cref{algo:sfl:quantile-dp:new} is given by \Cref{thm:sfl:quantile:new}. 
    Since the contribution $\delta_i\pow{t}$ of each client has $\ell_2$ norm $\norm{\delta_i\pow{t}} \le C$ and we add Gaussian noise $\Ncal(0, \sigma_w^2 I_d)$, the weight update step satisfies $\sigma_w^2/(2C^2)$-concentrated DP. 
    The proof is completed by noting that concentrated differential privacy composes additively. 
\end{proof}

To obtain a bound on the concentrated DP of the entire algorithm, 
we rely on generic upper bounds of \cite{zhu2019poisson} for privacy amplification by subsampling. 

\begin{algorithm}[t]
	\caption{The \newfl Algorithm with End-to-End Differential Privacy}
	\label{algo:sfl-private:end-to-end}
\begin{algorithmic}[1]
		\Require Initial iterate $w\pow{0}$,
		    number of communication rounds $T$, 
		    number of clients per round $m$, 
		    number of local updates $\tau$,
		    local step size $\gamma$, 
		    $\ell_2$ norm bound $C$ for weight updates, 
		    noise variance $\sigma_w^2$
	    \For{$t=0, 1, \ldots, T-1$}
	        \State Sample $m$ clients from $[n]$ without replacement in $S$
	        \State Estimate the $(1-\theta)$-quantile of 
	            $F_i(w\pow{t})$ for $i \in S$ with distributed differential privacy (\Cref{algo:sfl:quantile-dp:new}); call this $Q\pow{t}$ 
	        \State Set $m\pow{t} = \sum_{i \in S} \mathbb{I}\left(F_i(w\pow{t}) \ge Q\pow{t}\right)$
	        \For{each selected client $i \in S$ in parallel}
	            \State Initialize $w_{k, 0}\pow{t} = w\pow{t}$
	    	    \For{$k=0, \ldots, \tau-1$} 
	    	        \State $w_{i, k+1}\pow{t} = (1-\gamma\lambda)w_{i, k}\pow{t} - \gamma \grad F_i(w_{i, k}\pow{t})$
	    	    \EndFor
	    	    \State Define the norm-clipped update contributed by the client 
	    	    \[
	    	        \delta_i \pow{t} = 
	    	        \begin{cases}
	    	            \frac{C \, (w_{i, \tau}\pow{t} - w\pow{t})}{\max\left\{C, \, \norm{w_{i, \tau}\pow{t} - w\pow{t}}_2 \right\} } \,,
	    	            &  \text{ if } F_i(w\pow{t}) \ge Q\pow{t} \\
	    	            \bm{0}_d, & \text{ else }
	    	        \end{cases}
	    	  \]
	    	\EndFor
	    	\State Sample Gaussian noise $\xi\pow{t} \sim \Ncal(0, \sigma_w^2 I_d)$ and 
	    	update 
	    	\[
	    	w\pow{t+1} = w\pow{t} + \frac{1}{m\pow{t}}\sum_{i \in S} \delta_i \pow{t} + \xi\pow{t} 
	    	\]
	    \EndFor
	    \State \Return $w_T$
\end{algorithmic}
\end{algorithm}

\subsection{Experimental Setup}
We consider a synthetic classification dataset and train a linear model on it. 

\myparagraph{Dataset Description}
We create a $10$-class classification dataset in $d=20$ dimensions, inspired by \cite{guyon2003design}. 
The input $x$ for each class $k$ is drawn from a Gaussian of mean $\mu_i$ and identity covariance in $\reals^{15}$. The means $\mu_i$'s are the corners of a random polytope in $\reals^{15}$. 
We add $2$ features that are linear combinations of the $15$ informative ones and $3$ features that are pure noise. Overall, the dataset can be generated using the \texttt{make\_classification} function of scikit-learn~\cite{scikit-learn} as

\begin{center}
\begin{adjustbox}{max width=0.9\linewidth}
\begin{lstlisting}
x, y = make_classification(
	n_samples=int(5e5), n_features=20, 
	n_informative=15, n_redundant=2, n_repeated=0, 
	n_classes=10, n_clusters_per_class=1, 
	class_sep=5.0, hypercube=False, random_state=2345
)
\end{lstlisting}
\end{adjustbox}
\end{center}

We now split this dataset into a federated dataset with $n=2500$ training clients and $n'=500$ validation and $n''=500$ test clients. 
The data distribution $q_i(x, y) = q_i(y) q_i(x \vert y)$ across the clients is designed to exhibit a label shift, i.e., the distribution $q_i(y)$ over labels for each client is different while the class-conditional distribution $q_i(x \vert y=k) = \Ncal(\mu_k, I_d)$ is the same across clients. 
The class distribution $q_i(y)$ on each training client $i$ is drawn from a Dirichlet distribution $\text{Dir}(0.5)$, while that for a validation or test client is drawn from $\text{Dir}(0.01)$.
We sample $100$ input-output pairs for each training, validation, and test client. 

\myparagraph{Model and Per-Client Objective} 
We use a linear model (with intercept) on each client and the multinomial logistic loss, also known as the cross entropy loss, to define the per-client objective. 

\myparagraph{Algorithms and Hyperparameters}
We compare \Cref{algo:sfl-private:end-to-end} with DP-FedAvg~\cite{mcmahan2018learning}, a version of FedAvg with differential privacy. 

Both algorithms used a single full gradient step per client with a fixed learning rate of $0.1$. For each algorithm, we sample $100$ clients per round and run the training for a total of $1000$ rounds. 
We vary the privacy budget $\eps \in \{3, 5, 10, 15, 20\}$ and tune the following hyperparameters for each algorithm. 

For DP-FedAvg, we tune the $\ell_2$ norm bound (analogous to $C$ in \Cref{algo:sfl-private:end-to-end}) and set the noise scale $\sigma_w$ depending on the target privacy budget $\eps$ and the norm bound $C$. 
For \Cref{algo:sfl-private:end-to-end}, we allocate $r$-times the privacy budget of the weight updates to the quantile updates. In addition, we also tune: 
\begin{itemize}[noitemsep]
    \item the loss upper bound $B$, so that all losses are truncated to $[0, B]$, 
    \item the number of bins $b$ in the hierarchical histogram,
    \item the $\ell_2$ norm bound $C$ for the weight update. 
\end{itemize}
We tune all $4$ hyperparameters with a grid search and set the noise scale $\sigma_w$ for the weight update, and $\sigma_q/c$ for the quantile update depending on the selected hyperparameters and the privacy budget $\eps$. The objective of the grid search was to minimize the $90$\textsuperscript{th} percentile of the misclassification errors across all validation clients. 

The ranges of the hyperparameters considered are 
quantile privacy ratio $r \in \{0.1, 0.25, 0.5, 0.75\}$, 
loss upper bound $B \in \{0.7, 0.9, 1.1, 1.3, 1.5\},$\footnote{
The loss at convergence was around $0.7$, while that at random guessing is $\log 10 \approx 2.3$.}, 
number of bins $b \in \{16, 32, 64\}$, and
update norm $C \in \{0.9, 1.1, 1.3, 1.5\}$.\footnote{
These correspond approximately to the $0.3, 0.5, 0.7, 0.9$ quantiles of the update norms of FedAvg without differential privacy, during the latter half of training.}

\bibliography{bib/fl}
\bibliographystyle{abbrvnat} %

\end{document}